\ifdefined\pdfminorversion
\fi

\documentclass{article}

\usepackage{microtype}
\usepackage{graphicx}
\usepackage{subcaption}
\usepackage{booktabs} % for professional tables
\usepackage{bm}
\usepackage{enumitem}
\newcommand{\vecw}{\mathbf{w}}
\newcommand{\tr}{\operatorname{tr}}
\newcommand{\vecg}{\mathbf{g}}
\usepackage{hyperref}

\usepackage[preprint]{icml2026}

\usepackage{amsmath}
\usepackage{amssymb}
\usepackage{mathtools}
\usepackage{amsthm}
\usepackage{bbding} % 引入bbding包以使用\XSolidBrush符号
\usepackage{colortbl} % 引入colortbl包以支持\rowcolor
\usepackage{multirow} % 引入multirow包以支持\multirow

\usepackage[capitalize,noabbrev]{cleveref}

\theoremstyle{plain}
\newtheorem{theorem}{Theorem}[section]

\newtheorem{lemma}[theorem]{Lemma}
\newtheorem{corollary}[theorem]{Corollary}
\theoremstyle{definition}

\newtheorem{assumption}[theorem]{Assumption}
\theoremstyle{remark}
\newtheorem{remark}[theorem]{Remark}

\usepackage[textsize=tiny]{todonotes}

\icmltitlerunning{Heterogeneity-Aware Knowledge Sharing for Graph Federated Learning}

\begin{document}

\twocolumn[
  \icmltitle{Heterogeneity-Aware Knowledge Sharing for Graph Federated Learning}

    % Heterogeneity-Aware Knowledge Aggregation for Graph Neural Networks: A Feature and Spectral Alignment Approach
    % Dual-Level Knowledge Sharing for Addressing Feature and Structural Heterogeneity in Graph Federated Learning 



  \icmlsetsymbol{equal}{*}

  \begin{icmlauthorlist}
    \icmlauthor{Wentao Yu}{sch}
    \icmlauthor{Sheng Wan}{yyy}
    \icmlauthor{Shuo Chen}{comp}
    \icmlauthor{Bo Han}{hk}
    \icmlauthor{Chen Gong}{sh}
  \end{icmlauthorlist}

  \icmlaffiliation{sch}{School of Computer Science and Engineering, Nanjing University of Science and Technology, Nanjing, China}
  \icmlaffiliation{yyy}{College of Smart Agriculture, Nanjing Agricultural University, Nanjing, China}
  \icmlaffiliation{comp}{School of Intelligence Science and Technology, Nanjing University, Suzhou, China}
  \icmlaffiliation{hk}{Department of Computer Science, Hong Kong Baptist University, Hong Kong, China}
  \icmlaffiliation{sh}{School of Automation and Intelligent Sensing, Shanghai Jiao Tong University, Shanghai, China}

  \icmlcorrespondingauthor{Chen Gong}{chen.gong@sjtu.edu.cn}

  \icmlkeywords{Machine Learning, ICML}

  \vskip 0.3in
]



\printAffiliationsAndNotice{}  % no special notice (required even if empty)

\begin{abstract}
Graph Federated Learning (GFL) enables distributed graph representation learning while protecting the privacy of graph data. However, GFL suffers from heterogeneity arising from diverse node features and structural topologies across multiple clients. To address both types of heterogeneity, we propose a novel graph \underline{\textbf{Fed}}erated learning method via \underline{\textbf{S}}emantic and \underline{\textbf{S}}tructural \underline{\textbf{A}}lignment (FedSSA), which shares the knowledge of both node features and structural topologies. For node feature heterogeneity, we propose a novel variational model to infer class-wise node distributions, so that we can cluster clients based on inferred distributions and construct cluster-level representative distributions. We then minimize the divergence between local and cluster-level distributions to facilitate semantic knowledge sharing. For structural heterogeneity, we employ spectral Graph Neural Networks (GNNs) and propose a spectral energy measure to characterize structural information, so that we can cluster clients based on spectral energy and build cluster-level spectral GNNs. We then align the spectral characteristics of local spectral GNNs with those of cluster-level spectral GNNs to enable structural knowledge sharing. Experiments on six homophilic and five heterophilic graph datasets under both non-overlapping and overlapping partitioning settings demonstrate that FedSSA consistently outperforms eleven state-of-the-art methods.
% How to address heterogeneity is a critical problem for graph federated learning. Existing methods mitigate heterogeneity by performing weighted federated aggregation in parameter space. However, since heterogeneity arises from node features and structural topologies, model parameters fail to represent the characteristics of node features and structural topologies, thereby restricting their performance. To address this issue, we propose a graph \underline{\textbf{Fed}}erated learning method via \underline{\textbf{S}}emantic and \underline{\textbf{S}}tructural \underline{\textbf{A}}lignment (FedSSA), which shares the knowledge of both node features and structural topologies. For node feature heterogeneity, we propose to infer class-wise node distributions, so that we conduct class-wise clustering and build cluster-level representative distributions. We then minimize the divergence between local and cluster-level distributions. For structural heterogeneity, we employ spectral Graph Neural Networks (GNNs) and propose a spectral energy measure to characterize structural information, so that we cluster clients based on spectral energy and build cluster-level spectral GNNs. We then align the spectral characteristics of local spectral GNNs with those of cluster-level spectral GNNs to enable structural knowledge sharing. Experiments on six homophilic and five heterophilic graph datasets under two partitioning settings demonstrate that FedSSA consistently outperforms eleven state-of-the-art methods. Our code is available at \url{https://anonymous.4open.science/r/FedSSA}.
\end{abstract}

\section{Introduction}
Graphs are fundamental data structures in real-world applications, such as social networks, transportation systems, and molecular chemistry~\cite{bai2022two, 10032180, zhou2024traffic, zhou2025fedtps, yu2026atom}. In many real-world applications, large-scale graphs are often partitioned into a set of subgraphs and distributed across multiple clients due to storage and computation limitations~\cite{meng2024survey}. In addition, subgraphs are only locally accessible due to restrictions from privacy regulations and data protection protocols. Therefore, Graph Federated Learning (GFL) has emerged as a promising paradigm for distributed graph representation learning, where multiple clients collaboratively train models without sharing their raw graph data~\cite{baek2023personalized,yu2025homophily, yuintegrating}.

However, due to different distributions of graph data across clients, graph data on each client is usually non-independent and identically distributed (non-IID). Consequently, GFL suffers from heterogeneity arising from both diverse node features and varied structural topologies across clients~\cite{li2023fedgta, li2024adafgl, zhu2024fedtad}. These two types of heterogeneity pose significant challenges to GFL, leading to training instability and performance degradation~\cite{huang2024federated}. To tackle this fundamental and challenging issue, a number of studies have been proposed. Specifically, existing methods can be broadly categorized into two types: (i) methods that solely focus on addressing structural heterogeneity, and (ii) methods that do not differentiate between node feature heterogeneity and structural heterogeneity. For the first type, representative methods include FedGTA~\cite{li2023fedgta}, FedTAD~\cite{zhu2024fedtad}, and AdaFGL~\cite{li2024adafgl}, which perform topology-aware aggregation by considering structural information from other clients. However, these methods are based on homophily assumption (\textit{i.e.}, edges tend to connect similar nodes), which does not hold in heterophilic graphs~\cite{NEURIPS2023_01b68102, yu2025homophily}. For the second type, typical methods include GCFL~\cite{NEURIPS2021_9c6947bd}, FED-PUB~\cite{baek2023personalized}, and FedIIH~\cite{wentao2025fediih}, which primarily carry out personalized federated aggregation by adjusting aggregation weights among clients. However, simply performing the weighted aggregation of model parameters may not be truly effective in solving heterogeneity, since model parameters may fail to represent the intrinsic characteristics of node features and structural topologies, thereby restricting their effectiveness in scenarios with strong heterogeneity.

Therefore, we pose the following research question:
\begin{center}
\textit{How can we explicitly address node feature heterogeneity and structural heterogeneity in graph federated learning?}
\end{center}

To answer this question, we propose a novel graph \underline{\textbf{Fed}}erated learning method via \underline{\textbf{S}}emantic and \underline{\textbf{S}}tructural \underline{\textbf{A}}lignment (FedSSA), which shares the knowledge of both node features and structural topologies among clients to tackle two types of heterogeneity in GFL. On one hand, to address node feature heterogeneity, we propose a novel variational model to infer the class-wise distributions of node features on each client. By clustering clients based on these inferred distributions, we construct cluster-level representative distributions for each class by matching statistical moments (\textit{i.e.}, mean and covariance). Afterwards, we enforce class-wise semantic knowledge alignment by minimizing the divergence between local distributions and cluster-level distributions to facilitate semantic knowledge sharing. On the other hand, to address structural heterogeneity, we employ spectral Graph Neural Networks (GNNs) and propose a novel spectral energy measure to characterize the structural information of each client. By clustering clients based on their spectral energies embedded in Grassmann manifold, we construct a cluster-level spectral GNN for each cluster. Subsequently, we align the spectral characteristics of local spectral GNNs with those of cluster-level spectral GNNs to enable structural knowledge sharing. Furthermore, we theoretically prove that our proposed FedSSA converges at a linear rate. Extensive experiments on eleven datasets demonstrate the effectiveness of our FedSSA. To be specific, it outperforms the second-best method by a large margin of 2.82\% in terms of classification accuracy.

\section{Related Work}
In this section, we review the typical works related to this paper, including Graph Federated Learning (GFL), Clustered Federated Learning (CFL), and Spectral GNNs.

\subsection{Graph Federated Learning}
Graph Federated Learning (GFL) aims to train GNNs across multiple clients without sharing raw graph data~\cite{baek2023personalized, yu2025homophily, yuintegrating}. However, due to the non-IID nature of graph data across clients, GFL faces significant challenges stemming from heterogeneity in both node features and structural topologies~\cite{li2023fedgta, li2024adafgl, zhu2024fedtad}. To tackle this issue, existing methods can be broadly categorized into two types: (i) methods that solely focus on mitigating structural heterogeneity, and (ii) methods that do not differentiate between node feature heterogeneity and structural heterogeneity.

For the first type, representative methods include FedGTA~\cite{li2023fedgta}, FedTAD~\cite{zhu2024fedtad}, and AdaFGL~\cite{li2024adafgl}, which modifies local models with structural information from other clients. For example, FedGTA~\cite{li2023fedgta} recalibrates aggregation weights based on topology-aware local smoothing confidence. Subsequently, FedTAD~\cite{zhu2024fedtad} employs topology-aware knowledge distillation to rectify unreliable knowledge arising from structural heterogeneity. Meanwhile, AdaFGL~\cite{li2024adafgl} leverages federated knowledge extractor to optimize structural topology on each client. However, these methods are based on homophily assumption (\textit{i.e.}, edges tend to connect similar nodes), which does not hold in heterophilic graphs~\cite{NEURIPS2023_01b68102}.

For the second type, typical methods include GCFL~\cite{NEURIPS2021_9c6947bd}, FED-PUB~\cite{baek2023personalized}, and FedIIH~\cite{wentao2025fediih}, which treats node feature heterogeneity and structural heterogeneity as a unified challenge. These methods mainly carry out personalized federated aggregation by adjusting the aggregation weights of clients. For example, GCFL~\cite{NEURIPS2021_9c6947bd} dynamically clusters clients and aggregates model parameters within clusters to mitigate heterogeneity. Alternatively, FED-PUB~\cite{baek2023personalized} optimizes the aggregation weights based on similarities between pairwise clients, which are calculated based on the outputs of local GNNs. Similarly, FedIIH~\cite{wentao2025fediih} performs a weighted federation of model parameters based on similarities calculated from inferred graph distributions. Nevertheless, in practice, node feature heterogeneity and structural heterogeneity have different characteristics. Consequently, simply performing the weighted aggregation of model parameters may not be truly effective in solving heterogeneity, as model parameters may fail to reveal the intrinsic characteristics of node features and structural topologies. This limitation significantly restricts their performance in scenarios with strong heterogeneity. Unlike existing methods, we propose to separately address node feature heterogeneity and structural heterogeneity.
\subsection{Clustered Federated Learning}
Clustered Federated Learning (CFL) aims to mitigate heterogeneity by grouping clients with similar data distributions into clusters and constructing personalized models tailored for each cluster~\cite{sattler2021clustered, ghosh2022efficient, kim2024clustered, vardhan2024improved}. Most of the existing CFL methods mainly leverage three types of clustering signals, namely gradients~\cite{sattler2021clustered}, parameters~\cite{vardhan2024improved}, and losses~\cite{ghosh2022efficient}. Specifically, CFL~\cite{sattler2021clustered} adopts a top-down mechanism to recursively bipartition clients based on the cosine similarity of gradients. In contrast, SR-FCA~\cite{vardhan2024improved} utilizes a bottom-up mechanism by initializing distinct clusters for each client and then successively refining them based on the similarity of parameters. Alternatively, IFCA~\cite{ghosh2022efficient} broadcasts $K$ models to each client. Each client then evaluates $K$ models and selects the model that yields the lowest loss to determine its cluster assignment. However, these methods are primarily tailored for traditional federated learning scenarios and fail to incorporate the structural topologies of graph data. Consequently, our work proposes a novel clustering mechanism by explicitly leveraging node features and structural topologies, respectively.
\subsection{Spectral Graph Neural Networks}
Spectral GNNs leverage the spectral properties of the graph Laplacian to perform graph signal filtering in the spectral domain~\cite{defferrard2016convolutional, he2021bernnet, guo2023graph, chen2024polygcl}. Fundamentally, these methods utilize the eigenvalues and eigenvectors of the graph Laplacian matrix to define graph spectral filters that capture and model structural information. For instance, early approaches such as ChebNet~\cite{defferrard2016convolutional} employ Chebyshev polynomial bases to efficiently approximate graph spectral filters. However, conventional spectral GNNs typically optimize polynomial coefficients without sufficient constraints, which leads to ill-posed filters. To address this problem, BernNet~\cite{he2021bernnet} constrains the coefficients of the Bernstein basis to ensure the learned filters accurately reflect the spectral characteristics of observed graph. Subsequently, OptBasisGNN~\cite{guo2023graph} further improves the expressive power of spectral graph filters by learning an optimal polynomial basis directly from graph data. Motivated by the capability of Spectral GNNs to capture structural information, we employ spectral GNNs to characterize structural topologies and facilitate structural knowledge sharing.

\section{Notations}
This section introduces mathematical notations used throughout the paper. We focus on node classification task in GFL, where $M$ clients collaboratively train GNNs on their local graphs. Let $\mathcal{M} = \{1, 2, \cdots, M\}$ denote the set of clients. Each client $m \in \mathcal{M}$ holds a local graph $\mathcal{G}_m = \langle \mathcal{V}_m, \mathcal{E}_m \rangle$, where $\mathcal{V}_m$ and $\mathcal{E}_m$ denote the node set and edge set, respectively. For client $m$, let $n_m = |\mathcal{V}_m|$ and $e_m = |\mathcal{E}_m|$ denote the number of nodes and the number of edges, while $d$ and $C$ represent the feature dimension and the number of classes, respectively. In addition, we denote node feature matrix as $\mathbf{X}_m \in \mathbb{R}^{n_m \times d}$, adjacency matrix as $\mathbf{A}_m \in \mathbb{R}^{n_m \times n_m}$, and label matrix as $\mathbf{Y}_m \in \mathbb{R}^{n_m \times 1}$. Furthermore, let $\mathbf{L}_m = \mathbf{I}-\mathbf{D}_m^{-\frac{1}{2}}\mathbf{A}_m\mathbf{D}_m^{-\frac{1}{2}}$ denote the symmetric normalized Laplacian matrix of $\mathcal{G}_m$ without self-loops, where $\mathbf{I}$ is the identity matrix and $\mathbf{D}_m \in \mathbb{R}^{n_m \times n_m}$ is the degree matrix. Subsequently, let $\mathbf{L}_m = \mathbf{U}_m \mathbf{\Lambda}_m \mathbf{U}_m^{\top} $ denote the eigendecomposition of $\mathbf{L}_m$, where $\mathbf{U}_m$ represents the matrix of eigenvectors and $\mathbf{\Lambda}_m$ denotes the diagonal matrix of eigenvalues. Therefore, the graph Fourier transform of $\mathbf{X}_m$ can be defined as $\tilde{\mathbf{X}}_m = \mathbf{U}_m^{\top}\mathbf{X}_m \in \mathbb{R}^{n_m \times d}$. Here we also introduce some norms that will be used. First, $\Vert \mathbf{v} \Vert_2 = \sqrt{\sum\nolimits_{i} v_i^2}$ denotes the $\ell_2$-norm of a vector $\mathbf{v}$, where $v_i$ is the $i$-th element of $\mathbf{v}$. Second, $\Vert \mathbf{X} \Vert_\mathrm{F} = \sqrt{\sum\nolimits_{i,j} |x_{ij}| ^2}$ denotes the Frobenius norm of a matrix $\mathbf{X}$, where $x_{ij}$ is the $(i, j)$-th element of $\mathbf{X}$.

\section{Our Proposed Method}
In this section, we provide the details of our proposed graph \underline{\textbf{Fed}}erated learning method via \underline{\textbf{S}}emantic and \underline{\textbf{S}}tructural \underline{\textbf{A}}lignment (FedSSA).

\subsection{Framework Overview}
As shown in \cref{fig-framework}, our proposed FedSSA comprises two components. \textbf{(i) Semantic Knowledge Sharing:} To address the heterogeneity of node feature, we propose a variational model to infer class-wise node feature distributions on each client. Clients are clustered based on inferred distributions, and local distributions are aligned with cluster-level distributions. \textbf{(ii) Structural Knowledge Sharing:} To address structural heterogeneity, we employ spectral GNNs and introduce a spectral energy measure to characterize local structural topologies. Clients are clustered by spectral energy, which enables the alignment between local spectral GNNs and cluster-level spectral GNNs.

\begin{figure*}[t]
  % \vskip 0.2in
  \begin{center}
    \centerline{\includegraphics[width=17cm]{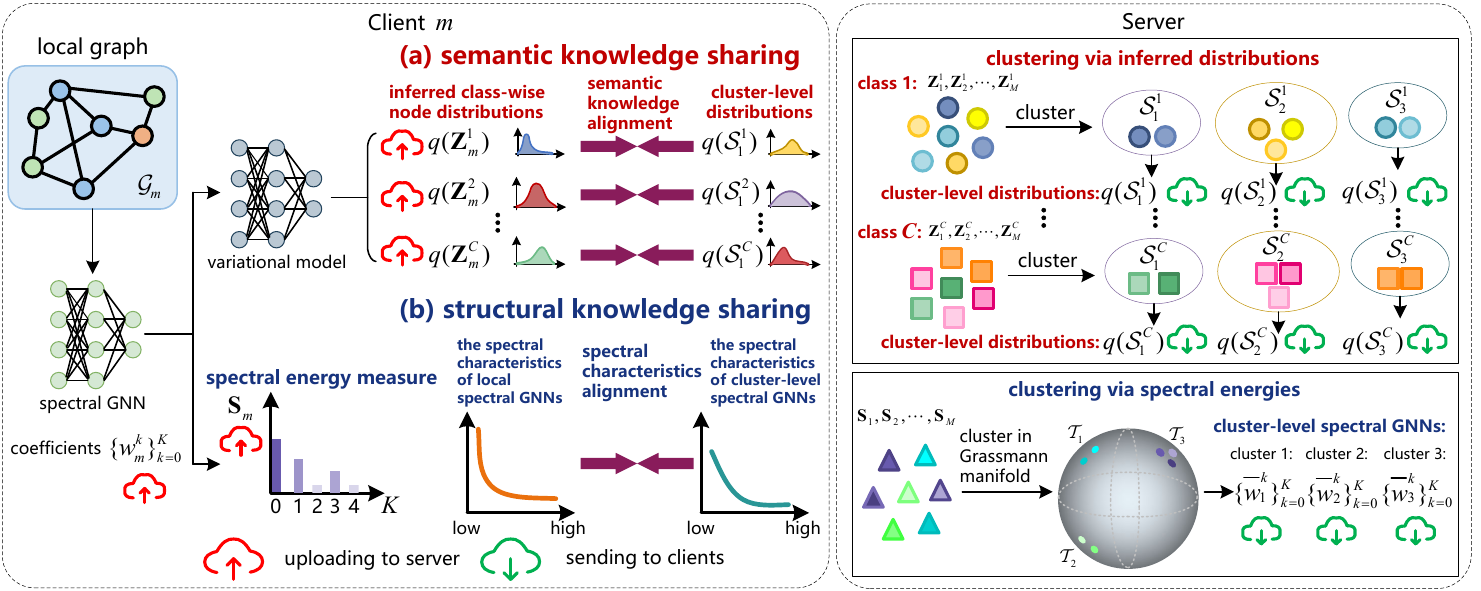}}
    \caption{The overview of our proposed FedSSA. \textbf{(a) Semantic Knowledge Sharing:} A variational model infers class-wise node feature distributions on each client. Clients are clustered based on inferred distributions, and local distributions are aligned with cluster-level distributions. \textbf{(b) Structural Knowledge Sharing:} Spectral GNNs are employed alongside a spectral energy measure to characterize local structural topologies. Clients are clustered by spectral energy, which enables alignment between local and cluster-level spectral GNNs.}
     \vskip -0.3in
    \label{fig-framework}
  \end{center}
\end{figure*}

\subsection{Semantic Knowledge Sharing}
To address node feature heterogeneity, we aim to share semantic knowledge among clients.

\textbf{Variational Model} We first propose a variational model to infer the class-wise distributions of node features on each client. As illustrated in \cref{fig-prob-graph}, $\mathbf{X}_m$ and $\mathbf{Y}_m$ are observed variables, while $\mathbf{Z}_m$ represents a latent variable on the $m$-th client. Based on the graphical model in \cref{fig-prob-graph}, the joint probability distribution can be factorized as follows:
\begin{equation}\label{joint-distribution}
p(\mathbf{X}_m, \mathbf{Y}_m, \mathbf{Z}_m) = p(\mathbf{X}_m | \mathbf{Z}_m, \mathbf{Y}_m)  p(\mathbf{Z}_m) p(\mathbf{Y}_m),
\end{equation}
where $p(\mathbf{Z}_m)$ denotes the prior distribution of $\mathbf{Z}_m$, and $p(\mathbf{X}_m | \mathbf{Z}_m, \mathbf{Y}_m)$ denotes the conditional distribution. We then employ the true posterior distribution $p(\mathbf{Z}_m|\mathbf{X}_m, \mathbf{Y}_m)$ to represent the class-wise distributions of node features. However, this true posterior inference is computationally intractable. Therefore, we attempt to infer it with a variational distribution $q(\mathbf{Z}_m|\mathbf{X}_m, \mathbf{Y}_m)$ via variational inference~\cite{kingma2013auto, kingma2014semi}. According to the graphical model in \cref{fig-prob-graph}, the Evidence Lower BOund (ELBO) can be formulated as follows:
\begin{equation}\label{eq1}
\begin{aligned}
\mathcal{L}_\mathrm{ELBO}(\mathbf{X}_m, \mathbf{Y}_m) &= \mathbb{E}_{q(\mathbf{Z}_m|\mathbf{X}_m, \mathbf{Y}_m)} \log p(\mathbf{X}_m|\mathbf{Y}_m, \mathbf{Z}_m) \\
&\quad + \log p(\mathbf{Y}_m)\\
&\quad - D_\mathrm{KL}\big(q(\mathbf{Z}_m|\mathbf{X}_m, \mathbf{Y}_m) \parallel p(\mathbf{Z}_m) \big),
\end{aligned}
\end{equation}
where $\log p(\mathbf{Y}_m)$ denotes the log prior probability of $\mathbf{Y}_m$, and $D_{\mathrm{KL}}$ denotes Kullback-Leibler (KL) divergence. The derivation of Eq.~\eqref{eq1} can be found in~\cref{appendix-elbo1}. Consequently, the latent distribution $q(\mathbf{Z}_m^c)$ of nodes with class label $c$ on the $m$-th client can be derived as follows:
\begin{equation}\label{eq3}
\begin{aligned}
\mathcal{I}_m^c &\triangleq \{ i \in \{1, 2, \cdots,n_m\} \mid (\mathbf{Y}_m)_{[i, 1]} = c \},\\
q(\mathbf{Z}_m^c) &\triangleq q\big((\mathbf{Z}_{m})_{[\mathcal{I}_m^c,:]} \,\big|\, (\mathbf{X}_{m})_{[\mathcal{I}_m^c,:]},\, (\mathbf{Y}_{m})_{[\mathcal{I}_m^c,:]}\big),
\end{aligned}
\end{equation}
where $c = 1, 2, \cdots, C$, $(\mathbf{X}_{m})_{[\mathcal{I}_m^c,:]}$, $(\mathbf{Y}_{m})_{[\mathcal{I}_m^c,:]}$, and $(\mathbf{Z}_{m})_{[\mathcal{I}_m^c,:]}$ denote submatrices (\textit{i.e.}, rows) indexed by $\mathcal{I}_m^c$.
\begin{figure}[t]
  \begin{center}
    \centerline{\includegraphics[width=3cm]{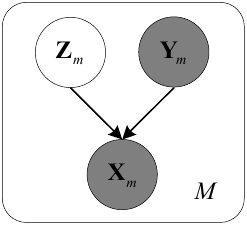}}
    \caption{The graphical model of our proposed variational model, where $\mathbf{Z}_m$ is a latent variable, $\mathbf{X}_m$ and $\mathbf{Y}_m$ are observed variables.}
     \vskip -0.3in
    \label{fig-prob-graph}
  \end{center}
\end{figure}

\textbf{Variational Graph Autoencoder} To infer latent distribution $q(\mathbf{Z}_m^c)$ in Eq.~\eqref{eq3}, we employ a Variational Graph AutoEncoder (VGAE)~\cite{kipf2016variational} on each client. Due to space limitations, details are provided in Appendix~\ref{appendix-vgae}. Specifically, we instantiate $q(\mathbf{Z}_m^c)$ as a multivariate Gaussian distribution $\mathcal{N}(\boldsymbol{\mu}_m^c, \boldsymbol{\Sigma}_m^c)$, where $\boldsymbol{\mu}_m^c$ and $\boldsymbol{\Sigma}_m^c$ denote the mean vector and covariance matrix, respectively.

\textbf{Clustering via Inferred Distributions} After obtaining the class-wise latent distributions of each client, we cluster clients based on these inferred distributions to facilitate semantic knowledge sharing. Specifically, for each class $c$, we first employ the reparameterization trick~\cite{kingma2013auto} to obtain $\tilde{\mathbf{Z}}_m^c=\boldsymbol{\mu}_m^c + (\boldsymbol{\Sigma}_m^c)^{\frac12}\bm{\epsilon}_m^c$, where $\bm{\epsilon}_m^c\sim \mathcal{N}(\mathbf{0}, \mathbf{I})$. We then collect samples $\{\tilde{\mathbf{Z}}_m^c | m \in \mathcal{M}\}$ from all clients and cluster them into $K_\mathrm{node}$ clusters, which can be defined as
\begin{equation}\label{eq4}
\{ \tilde{\mathbf{Z}}_m^c\}_{m=1}^M \xrightarrow{\mathrm{Cluster}} \{\mathcal{S}_i^c\}_{i=1}^{K_\mathrm{node}},
\end{equation}
where $\mathcal{S}_i^c$ denotes the set of clients in the $i$-th cluster for class $c$. Therefore, we can use the Gaussian Mixture Model (GMM) to represent mixed Gaussian distributions in cluster $\mathcal{S}_i^c$ as follows:
\begin{equation}\label{eq5}
p(\mathcal{S}_i^c) = \sum_{m \in \mathcal{S}_i^c} \omega_m^c \cdot \mathcal{N}(\boldsymbol{\mu}_m^c, \boldsymbol{\Sigma}_m^c),
\end{equation}
where $\omega_m^c$ denotes the mixture weight of the $m$-th client in cluster $\mathcal{S}_i^c$, and $\sum_{m \in \mathcal{S}_i^c} \omega_m^c = 1$. Here, $p(\mathcal{S}_i^c)$ denotes the mixture distribution of class-$c$ latent variables from all clients in cluster $\mathcal{S}_i^c$, \textit{i.e.}, the mixture of $\{\mathcal{N}(\boldsymbol{\mu}_m^c, \boldsymbol{\Sigma}_m^c)\}_{m \in \mathcal{S}_i^c}$. However, the computational complexity of GMM increases with the number of clients in the cluster, which makes it inefficient. Therefore, we further simplify $p(\mathcal{S}_i^c)$ by approximating it with a single Gaussian distribution. Specifically, we construct a cluster-level representative distribution $q(\mathcal{S}_i^c) = \mathcal{N}(\boldsymbol{\mu}_{i}^c, \boldsymbol{\Sigma}_{i}^c)$ by matching the first and second moments of local distributions within the cluster, which can be calculated as follows:
\begin{equation}\label{eq6}
\begin{gathered}
\boldsymbol{\mu}_{i}^c = \sum_{m \in \mathcal{S}_i^c} \omega_m^c \boldsymbol{\mu}_m^c,\\
\boldsymbol{\Sigma}_{i}^c
= \big(\sum_{m \in \mathcal{S}_i^c} \omega_m^c( \boldsymbol{\Sigma}_m^c + \boldsymbol{\mu}_m^c (\boldsymbol{\mu}_m^c)^\top)\big)
- \boldsymbol{\mu}_{i}^c (\boldsymbol{\mu}_{i}^c)^\top,
\end{gathered}
\end{equation}
where $\omega_m^c = \frac{n_m^c}{\sum_{n \in \mathcal{S}_i^c} n_n^c}$, $n_m^c$ is the number of nodes with label $c$ on client $m$.

\textbf{Semantic Knowledge Alignment} After obtaining cluster-level representative distributions, we align local distributions with cluster-level distributions to facilitate semantic knowledge sharing. Specifically, we minimize the KL divergence between local distribution $q(\mathbf{Z}_m^c)$ and cluster-level distribution $q(\mathcal{S}_i^c)$ as follows:
\begin{equation}\label{eq7}
\mathcal{L}_\mathrm{node} = \sum_{c=1}^{C} \sum_{i=1}^{K_\mathrm{node}} \sum_{m \in \mathcal{S}_i^c} D_\mathrm{KL}\big(q(\mathbf{Z}_m^c) \parallel q(\mathcal{S}_i^c  ) \big),
\end{equation}
which implies that the optimization of $\mathcal{L}_\mathrm{node}$ can be performed in parallel across $M$ clients.

\subsection{Structural Knowledge Sharing}
To address structural heterogeneity, we aim to share structural knowledge among clients.

\textbf{Spectral GNN} Since spectral GNNs are constructed based on polynomial bases derived from graph Laplacian spectrum, they naturally capture structural information of graphs. Consequently, we employ a spectral GNN on each client to effectively capture structural information, which facilitates structural knowledge sharing. In this paper, we employ the polynomial filter-based spectral GNNs (\textit{e.g.}, ChebNet~\cite{defferrard2016convolutional}, BernNet~\cite{he2021bernnet}). Typically, they can be formulated as follows:
\begin{equation}\label{eq-polynomial-filter}
\mathbf{P}_m = \sum_{k=0}^K w_m^k\, \mathbf{H}_m^k,
\end{equation}
where $\mathbf{H}_m^k = (\mathbf{L}_m)^k \mathbf{X}_m \in \mathbb{R}^{n_m \times d}$ denotes the polynomial filters of the $k$-th order, $w_m^k \in \mathbb{R}$ represents the learnable coefficients of the $k$-th order, $\mathbf{P}_m$ denotes graph representation, and $K$ is the number of orders. According to the theory of graph Fourier transform~\cite{shuman2013emerging}, we can have $\mathbf{H}_m^k = \mathbf{U}_m \mathbf{\Lambda}_m^k \tilde{\mathbf{X}}_m$. Therefore, the physical meaning of $\mathbf{H}_m^k$ is the signal of the $k$-th spectral band. Meanwhile, $w_m^k$ represents learnable coefficients of the $k$-th spectral band. Consequently, $\mathbf{P}_m$ can be viewed as a weighted combination of signals from different spectral bands, where weights are determined by learnable coefficients $\{w_m^k\}_{k=0}^K$. In other words, spectral GNNs effectively capture the structural information of graphs by learning to combine signals from various spectral bands~\cite{yangdisentangled}.

\textbf{Spectral Energy Measure} To characterize the structural information of each client, we analyze the contribution of different spectral bands to the learned graph representation. Specifically, the term $w_m^k \mathbf{H}_m^k$ in Eq.~\eqref{eq-polynomial-filter} can be interpreted as the \textit{weighted spectral response} of the $k$-th band. It represents graph signal explicitly activated by the learnable filter for downstream task. To derive a descriptor effectively capturing the spectral characteristics of the $m$-th client, we aim to propose a spectral energy measure. First, we compute the average response vector for each band via
\begin{equation}\label{eq8}
\begin{gathered}
\mathbf{F}_m^k = w_m^k \mathbf{H}_m^k \in \mathbb{R}^{n_m \times d},\\
\mathbf{E}_m^k = \frac{1}{n_m} \sum_{i=1}^{n_m} (\mathbf{F}_m^k)_{[i, :]} \in \mathbb{R}^d,\\
\end{gathered}
\end{equation}
where $(\mathbf{F}_m^k)_{[i, :]}$ represents the $i$-th row of matrix $\mathbf{F}_m^k$. Here, $\mathbf{E}_m^k$ quantifies the \textit{expectation of spectral response} corresponding to the $k$-th spectral order. A larger magnitude of $\mathbf{E}_m^k$ implies that the $k$-th spectral band plays a more dominant role in graph representation. Subsequently, we define the spectral energy measure $\mathbf{S}_m$ for the $m$-th client by concatenating these response vectors as follows:
\begin{equation}\label{eq-spectral-energy}
\mathbf{S}_m = [\mathbf{E}_m^0, \mathbf{E}_m^1, \cdots, \mathbf{E}_m^K] \in \mathbb{R}^{d \times (K+1)},
\end{equation}
where $[\cdot, \cdot]$ denotes the concatenation operation. This measure $\mathbf{S}_m$ essentially acts as a \textit{spectral fingerprint}, which reflects the energy distribution of structural topologies as learned by spectral GNN. Therefore, it can be utilized to facilitate structural knowledge sharing among clients.

\textbf{Clustering via Spectral Energies} After obtaining the spectral energy measure for each client, we cluster clients based on these spectral energies to facilitate structural knowledge sharing. Specifically, we collect spectral energies $\{\mathbf{S}_m | m \in \mathcal{M}\}$ from all clients. Since each spectral energy measure can be spanned as a spectral subspace, we embed them into the Grassmann manifold, which provides a natural geometric space for measuring subspace-based spectral representations. To measure the distance of spectral subspaces, we employ Chordal distance due to its effectiveness and simplicity. Specifically, we first obtain an orthonormal basis of the column space of $\mathbf{S}_m$ via QR decomposition:
\begin{equation}
\mathbf{S}_m = \mathbf{Q}_m \mathbf{R}_m,\quad \mathbf{Q}_m^\top \mathbf{Q}_m = \mathbf{I}_{K+1},
\end{equation}
where $\mathbf{I}_{K+1} \in \mathbb{R}^{(K+1) \times (K+1)}$ is an identity matrix and $\mathbf{Q}_m \in \mathbb{R}^{d \times (K+1)}$. The Chordal distance between two spectral energy measures is then calculated as
\begin{equation}\label{Chordal_distance}
d_\mathrm{Chordal}(\mathbf{Q}_m, \mathbf{Q}_n) = \big(K+1 - \Vert  \mathbf{Q}_m^{\top} \mathbf{Q}_n \Vert_\mathrm{F}^2\big)^{\frac{1}{2}}.
\end{equation}
We then cluster spectral energies based on their Chordal distances into $K_\mathrm{struct}$ clusters, which can be defined as
\begin{equation}\label{eq9}
\{ \mathbf{S}_m\}_{m=1}^M \xrightarrow{\mathrm{Cluster}} \{\mathcal{T}_j\}_{j=1}^{K_\mathrm{struct}},
\end{equation}
where $\mathcal{T}_j$ denotes the set of clients in the $j$-th cluster.%Note that the Chordal distance is only used to determine the cluster assignment $\{\mathcal{T}_j\}$; the optimization of structural knowledge sharing is performed on the coefficients $\{w_m^k\}$ via the objective below.

\textbf{Structural Knowledge Alignment} After obtaining cluster-level spectral GNNs, we align the spectral characteristics of local spectral GNNs with those of cluster-level spectral GNNs to enable structural knowledge sharing. Specifically, for cluster $\mathcal{T}_j$, we compute the mean of learnable coefficients of spectral GNNs in $\mathcal{T}_j$ as follows:
\begin{equation}\label{eq10}
\overline{w}_j^k \leftarrow \frac{1}{\vert \mathcal{T}_j \vert} \sum_{m \in \mathcal{T}_j} w_m^k.
\end{equation}
Subsequently, we minimize the discrepancy between the coefficients of local spectral GNNs and those of cluster-level spectral GNNs as follows:
\begin{equation}\label{eq11}
\mathcal{L}_\mathrm{align} = \sum_{k=0}^K \sum_{j=1}^{K_\mathrm{struct}} \sum_{m \in \mathcal{T}_j} | w_m^k - \overline{w}_j^k | .
\end{equation}
Moreover, we add a regularization term to constrain the coefficients of spectral GNNs as:
\begin{equation}\label{eq12}
\mathcal{L}_\mathrm{reg} = \sum_{m=1}^M \sum_{k=0}^K \big( \lambda_1 | w_m^k | + \frac{\lambda_2}{2} (w_m^k)^2 \big).
\end{equation}
Finally, we can have $\mathcal{L}_\mathrm{struct} = \mathcal{L}_\mathrm{align} + \mathcal{L}_\mathrm{reg}$, where $\lambda_1,\lambda_2>0$ are hyperparameters adjusting regularization.

\subsection{Convergence Analysis}
Here we analyze the convergence property of our proposed FedSSA. For client $m \in \mathcal{M}$, let $\vecw_m^t$ denote its local model parameters after $t$ communication rounds. For simplicity, we omit subscript $m$ if no notational confusion is incurred. In particular, $\vecw^0$ denotes the initialized local model and $\vecw^T$ denotes the local model after $T$ communication rounds.%First, we make the following assumptions to facilitate theoretical analysis. The detailed descriptions of assumptions are provided in~\cref{appendix:intro-assumptions}. 
% 在正文中定义一个综合的 Assumption
\begin{assumption}[Regularity and Boundedness]\label{asm:combined}
To facilitate theoretical analysis, we make following assumptions (Detailed descriptions are provided in~\cref{appendix:intro-assumptions}.):
\begin{itemize}[leftmargin=*]
    \item \textbf{Regularity of Loss Function:} The population risk function $F(\vecw)$ is $L_F$-smooth and $\lambda_F$-strongly convex.
    \item \textbf{Bounded Intra-cluster Heterogeneity:} We assume the divergence within each cluster is bounded. Specifically, for semantic knowledge, the differences in variational parameters are bounded by constants $\delta_\mu$ and $\delta_\Sigma$, respectively. For structural knowledge, Chordal distance between spectral energy matrices is bounded by $\epsilon_U$.
    \item \textbf{Properties of Learnable Components:} The variational parameters are differentiable with Jacobians bounded by $L_q$. Furthermore, the learnable coefficients of spectral GNNs are bounded by $w_{\max}$ and satisfy the Lipschitz continuity property with constant $L_w$ with respect to spectral energy matrices.
\end{itemize}
\end{assumption}%Subsequently, we present the theorem regarding the convergence of FedSSA as follows.
\begin{theorem}[Convergence of FedSSA]\label{thm:FedSSA-convergence}
Suppose that Assumption~\ref{asm:combined} holds. Let $\eta = \frac{1}{L_F}$ be the learning rate. After $T$ communication rounds, FedSSA satisfies
\begin{equation}\label{eq:convergence-bound}
\begin{aligned}
\|\vecw^T - \vecw^*\|_2
&\leq \left(1 - \frac{\lambda_F}{L_F + \lambda_F}\right)^T \|\vecw^0 - \vecw^*\|_2 \\
&\quad + \frac{L_F + \lambda_F}{\lambda_F L_F} \mathcal{E},
\end{aligned}
\end{equation}
where $\vecw^*$ denotes the local optimal solution of $F(\vecw)$, $F(\vecw)$ denotes population risk (\textit{i.e.}, expected loss), $L_F$ and $\lambda_F$ denote the smoothness and strong convexity constants of $F(\vecw)$, and $\mathcal{E}$ is aggregated error floor:
\begin{equation}\label{eq:delta-bound}
\mathcal{E} = \underbrace{C_1 \left( \delta_\mu + \delta_\mu^2 + \delta_\Sigma \right)}_{\text{semantic error}} + \underbrace{C_2 (K+1)\epsilon_U + \lambda_1 C_3 + \lambda_2 C_4}_{\text{structural error}},
\end{equation}
with constants $C_1, C_2, C_3, C_4 > 0$. Here $\delta_\mu$ and $\delta_\Sigma$ represent bounds on the discrepancies of mean and covariance, while $\epsilon_U$ denotes the bound on Chordal distance.
\end{theorem}

The proof of Theorem~\ref{thm:FedSSA-convergence} is provided in Appendix~\ref{appendix:convergence-proof}. Moreover, we can derive the following corollary on the convergence rate from Eq.~\eqref{eq:convergence-bound}. 
\begin{corollary}[Convergence Rate]\label{cor:convergence-rate}
Under the conditions of Theorem~\ref{thm:FedSSA-convergence}, for any $\xi > 0$, if $T \geq \frac{L_F + \lambda_F}{\lambda_F} \log \frac{\|\vecw^0 - \vecw^*\|_2}{\xi}$, then the following holds:
\begin{equation}\label{eq:convergence-main}
\|\vecw^T - \vecw^*\|_2 \leq \xi + \frac{L_F + \lambda_F}{\lambda_F L_F} \mathcal{E}.
\end{equation}
\end{corollary}

\begin{remark}
Theorem~\ref{thm:FedSSA-convergence} reveals that FedSSA converges at a linear rate to an $\mathcal{O}(\mathcal{E})$-neighborhood of the optimal solution $\vecw^*$. The error floor $\mathcal{E}$ consists of two components: (i) semantic error induced by aligning local distributions with cluster-level distributions, and (ii) structural error induced by aligning local spectral GNNs with cluster-level spectral GNNs. When clustering quality improves to the optimal (\textit{i.e.}, $\delta_\mu, \delta_\Sigma, \epsilon_U \to 0$) and regularization parameters $\lambda_1, \lambda_2$ are sufficiently small, $\mathcal{E}$ approaches zero, and FedSSA converges to a neighborhood of the optimal solution.
\end{remark}

\section{Experiments}
To validate the effectiveness of our FedSSA, we perform extensive experiments on eleven widely used datasets, which include both homophilic and heterophilic graphs. Specifically, we compare FedSSA with eleven baseline methods. Following previous work~\cite{wentao2025fediih}, we evaluate 66 scenarios by varying the number of clients and adopting both non-overlapping and overlapping partitioning settings. To ensure a fair comparison, we compute the mean accuracy and standard deviation over ten independent runs. More detailed experimental settings can be found in~\cref{implementation_details}.

\subsection{Main Results}
\textbf{Homophilic Datasets} \cref{table1} shows experimental results on homophilic datasets under non-overlapping partitioning setting. FedSSA achieves the best performance among all methods, and standard deviations are relatively small as well, which suggests that FedSSA is more effective and stable than the compared methods. Moreover, experimental results under overlapping partitioning setting are provided in~\cref{additional_tables}.

\textbf{Heterophilic Datasets} \cref{table3} shows experimental results on heterophilic datasets under non-overlapping partitioning setting. FedSSA not only achieves the best average performance among all baseline methods, but also outperforms the second-best method (\textit{i.e.}, FedIIH) by 2.82\% in terms of classification accuracy. This is because our FedSSA explicitly tackles both node feature heterogeneity and structural heterogeneity through semantic and structural knowledge alignment. Similarly, experimental results under overlapping partitioning setting are provided in~\cref{additional_tables}.

\begin{table*}[t]
  \centering
  \scriptsize
  \caption{Accuracy (\%) of methods on six \textbf{homophilic} graph datasets under \textbf{non-overlapping} subgraph partitioning setting.}
    \label{table1}
    \renewcommand{\arraystretch}{0.9} % 调整行间距
       \scalebox{0.8}{
  \begin{tabular}{lcccccccccc}
  \hline
  \rowcolor{gray!50}
  \textbf{}     & \multicolumn{3}{c}{Cora}                                                    & \multicolumn{3}{c}{CiteSeer}                                                & \multicolumn{3}{c}{PubMed}                                                  & -              \\ \cline{2-11}
  Methods       & 5 Clients               & 10 Clients              & 20 Clients              & 5 Clients               & 10 Clients              & 20 Clients              & 5 Clients               & 10 Clients              & 20 Clients              & -              \\ \hline
  \rowcolor{gray!20}
  Local         & 81.30$\pm$0.21          & 79.94$\pm$0.24          & 80.30$\pm$0.25          & 69.02$\pm$0.05          & 67.82$\pm$0.13          & 65.98$\pm$0.17          & 84.04$\pm$0.18          & 82.81$\pm$0.39          & 82.65$\pm$0.03          & -              \\ \hline
  FedAvg~\cite{mcmahan2017communication}        & 74.45$\pm$5.64          & 69.19$\pm$0.67          & 69.50$\pm$3.58          & 71.06$\pm$0.60          & 63.61$\pm$3.59          & 64.68$\pm$1.83          & 79.40$\pm$0.11          & 82.71$\pm$0.29          & 80.97$\pm$0.26          & -              \\
  \rowcolor{gray!20}
  FedProx~\cite{MLSYS2020_1f5fe839}       & 72.03$\pm$4.56          & 60.18$\pm$7.04          & 48.22$\pm$6.18          & 71.73$\pm$1.11          & 63.33$\pm$3.25          & 64.85$\pm$1.35          & 79.45$\pm$0.25          & 82.55$\pm$0.24          & 80.50$\pm$0.25          & -              \\
  FedPer~\cite{Arivazhagan2019}        & 81.68$\pm$0.40          & 79.35$\pm$0.04          & 78.01$\pm$0.32          & 70.41$\pm$0.32          & 70.53$\pm$0.28          & 66.64$\pm$0.27          & 85.80$\pm$0.21          & 84.20$\pm$0.28          & 84.72$\pm$0.31          & -              \\
  \rowcolor{gray!20}
  GCFL~\cite{NEURIPS2021_9c6947bd}          & 81.47$\pm$0.65          & 78.66$\pm$0.27          & 79.21$\pm$0.70          & 70.34$\pm$0.57          & 69.01$\pm$0.12          & 66.33$\pm$0.05          & 85.14$\pm$0.33          & 84.18$\pm$0.19          & 83.94$\pm$0.36          & -              \\
  FedGNN~\cite{wu2021fedgnn}        & 81.51$\pm$0.68          & 70.12$\pm$0.99          & 70.10$\pm$3.52          & 69.06$\pm$0.92          & 55.52$\pm$3.17          & 52.23$\pm$6.00          & 79.52$\pm$0.23          & 83.25$\pm$0.45          & 81.61$\pm$0.59          & -              \\
  \rowcolor{gray!20}
  FedSage+\cite{NEURIPS2021_34adeb8e}      & 72.97$\pm$5.94          & 69.05$\pm$1.59          & 57.97$\pm$12.60          & 70.74$\pm$0.69          & 65.63$\pm$3.10          & 65.46$\pm$0.74          & 79.57$\pm$0.24          & 82.62$\pm$0.31          & 80.82$\pm$0.25          & -              \\
  FED-PUB~\cite{baek2023personalized}       & 83.70$\pm$0.19          & 81.54$\pm$0.12          & 81.75$\pm$0.56          & 72.68$\pm$0.44          & 72.35$\pm$0.53          & 67.62$\pm$0.12          & 86.79$\pm$0.09          & 86.28$\pm$0.18          & 85.53$\pm$0.30          & -              \\
  \rowcolor{gray!20}
  FedGTA~\cite{li2023fedgta}        & 80.06$\pm$0.63          & 80.59$\pm$0.38          & 79.01$\pm$0.31          & 70.12$\pm$0.10          & 71.57$\pm$0.34          & 69.94$\pm$0.14          & 87.75$\pm$0.01          & 86.80$\pm$0.01          & 87.12$\pm$0.05          & -              \\
  AdaFGL~\cite{li2024adafgl}        & 82.01$\pm$0.51          & 80.09$\pm$0.08          & 79.74$\pm$0.05          & 71.44$\pm$0.27          & 72.34$\pm$0.09          & 70.95$\pm$0.45          & 86.91$\pm$0.28          & 86.97$\pm$0.10          & 86.59$\pm$0.21          & -              \\ 
  \rowcolor{gray!20}
  FedTAD~\cite{zhu2024fedtad}       & 80.31$\pm$0.26          & 80.87$\pm$0.11          & 80.07$\pm$0.15          & 70.34$\pm$0.37          & 69.43$\pm$0.75          & 68.09$\pm$0.69          & 84.00$\pm$0.13          & 84.61$\pm$0.17          & 84.33$\pm$0.18          & -              \\ 
  FedIIH~\cite{wentao2025fediih}    & \underline{84.11$\pm$0.17}          & \underline{81.85$\pm$0.09}          & \underline{83.01$\pm$0.15}          & \underline{72.86$\pm$0.25}          & \underline{76.50$\pm$0.06}          & \underline{73.36$\pm$0.41}          & \underline{87.80$\pm$0.18}          & \underline{87.65$\pm$0.18}          & \underline{87.19$\pm$0.25}         & -              \\ \hline
  \rowcolor{yellow!30}
  FedSSA (Ours)                 & \textbf{84.67$\pm$0.05} & \textbf{82.32$\pm$0.04} & \textbf{84.13$\pm$0.09} & \textbf{73.06$\pm$0.08} & \textbf{77.65$\pm$0.10} & \textbf{74.09$\pm$0.09} & \textbf{88.11$\pm$0.07} & \textbf{87.78$\pm$0.13} & \textbf{87.37$\pm$0.14} & -              \\ \hline
  \rowcolor{gray!50}
  & \multicolumn{3}{c}{Amazon-Computer}                                         & \multicolumn{3}{c}{Amazon-Photo}                                            & \multicolumn{3}{c}{ogbn-arxiv}                                              & Avg.            \\ \cline{2-11} 
  Methods       & 5 Clients               & 10 Clients              & 20 Clients              & 5 Clients               & 10 Clients              & 20 Clients              & 5 Clients               & 10 Clients              & 20 Clients              & All           \\ \hline
  \rowcolor{gray!20}
  Local         & 89.22$\pm$0.13          & 88.91$\pm$0.17          & 89.52$\pm$0.20          & 91.67$\pm$0.09          & 91.80$\pm$0.02          & 90.47$\pm$0.15          & 66.76$\pm$0.07          & 64.92$\pm$0.09          & 65.06$\pm$0.05          & 79.57          \\ \hline
  FedAvg~\cite{mcmahan2017communication}        & 84.88$\pm$1.96          & 79.54$\pm$0.23          & 74.79$\pm$0.24          & 89.89$\pm$0.83          & 83.15$\pm$3.71          & 81.35$\pm$1.04          & 65.54$\pm$0.07          & 64.44$\pm$0.10          & 63.24$\pm$0.13          & 74.58          \\
  \rowcolor{gray!20}
  FedProx~\cite{MLSYS2020_1f5fe839}       & 85.25$\pm$1.27          & 83.81$\pm$1.09          & 73.05$\pm$1.30          & 90.38$\pm$0.48          & 80.92$\pm$4.64          & 82.32$\pm$0.29          & 65.21$\pm$0.20          & 64.37$\pm$0.18          & 63.03$\pm$0.04          & 72.84          \\
  FedPer~\cite{Arivazhagan2019}        & 89.67$\pm$0.34          & 89.73$\pm$0.04          & 87.86$\pm$0.43          & 91.44$\pm$0.37          & 91.76$\pm$0.23          & 90.59$\pm$0.06          & 66.87$\pm$0.05          & 64.99$\pm$0.18          & 64.66$\pm$0.11          & 79.94          \\
  \rowcolor{gray!20}
  GCFL~\cite{NEURIPS2021_9c6947bd}          & 89.07$\pm$0.91          & 90.03$\pm$0.16          & 89.08$\pm$0.25          & 91.99$\pm$0.29          & 92.06$\pm$0.25          & 90.79$\pm$0.17          & 66.80$\pm$0.12          & 65.09$\pm$0.08          & 65.08$\pm$0.04          & 79.90          \\
  FedGNN~\cite{wu2021fedgnn}        & 88.08$\pm$0.15          & 88.18$\pm$0.41          & 83.16$\pm$0.13          & 90.25$\pm$0.70          & 87.12$\pm$2.01          & 81.00$\pm$4.48          & 65.47$\pm$0.22          & 64.21$\pm$0.32          & 63.80$\pm$0.05          & 75.23          \\
  \rowcolor{gray!20}
  FedSage+\cite{NEURIPS2021_34adeb8e}      & 85.04$\pm$0.61          & 80.50$\pm$1.13          & 70.42$\pm$0.85          & 90.77$\pm$0.44          & 76.81$\pm$8.24          & 80.58$\pm$1.15          & 65.69$\pm$0.09          & 64.52$\pm$0.14          & 63.31$\pm$0.20          & 73.47          \\
  FED-PUB~\cite{baek2023personalized}       & \underline{90.74$\pm$0.05} & 90.55$\pm$0.13          & 90.12$\pm$0.09          & 93.29$\pm$0.19          & 92.73$\pm$0.18          & 91.92$\pm$0.12          & 67.77$\pm$0.09          & 66.58$\pm$0.08          & 66.64$\pm$0.12          & 81.59          \\
  \rowcolor{gray!20}
  FedGTA~\cite{li2023fedgta}        & 86.69$\pm$0.18          & 86.66$\pm$0.23          & 85.01$\pm$0.87          & 93.33$\pm$0.12          & 93.50$\pm$0.21          & 92.61$\pm$0.15          & 60.32$\pm$0.04                   & 60.22$\pm$0.09                   & 58.74$\pm$0.14                         & 79.45          \\ 
  AdaFGL~\cite{li2024adafgl}        & 80.20$\pm$0.05          & 83.62$\pm$0.26          & 84.53$\pm$0.23          & 86.69$\pm$0.19          & 89.85$\pm$0.83          & 88.11$\pm$0.05          & 52.73$\pm$0.19          & 51.77$\pm$0.36          & 50.94$\pm$0.08          & 76.97 \\     
  \rowcolor{gray!20}
  FedTAD~\cite{zhu2024fedtad}       & 82.20$\pm$1.20          & 85.50$\pm$0.33          & 83.91$\pm$1.54          & 92.29$\pm$0.39          & 90.59$\pm$0.09          & 89.18$\pm$0.84          & 65.35$\pm$0.14          & 64.06$\pm$0.25          & 64.45$\pm$0.13          & 78.87 \\    
  FedIIH~\cite{wentao2025fediih} & \underline{90.74$\pm$0.13}          & \underline{90.86$\pm$0.23} & \underline{90.44$\pm$0.05} & \underline{93.42$\pm$0.02} & \underline{94.22$\pm$0.08} & \underline{93.55$\pm$0.09} & \underline{70.30$\pm$0.06} & \underline{69.34$\pm$0.02} & \underline{68.65$\pm$0.04} & \underline{83.10} \\ \hline
  \rowcolor{yellow!30}
  FedSSA (Ours)                  & \textbf{91.09$\pm$0.07}          & \textbf{91.30$\pm$0.13} & \textbf{90.62$\pm$0.09} & \textbf{93.86$\pm$0.15} & \textbf{94.62$\pm$0.14} & \textbf{93.76$\pm$0.07} & \textbf{70.86$\pm$0.13} & \textbf{69.47$\pm$0.08} & \textbf{68.77$\pm$0.13} & \textbf{83.53} \\ \hline  
  \end{tabular}
  }
  \end{table*}

\begin{table*}[t]
\centering
\scriptsize
\caption{Comparisons on five \textbf{heterophilic} graph datasets under \textbf{non-overlapping} subgraph partitioning setting. Accuracy (\%) is reported for \textit{Roman-empire} and \textit{Amazon-ratings}, and AUC (\%) is reported for \textit{Minesweeper}, \textit{Tolokers}, and \textit{Questions}.}
\label{table3}
\renewcommand{\arraystretch}{0.9} % 调整行间距
  \scalebox{0.8}{
  \begin{tabular}{lcccccccccc}
  \hline
  \rowcolor{gray!50}
  \textbf{}     & \multicolumn{3}{c}{Roman-empire}                                            & \multicolumn{3}{c}{Amazon-ratings}                                          & \multicolumn{3}{c}{Minesweeper}                                             & -              \\ \cline{2-11} 
  Methods       & 5 Clients               & 10 Clients              & 20 Clients              & 5 Clients               & 10 Clients              & 20 Clients              & 5 Clients               & 10 Clients              & 20 Clients              & -              \\ \hline
  \rowcolor{gray!20}
  Local         & 33.65$\pm$0.13          & 28.42$\pm$0.26          & 23.89$\pm$0.32          & \underline{45.03$\pm$0.31} & \textbf{45.89$\pm$0.19} & \underline{46.02$\pm$0.02} & 71.35$\pm$0.17          & 69.96$\pm$0.16          & 69.31$\pm$0.09          & -              \\ \hline
  FedAvg~\cite{mcmahan2017communication}        & 38.93$\pm$0.32          & 35.43$\pm$0.32          & 32.00$\pm$0.39          & 41.26$\pm$0.53          & 41.66$\pm$0.14          & 42.20$\pm$0.21          & 72.60$\pm$0.08          & 69.68$\pm$0.09          & 71.36$\pm$0.16          & -              \\
  \rowcolor{gray!20}
  FedProx~\cite{MLSYS2020_1f5fe839}       & 27.95$\pm$0.59          & 26.43$\pm$1.41          & 23.12$\pm$0.49          & 36.92$\pm$0.02          & 36.86$\pm$0.14          & 36.96$\pm$0.05          & 71.91$\pm$0.27          & 70.66$\pm$0.20          & 71.50$\pm$0.37          & -              \\
  FedPer~\cite{Arivazhagan2019}        & 20.75$\pm$1.75          & 15.51$\pm$1.13          & 15.45$\pm$2.76          & 36.62$\pm$0.30          & 32.34$\pm$1.01          & 36.96$\pm$0.03          & 58.73$\pm$10.45         & 65.35$\pm$7.02          & 53.80$\pm$11.40         & -              \\
  \rowcolor{gray!20}
  GCFL~\cite{NEURIPS2021_9c6947bd}          & 40.65$\pm$0.14          & 40.51$\pm$0.24          & 37.85$\pm$0.25          & 36.92$\pm$0.05          & 36.86$\pm$0.14          & 36.96$\pm$0.02          & 72.04$\pm$0.13          & 71.88$\pm$0.12          & 69.20$\pm$0.18          & -              \\
  FedGNN~\cite{wu2021fedgnn}        & 30.26$\pm$0.11          & 29.09$\pm$0.13          & 26.60$\pm$0.02          & 36.80$\pm$0.06          & 36.72$\pm$0.00          & 36.45$\pm$0.09          & 72.15$\pm$0.13          & 71.08$\pm$0.07          &  71.71$\pm$0.27    & -              \\
  \rowcolor{gray!20}
  FedSage+\cite{NEURIPS2021_34adeb8e}      & 57.26$\pm$0.00          & 49.07$\pm$0.00          & 38.36$\pm$0.00          & 36.82$\pm$0.00          & 36.71$\pm$0.00          & 37.03$\pm$0.02          & \underline{77.74$\pm$0.00} & 72.80$\pm$0.00    & 69.70$\pm$0.00          & -              \\
  FED-PUB~\cite{baek2023personalized}       & 40.80$\pm$0.26          & 36.77$\pm$0.30          & 32.67$\pm$0.39          & 44.41$\pm$0.41    &  44.85$\pm$0.17    & 45.39$\pm$0.50    & 72.18$\pm$0.02          & 71.69$\pm$0.71          & 71.41$\pm$0.87          & -              \\
  \rowcolor{gray!20}
  FedGTA~\cite{li2023fedgta}        & 61.56$\pm$0.27    & 60.94$\pm$0.19    & 59.65$\pm$0.28    & 41.22$\pm$0.66          & 39.40$\pm$0.44          & 39.24$\pm$0.12          & 73.54$\pm$1.56          & 72.65$\pm$1.21          & 69.63$\pm$4.54          & -              \\
  AdaFGL~\cite{li2024adafgl}        & 67.64$\pm$0.18          & 64.55$\pm$0.09          & 62.42$\pm$0.26          & 41.70$\pm$0.06          & 42.30$\pm$0.00          & 42.59$\pm$0.14          & 73.24$\pm$1.13                  & 70.79$\pm$2.14                   & 71.26$\pm$1.31                         & -          \\ 
  \rowcolor{gray!20}
  FedTAD~\cite{zhu2024fedtad} & 45.26$\pm$0.19          & 44.71$\pm$0.38          & 42.04$\pm$0.13          & 43.59$\pm$0.33          & 43.35$\pm$0.29          & 44.50$\pm$0.26          & 72.39$\pm$0.43                  & 71.99$\pm$0.13                   & 72.74$\pm$0.03                         & -          \\ 
  FedIIH~\cite{wentao2025fediih} & \underline{68.32$\pm$0.05} & \underline{66.44$\pm$0.28} & \underline{64.61$\pm$0.13} & 44.26$\pm$0.24          & 44.24$\pm$0.10          & 45.19$\pm$0.04          & 74.29$\pm$0.02    & \underline{73.23$\pm$0.04} & \underline{72.81$\pm$0.02} &   -             \\ \hline
  \rowcolor{yellow!30}
FedSSA (Ours) & \textbf{68.67$\pm$0.10} & \textbf{66.81$\pm$0.09} & \textbf{65.14$\pm$0.15} & \textbf{45.18$\pm$0.14}          & \underline{45.11$\pm$0.15}          & \textbf{46.13$\pm$0.05}          & \textbf{82.26$\pm$0.14}    & \textbf{82.16$\pm$0.08} & \textbf{82.60$\pm$0.11} & -               \\ \hline
\rowcolor{gray!50}
                & \multicolumn{3}{c}{Tolokers}                                                & \multicolumn{3}{c}{Questions}                                               & \multicolumn{4}{c}{Avg.}                                                                     \\ \cline{2-11} 
  Methods       & 5 Clients               & 10 Clients              & 20 Clients              & 5 Clients               & 10 Clients              & 20 Clients              & 5 Clients               & 10 Clients              & 20 Clients              & All            \\ \hline
  \rowcolor{gray!20}
  Local         & 67.81$\pm$0.17          & 70.04$\pm$0.23          & 62.34$\pm$0.67          & 66.73$\pm$0.57          & 57.96$\pm$0.10          & 60.00$\pm$0.21          & 56.91                   & 54.45                   & 52.31                   & 54.56          \\ \hline
  FedAvg~\cite{mcmahan2017communication}        & 60.74$\pm$0.31          & 54.73$\pm$0.50          & 56.36$\pm$0.39          & 65.68$\pm$0.23          & 58.91$\pm$0.22          & 60.33$\pm$0.15          & 55.84                   & 52.08                   & 52.45                   & 53.46          \\
  \rowcolor{gray!20}
  FedProx~\cite{MLSYS2020_1f5fe839}       & 42.90$\pm$0.24          & 41.15$\pm$0.22          & 40.42$\pm$0.62          & 47.36$\pm$0.38          & 45.46$\pm$0.34          & 46.83$\pm$0.11          & 45.41                   & 44.11                   & 43.77                   & 44.43          \\
  FedPer~\cite{Arivazhagan2019}        & 46.61$\pm$9.88          & 54.97$\pm$13.23         & 44.82$\pm$11.61         & 58.38$\pm$9.39          & 59.40$\pm$9.71          & 62.32$\pm$1.56          & 44.22                   & 45.51                   & 42.67                   & 44.13          \\
  \rowcolor{gray!20}
  GCFL~\cite{NEURIPS2021_9c6947bd}          & 64.39$\pm$1.17          & 59.90$\pm$0.85          & 58.82$\pm$0.70          & 60.51$\pm$1.18          & 59.85$\pm$0.16          & 60.31$\pm$0.48          & 54.90                   & 53.80                   & 52.63                   & 53.78          \\
  FedGNN~\cite{wu2021fedgnn}        & 43.10$\pm$0.27          & 41.57$\pm$0.07          & 40.70$\pm$0.74          & 47.55$\pm$0.02          & 45.65$\pm$0.12          & 47.39$\pm$0.13          & 45.97                   & 44.82                   & 44.57                   & 45.12          \\
  \rowcolor{gray!20}
  FedSage+\cite{NEURIPS2021_34adeb8e}      & \underline{75.06$\pm$0.00} & 71.31$\pm$0.00          &  69.73$\pm$0.00    & 64.95$\pm$0.00          &  65.06$\pm$0.00    & 59.33$\pm$0.00          & 62.37            &  58.99             & 54.83                   &  58.73    \\
  FED-PUB~\cite{baek2023personalized}       & 70.88$\pm$0.58          & \underline{72.46$\pm$0.68} & 65.26$\pm$0.59          &  67.71$\pm$3.99    & 59.64$\pm$0.52          & 62.48$\pm$2.92    & 59.20                   & 57.08                   &  55.44             & 57.24          \\
  \rowcolor{gray!20}
  FedGTA~\cite{li2023fedgta}        & 60.83$\pm$0.45          & 55.18$\pm$1.20          & 57.89$\pm$1.61          & 65.56$\pm$1.91          & 58.29$\pm$1.57          & 61.70$\pm$0.35          & 60.54                   & 57.29                   & 57.62                   & 58.49          \\
  AdaFGL~\cite{li2024adafgl}        & 59.26$\pm$2.18          & 54.78$\pm$2.12          & 56.61$\pm$2.93          & 64.23$\pm$2.09          & 58.82$\pm$1.14          & 62.84$\pm$0.49          & 61.21                   & 58.25                   & 59.14                         & 59.54          \\ 
  \rowcolor{gray!20} 
  FedTAD~\cite{zhu2024fedtad} & 60.91$\pm$0.25          & 53.39$\pm$1.73          & 56.47$\pm$1.58          & 68.89$\pm$1.20          & 58.44$\pm$1.06          & 61.51$\pm$1.56          & 58.21                  & 54.38                   & 55.45                         & 56.01         \\
  FedIIH~\cite{wentao2025fediih} &  71.09$\pm$0.26    & 71.32$\pm$0.09    & \underline{70.30$\pm$0.10} & \underline{68.32$\pm$0.03} & \underline{67.99$\pm$0.09} & \underline{65.40$\pm$0.07} & \underline{65.26}          & \underline{64.64}          & \underline{63.66}          & \underline{64.52} \\ \hline
  \rowcolor{yellow!30}
FedSSA (Ours) & \textbf{75.82$\pm$0.05} & \textbf{73.96$\pm$0.10} & \textbf{72.29$\pm$0.11} & \textbf{69.51$\pm$0.15}          & \textbf{68.69$\pm$0.18}          & \textbf{65.74$\pm$0.07}          & \textbf{68.29}    & \textbf{67.35} & \textbf{66.10} & \textbf{67.34}               \\ \hline
      \end{tabular}
       }
      \end{table*}

\subsection{Ablation Study}
To shed light on the contributions of components in our FedSSA, we perform ablation studies on \textit{Cora} and \textit{Tolokers} datasets, which are shown in \cref{fig-ablation}. Specifically, we employ `w/o semantic' and `w/o structural' to represent the reduced methods by removing `sharing semantic knowledge' and `sharing structural knowledge', respectively. It can be observed that the performance decreases when any component is removed, which demonstrates that each component contributes significantly. For example, the accuracies on \textit{Cora} dataset are significantly decreased by more than 5\% when both components are disabled. Ablation studies on other datasets are presented in~\cref{ap_ablation_study}.

\begin{figure}[]
  \begin{center}
    \centerline{\includegraphics[width=6.5cm]{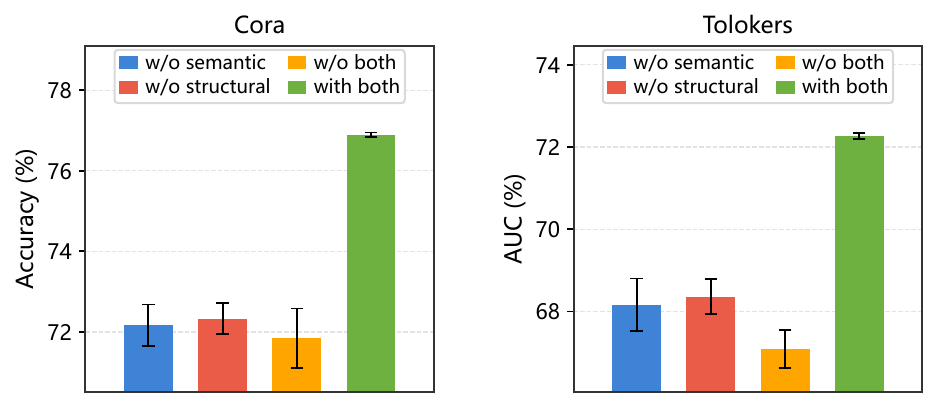}}
    \caption{Ablation studies under overlapping partitioning setting with 30 clients.}
     \vskip -0.5in
    \label{fig-ablation}
  \end{center}
\end{figure}

% \begin{table}[t]
%   \centering
%   \scriptsize
%   \caption{Ablation studies under overlapping with 30 clients.}
%   \label{table5}
%   \renewcommand{\arraystretch}{0.9} % 调整行间距
%   \scalebox{0.98}{
%   \begin{tabular}{cccc}
%   \hline
%   \rowcolor{gray!50}
%            Semantic  & Structural & Cora & Tolokers \\ \hline
%   \textcolor{teal}{\CheckmarkBold}  & \textcolor{red}{\XSolidBrush}      & 72.33$\pm$0.39 ($\downarrow$ 4.56)                                  & 68.36$\pm$0.42 ($\downarrow$ 3.91)   \\ \hline
%   \textcolor{red}{\XSolidBrush} & \textcolor{teal}{\CheckmarkBold}       & 72.17$\pm$0.51 ($\downarrow$ 4.72)                                 & 68.16$\pm$0.64 ($\downarrow$ 4.11) \\ \hline
%   \textcolor{red}{\XSolidBrush} & \textcolor{red}{\XSolidBrush}      & 71.84$\pm$0.74 ($\downarrow$ 5.05)                                 & 67.09$\pm$0.46 ($\downarrow$ 5.18)                   \\ \hline
%   \textcolor{teal}{\CheckmarkBold}  & \textcolor{teal}{\CheckmarkBold}       & \textbf{76.89$\pm$0.06}                                             & \textbf{72.27$\pm$0.07}                                        \\ \hline
%   \end{tabular}
%   }
%   \vskip -0.3in
% \end{table}

\subsection{Convergence Curves}
\label{Convergence_Curves}
As shown in \cref{fig5}, the convergence curves of our FedSSA exhibit small fluctuations, which validates its strong stability. Moreover, our FedSSA converges quickly within only a few communication rounds, which indicates its efficiency in real-world applications. For example, our method achieves convergence in an average of 60 communication rounds, while the compared baseline methods (\textit{e.g.}, GCFL) require an average of 90 communication rounds. This superior performance can be attributed to our proposed semantic and structural knowledge sharing, which effectively mitigates the impact of heterogeneity among clients, thereby accelerating convergence. More convergence curves are provided in~\cref{ap_convergence_curves}.

\begin{figure}[t]
  \centering
  \subfloat[\footnotesize{\textit{ogbn-arxiv}}]{\includegraphics[width=0.45\columnwidth]{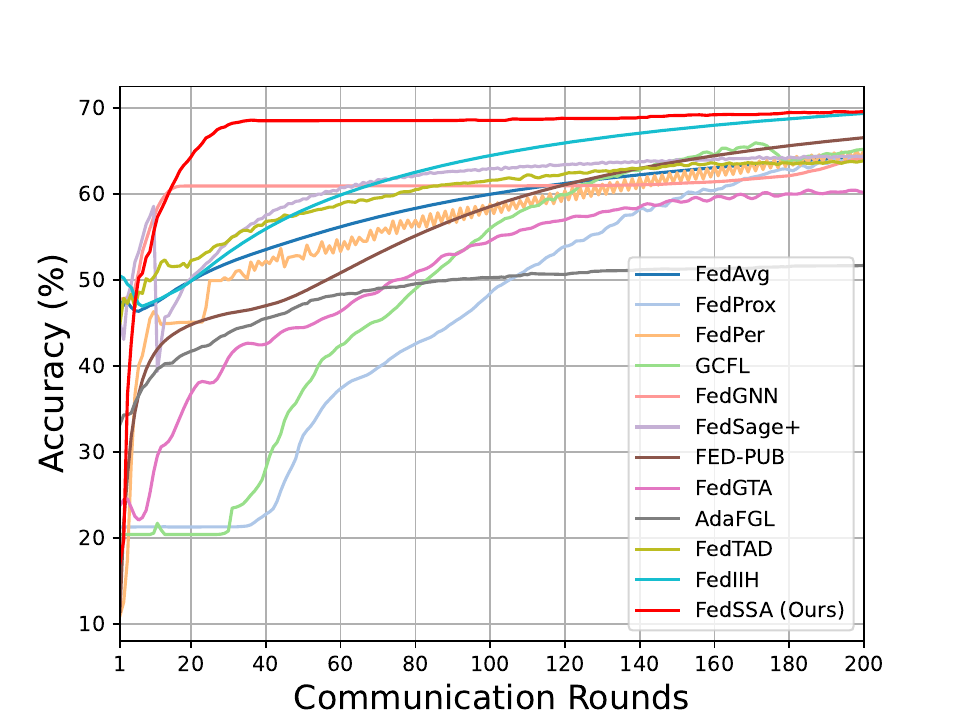}\label{fig5_1}}
  \hfill
  \subfloat[\footnotesize{\textit{Minesweeper}}]{\includegraphics[width=0.45\columnwidth]{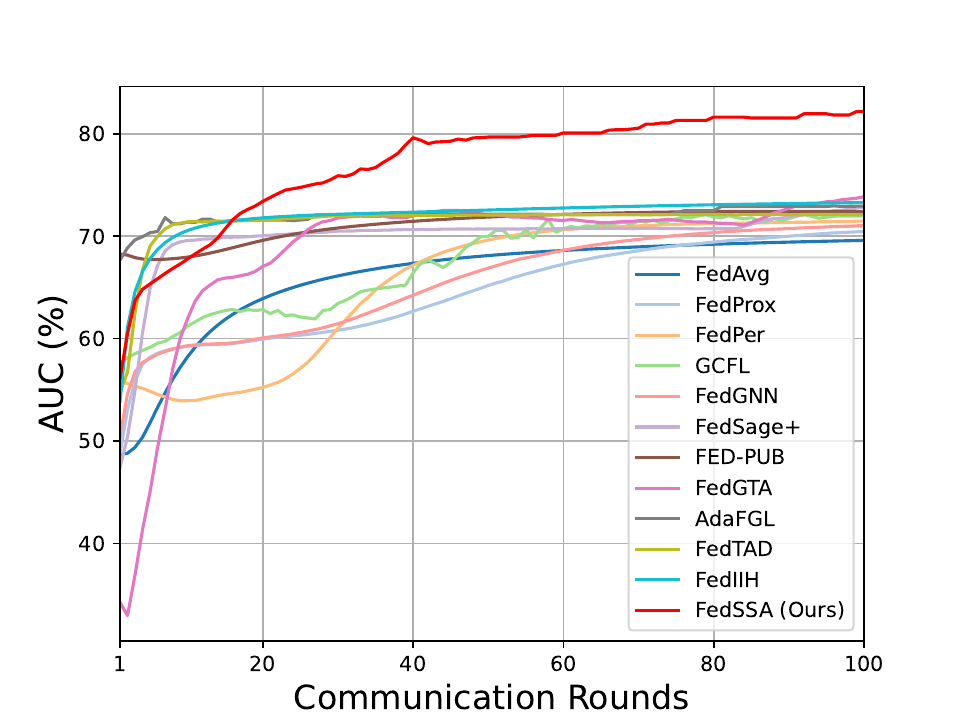}\label{fig5_2}}
  \vskip -0.1in
  \caption{Convergence curves on two datasets under non-overlapping partitioning setting with 10 clients.}
  \vskip -0.1in
  \label{fig5}
\end{figure}

\subsection{Sensitivity Analysis on Hyperparameters}
Here we perform a detailed sensitivity analysis of hyperparameters involved in our proposed FedSSA. Since our FedSSA includes four key hyperparameters (\textit{i.e.}, the number of node clusters $K_\mathrm{node}$, the number of structural clusters $K_\mathrm{struct}$, and regularization hyperparameters $\lambda_1$ and $\lambda_2$), we plot accuracy curves with variance bars under different values of hyperparameters on \textit{Cora} dataset. As shown in \cref{fig6}, the performance variations under different values of hyperparameters are small, which validates that FedSSA is not sensitive to these hyperparameters. More sensitivity analyses are provided in Appendix~\ref{ap_sensitivity_analysis}.

\begin{figure}[]
  \centering
  \subfloat[\footnotesize{$K_\mathrm{node}$}]{\includegraphics[width=0.4\columnwidth]{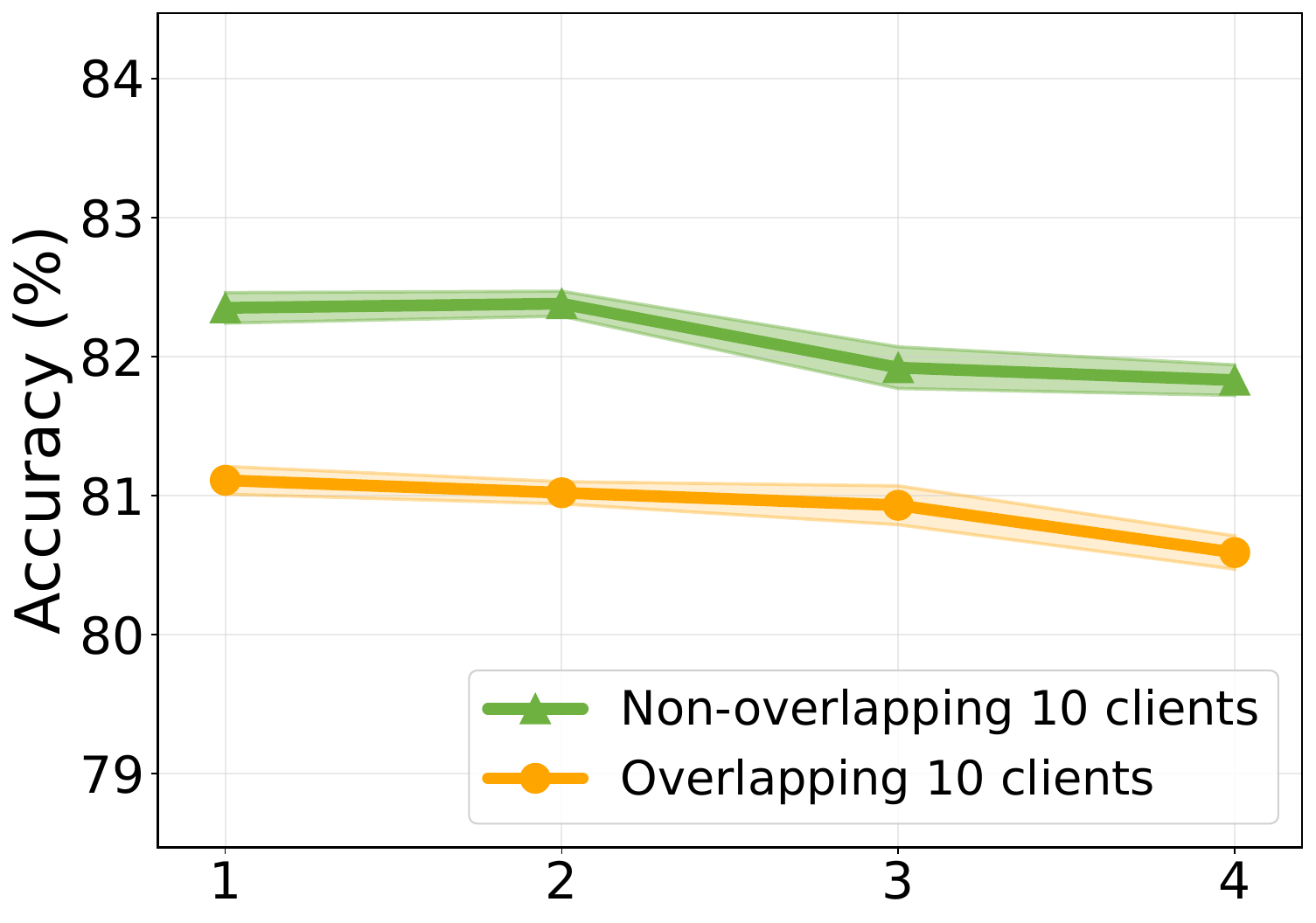}\label{fig6_1}}
  \hfill
  \subfloat[\footnotesize{$K_\mathrm{struct}$}]{\includegraphics[width=0.4\columnwidth]{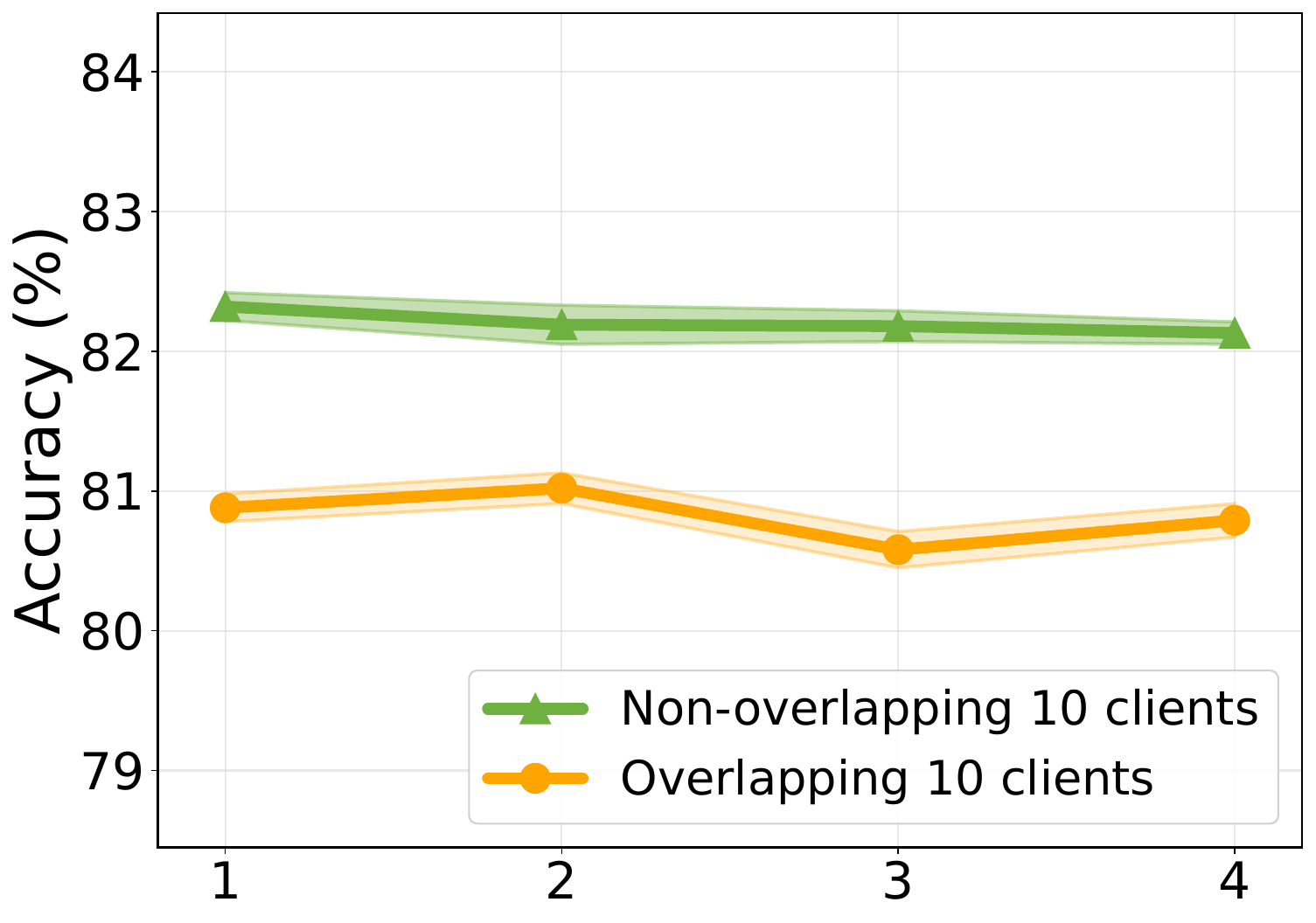}\label{fig6_2}}
  \hfill
  \subfloat[\footnotesize{$\lambda_1$}]{\includegraphics[width=0.4\columnwidth]{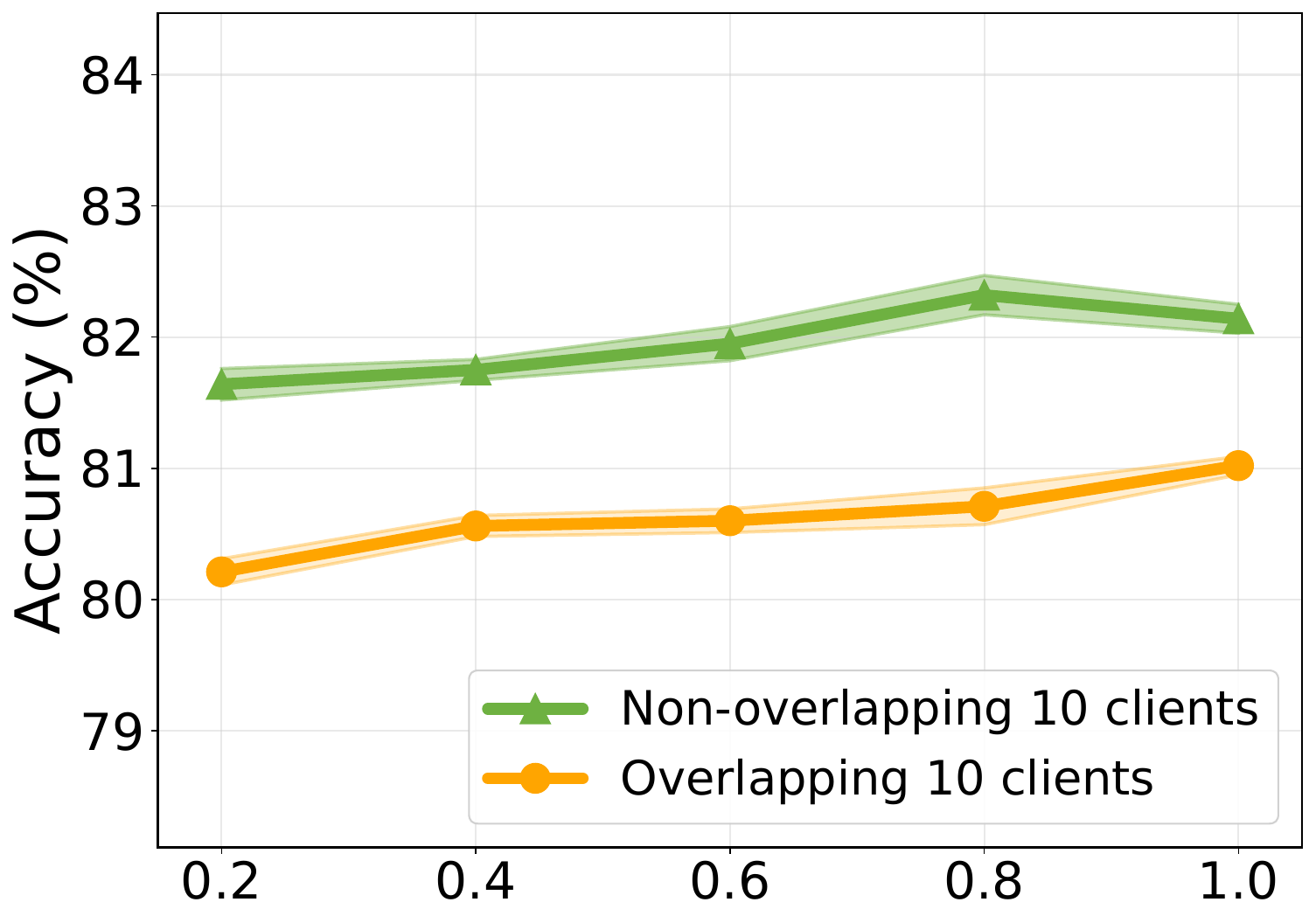}\label{fig6_3}}
  \hfill
    \subfloat[\footnotesize{$\lambda_2$}]{\includegraphics[width=0.4\columnwidth]{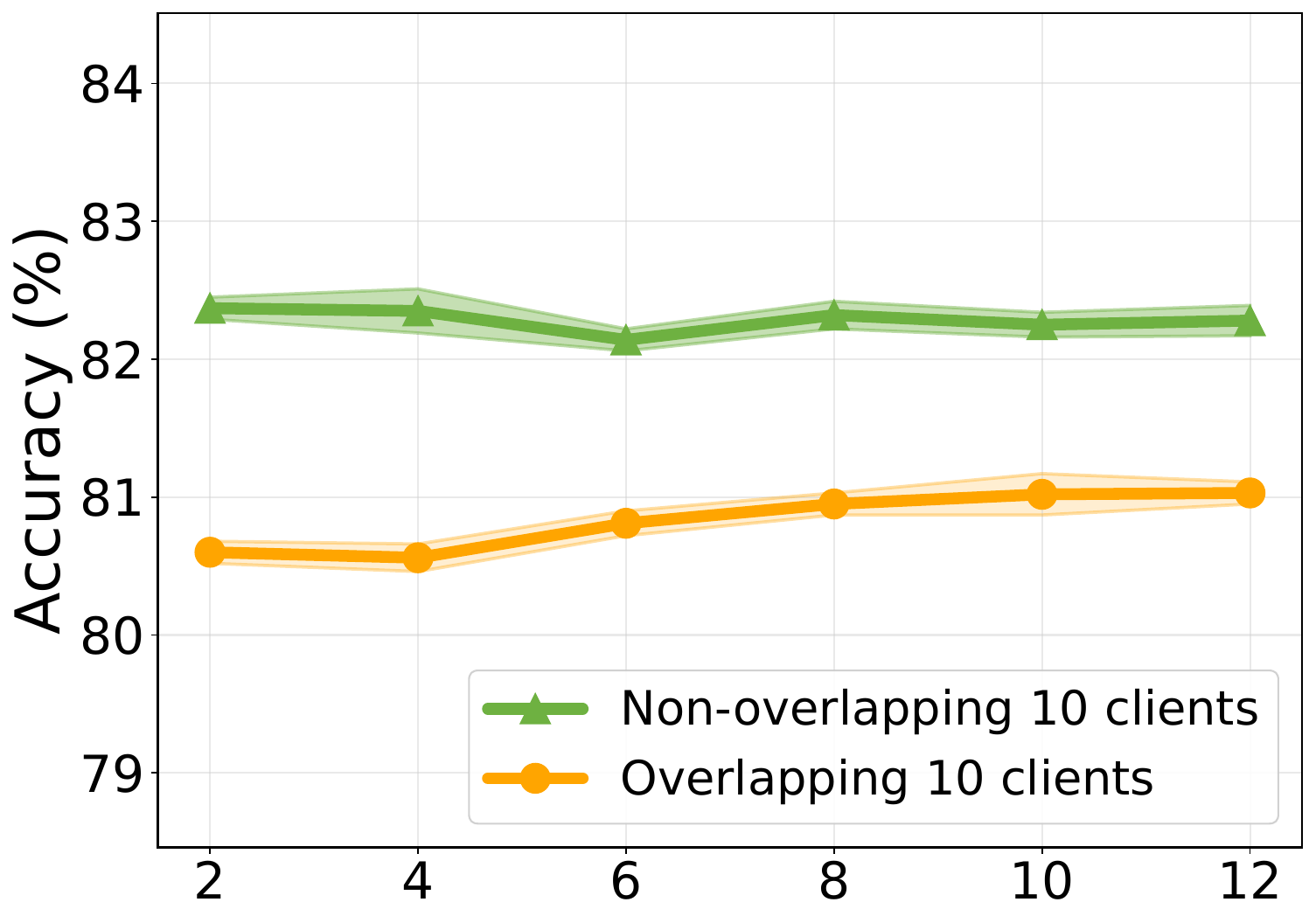}\label{fig6_4}}
  \caption{Accuracy curves with standard deviation bands on \textit{Cora} dataset under different values of $K_\mathrm{node}$, $K_\mathrm{struct}$, $\lambda_1$, and $\lambda_2$.}
  \vskip -0.3in
  \label{fig6}
\end{figure}

\subsection{Case Study}
To illustrate the effectiveness of FedSSA in mitigating heterogeneity, we conduct case studies on two datasets. Specifically, we first visualize semantic representations of various clients obtained from `w/o semantic' and `with semantic' (\textit{i.e.}, FedSSA without/with semantic knowledge sharing) by using t-SNE~\cite{van2008visualizing}. Subsequently, we plot spectral properties captured by local models under `w/o structural' and `with structural' (\textit{i.e.}, FedSSA without/with structural knowledge sharing). As shown in \cref{fig-case-study}, the 2D projections of representations obtained from `with semantic' show more compact clusters when compared with `w/o semantic', which demonstrates the effectiveness of FedSSA in mitigating node feature heterogeneity. Moreover, spectral properties obtained from `with structural' align more closely to those of cluster-level when compared with `w/o structural', which validates the effectiveness of FedSSA in addressing structural heterogeneity. More case studies are provided in Appendix~\ref{ap_case_study}.

\begin{figure}[]
\centering
  \subfloat[\footnotesize{semantic representations of clients}]{\includegraphics[width=0.87\columnwidth]{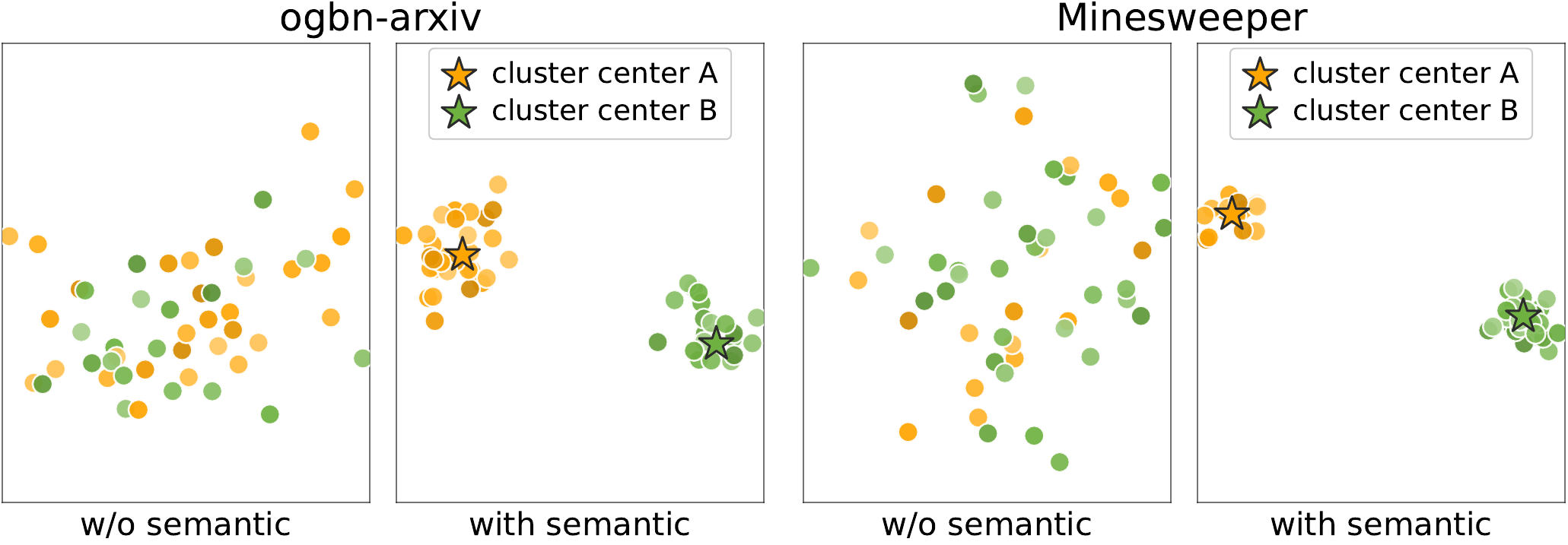}\label{fig-case-study_1}}
  \hfill
  \subfloat[\footnotesize{spectral properties of local models}]{\includegraphics[width=0.72\columnwidth]{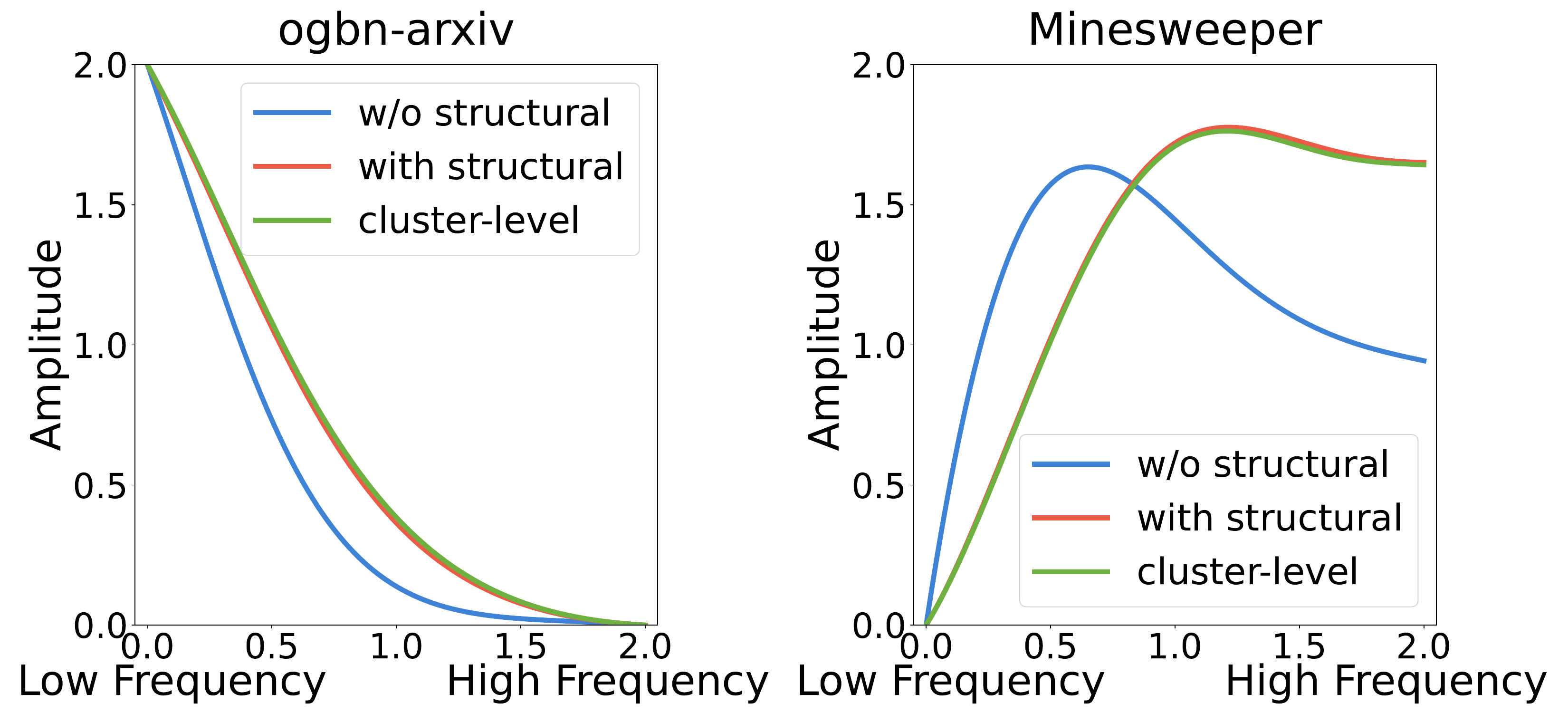}\label{fig-case-study_2}}
  \caption{Case studies on two datasets under overlapping partitioning setting with 50 clients.}
  \vskip -0.3in
  \label{fig-case-study}
\end{figure}

\section{Conclusion}\label{conclusion}
In this paper, we propose a novel graph \underline{\textbf{Fed}}erated learning method via \underline{\textbf{S}}emantic and \underline{\textbf{S}}tructural \underline{\textbf{A}}lignment (FedSSA) to address heterogeneity in Graph Federated Learning (GFL). Specifically, instead of simply performing the weighted aggregation of model parameters, our proposed FedSSA enforces class-wise semantic knowledge sharing and structural knowledge sharing. On one hand, we minimize the divergence between local distributions and cluster-level distributions to mitigate the heterogeneity in node features. On the other hand, we align the spectral characteristics of local spectral GNNs with those of cluster-level spectral GNNs to address the heterogeneity in structural topologies. By separately addressing two types of heterogeneity, our proposed FedSSA achieves strong performance on eleven datasets and outperforms the second-best method by a large margin of 2.82\% in terms of classification accuracy.

% \section*{Acknowledgements}

% \textbf{Do not} include acknowledgements in the initial version of the paper

% If a paper is accepted, the final camera-ready version can (and usually should)
% include acknowledgements. Such acknowledgements should be placed at the end of
% the section, in an unnumbered section that does not count towards the paper
% page limit. Typically, this will include thanks to reviewers who gave useful
% comments, to colleagues who contributed to the ideas, and to funding agencies
% and corporate sponsors that provided financial support.

\section*{Impact Statement}
This work advances the development of Graph Federated Learning (GFL) methods that satisfy both privacy preservation and model effectiveness, thereby reducing the risks of privacy leakage in distributed graph learning scenarios. We encourage continued research into legally compliant and reliable GFL approaches that respect individual rights and intellectual property while maintaining robustness across diverse applications. By facilitating the broader deployment of GFL methods that can adapt to evolving legal and ethical standards, this research contributes to ensuring that such technologies remain trustworthy and socially beneficial.

\bibliography{example_paper}
\bibliographystyle{icml2026}

%%%%%%%%%%%%%%%%%%%%%%%%%%%%%%%%%%%%%%%%%%%%%%%%%%%%%%%%%%%%%%%%%%%%%%%%%%%%%%%
%%%%%%%%%%%%%%%%%%%%%%%%%%%%%%%%%%%%%%%%%%%%%%%%%%%%%%%%%%%%%%%%%%%%%%%%%%%%%%%
% APPENDIX
%%%%%%%%%%%%%%%%%%%%%%%%%%%%%%%%%%%%%%%%%%%%%%%%%%%%%%%%%%%%%%%%%%%%%%%%%%%%%%%
%%%%%%%%%%%%%%%%%%%%%%%%%%%%%%%%%%%%%%%%%%%%%%%%%%%%%%%%%%%%%%%%%%%%%%%%%%%%%%%
\newpage
\appendix
\onecolumn
\section{Notations}\label{ap_notations}
For convenience, key notations used throughout the paper are summarized in \cref{tab-notations}.

\begin{table}[t]
\centering
\scriptsize
\caption{Key notations used throughout the paper.}
\label{tab-notations}
\renewcommand{\arraystretch}{1.1}
\begin{tabular}{ll}
\toprule
\textbf{Notation} & \textbf{Description} \\
\hline
$\mathcal{M}$ & Set of clients, $\mathcal{M} = \{1, 2, \cdots, M\}$. \\
$\mathcal{G}_m = \langle \mathcal{V}_m, \mathcal{E}_m \rangle$ & Local graph on client $m$ with nodes $\mathcal{V}_m$ and edges $\mathcal{E}_m$. \\
$n_m$, $e_m$ & Number of nodes and edges on client $m$. \\
$d$, $C$ & Feature dimension and number of classes. \\
$\mathbf{X}_m \in \mathbb{R}^{n_m \times d}$ & Node feature matrix of client $m$. \\
$\mathbf{A}_m \in \mathbb{R}^{n_m \times n_m}$ & Adjacency matrix of client $m$. \\
$\mathbf{Y}_m$ & Label matrix of client $m$. \\
$\mathbf{L}_m$ & Symmetric normalized Laplacian of client $m$. \\
$\mathbf{U}_m$, $\mathbf{\Lambda}_m$ & Eigenvectors and eigenvalues of $\mathbf{L}_m$. \\
$\tilde{\mathbf{X}}_m$ & Graph Fourier transform of $\mathbf{X}_m$. \\
$K_\mathrm{node}$ & Number of node-level clusters. \\
$K_\mathrm{struct}$ & Number of structural clusters. \\
$T$ & Total number of communication rounds. \\
$\lambda_1$, $\lambda_2$ & Hyperparameters for regularization weights. \\
$K$ & Number of polynomial filter orders. \\
\bottomrule
\end{tabular}
\end{table}

\section{Derivation of Eq.~\eqref{eq1}}\label{appendix-elbo1}
The log marginal likelihood $\log p(\mathbf{X}_m, \mathbf{Y}_m)$ can be derived as
\begin{equation}\label{eq-log-marginal}
  \begin{aligned}
\log p(\mathbf{X}_m, \mathbf{Y}_m)
&= \log \int p(\mathbf{X}_m, \mathbf{Y}_m, \mathbf{Z}_m) \, d\mathbf{Z}_m\\
&=\log \int p(\mathbf{X}_m | \mathbf{Z}_m, \mathbf{Y}_m)\, p(\mathbf{Z}_m)\, p(\mathbf{Y}_m) \, d\mathbf{Z}_m\\
&=\log \int p(\mathbf{X}_m | \mathbf{Z}_m, \mathbf{Y}_m)\, p(\mathbf{Z}_m)\, p(\mathbf{Y}_m)\,
\frac{q(\mathbf{Z}_m | \mathbf{X}_m, \mathbf{Y}_m)}{q(\mathbf{Z}_m | \mathbf{X}_m, \mathbf{Y}_m)} \, d\mathbf{Z}_m\\
&=\log \mathbb{E}_{\mathbf{Z}_m \sim q(\mathbf{Z}_m | \mathbf{X}_m, \mathbf{Y}_m)}
\frac{p(\mathbf{X}_m | \mathbf{Z}_m, \mathbf{Y}_m)\, p(\mathbf{Z}_m)\, p(\mathbf{Y}_m)}
{q(\mathbf{Z}_m | \mathbf{X}_m, \mathbf{Y}_m)}\\
&\ge \mathbb{E}_{\mathbf{Z}_m \sim q(\mathbf{Z}_m | \mathbf{X}_m, \mathbf{Y}_m)}
\log \frac{p(\mathbf{X}_m | \mathbf{Z}_m, \mathbf{Y}_m)\, p(\mathbf{Z}_m)\, p(\mathbf{Y}_m)}
{q(\mathbf{Z}_m | \mathbf{X}_m, \mathbf{Y}_m)}\\
&\triangleq \mathcal{L}_\mathrm{ELBO}(\mathbf{X}_m, \mathbf{Y}_m),
  \end{aligned}
\end{equation}
where the inequality is obtained by Jensen's inequality. Afterwards, we can further derive the Evidence Lower BOund (ELBO) $\mathcal{L}_\mathrm{ELBO}(\mathbf{X}_m, \mathbf{Y}_m)$ as
\begin{equation}\label{eq-elbo-derivation1}
  \begin{aligned}
\mathcal{L}_\mathrm{ELBO}(\mathbf{X}_m, \mathbf{Y}_m) &= \mathbb{E}_{\mathbf{Z}_m \sim q(\mathbf{Z}_m | \mathbf{X}_m, \mathbf{Y}_m)} \log p(\mathbf{X}_m | \mathbf{Z}_m, \mathbf{Y}_m) + \mathbb{E}_{\mathbf{Z}_m \sim q(\mathbf{Z}_m | \mathbf{X}_m, \mathbf{Y}_m)}[\log p(\mathbf{Z}_m) + \log p(\mathbf{Y}_m)]\\
&\quad - \mathbb{E}_{\mathbf{Z}_m \sim q(\mathbf{Z}_m | \mathbf{X}_m, \mathbf{Y}_m)} \log q(\mathbf{Z}_m | \mathbf{X}_m, \mathbf{Y}_m),\\
&=\mathbb{E}_{\mathbf{Z}_m \sim q(\mathbf{Z}_m | \mathbf{X}_m, \mathbf{Y}_m)} \log p(\mathbf{X}_m | \mathbf{Z}_m, \mathbf{Y}_m) + \mathbb{E}_{\mathbf{Z}_m \sim q(\mathbf{Z}_m | \mathbf{X}_m, \mathbf{Y}_m)}\log p(\mathbf{Y}_m)\\
&\quad + \mathbb{E}_{\mathbf{Z}_m \sim q(\mathbf{Z}_m | \mathbf{X}_m, \mathbf{Y}_m)}\log p(\mathbf{Z}_m) - \mathbb{E}_{\mathbf{Z}_m \sim q(\mathbf{Z}_m | \mathbf{X}_m, \mathbf{Y}_m)} \log q(\mathbf{Z}_m | \mathbf{X}_m, \mathbf{Y}_m),\\
&=\mathbb{E}_{\mathbf{Z}_m \sim q(\mathbf{Z}_m | \mathbf{X}_m, \mathbf{Y}_m)} \log p(\mathbf{X}_m | \mathbf{Z}_m, \mathbf{Y}_m) + \log p(\mathbf{Y}_m) - D_\mathrm{KL}\big(q(\mathbf{Z}_m|\mathbf{X}_m, \mathbf{Y}_m) \parallel p(\mathbf{Z}_m) \big).
  \end{aligned}
\end{equation}
Finally, we obtain Eq.~\eqref{eq1} in the main paper.

\section{Details of Variational Graph AutoEncoder}\label{appendix-vgae}
In this section, we provide the architectural details of Variational Graph AutoEncoder (VGAE), which is proposed to infer the class-conditional latent distribution $q(\mathbf{Z}_m^c)$ in Eq.~\eqref{eq3}.

\subsection{Architecture of VGAE}
Here we present the architecture of our proposed VGAE. Since VGAE is deployed on each client, we take the $m$-th client as an illustrative example. Recall that the local graph $\mathcal{G}_m$ on client $m$ consists of a node feature matrix $\mathbf{X}_m$ and the corresponding label matrix $\mathbf{Y}_m$. As illustrated in \cref{fig-ap-vgae}, our VGAE is made up of an encoder and a decoder. The encoder approximates the posterior distribution $q(\mathbf{Z}_m | \mathbf{X}_m, \mathbf{Y}_m)$ by inferring latent variables $\mathbf{Z}_m$ from node features $\mathbf{X}_m$ conditioned on labels $\mathbf{Y}_m$. Specifically, we concatenate the one-hot encoding of node labels with node features as input to the encoder. Since class labels are explicitly incorporated, we can readily obtain the class-conditional distribution $q(\mathbf{Z}_m^c)$ for any given class $c$. Note that we employ reparameterization trick~\cite{kingma2013auto} to sample $\mathbf{Z}_m$. Meanwhile, we employ a decoder to approximate the generative distribution $p(\mathbf{X}_m|\mathbf{Y}_m, \mathbf{Z}_m)$. Specifically, the decoder is implemented as an inner product layer following~\cite{kipf2016variational}.

\begin{figure*}[t]
  \begin{center}
    \centerline{\includegraphics[width=10cm]{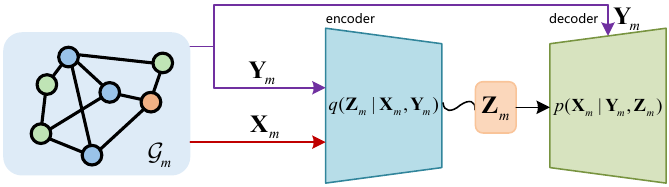}}
    \caption{The architecture of our proposed VGAE, which comprises an encoder to infer latent variables and a decoder for graph reconstruction.}
    \label{fig-ap-vgae}
  \end{center}
\end{figure*}

\subsection{ELBO Formulation}
As ELBO is defined in Eq.~\eqref{eq1}, we provide its detailed formulation as follows. First, ELBO in Eq.~\eqref{eq1} is composed of three terms. The first term (\textit{i.e.}, $\mathbb{E}_{\mathbf{Z}_m \sim q(\mathbf{Z}_m | \mathbf{X}_m, \mathbf{Y}_m)} \log p(\mathbf{X}_m | \mathbf{Z}_m, \mathbf{Y}_m)$) is the expected log-likelihood, which can be implemented by a reconstruction loss. The second term (\textit{i.e.}, $\log p(\mathbf{Y}_m)$) is the log prior of labels, which can be computed by the class distribution of labeled nodes. The third term (\textit{i.e.}, $- D_\mathrm{KL}\big(q(\mathbf{Z}_m|\mathbf{X}_m, \mathbf{Y}_m) \parallel p(\mathbf{Z}_m) \big)$) is the negative KL divergence between the variational posterior and the prior of latent variables. Since we assume that both the variational posterior and prior follow Gaussian distributions, this term can be computed in closed form~\cite{kingma2013auto}. Finally, the overall ELBO can be optimized by maximizing these three terms.

\section{Convergence Analysis of FedSSA}\label{appendix:convergence-proof}
In this section, we first provide formal and detailed mathematical definitions for assumptions summarized in~\cref{asm:combined}. Second, we provide the complete proof of Theorem~\ref{thm:FedSSA-convergence}, which consists of three parts:  (i) error bound for semantic knowledge sharing, (ii) error bound for structural knowledge sharing, and (iii) overall convergence analysis.

\subsection{Assumptions}\label{appendix:intro-assumptions}
To provide a rigorous theoretical foundation, we further decompose the summarized assumption in~\cref{asm:combined} into the following detailed assumptions, which are commonly adopted in the convergence analysis of federated learning~\cite{hu2024fedcross, yuintegrating, valdeiravertical}.

\begin{assumption}[Smoothness]\label{asm:smoothness}
The population risk function $F(\vecw)$ is $L_F$-smooth, namely, for all local models $\vecw, \vecw' \in \mathcal{W}$,
\[
\|\nabla F(\vecw) - \nabla F(\vecw')\|_2 \leq L_F \|\vecw - \vecw'\|_2.
\]
\end{assumption}

\begin{assumption}[Strong Convexity]\label{asm:strong-convexity}
The population risk function $F(\vecw)$ is $\lambda_F$-strongly convex, namely, for all local models $\vecw, \vecw' \in \mathcal{W}$,
\[
F(\vecw') \geq F(\vecw) + \langle \nabla F(\vecw), \vecw' - \vecw \rangle + \frac{\lambda_F}{2} \|\vecw' - \vecw\|_2^2.
\]
\end{assumption}

\begin{assumption}[Bounded Intra-cluster Heterogeneity]\label{asm:bounded-heterogeneity}
For semantic knowledge sharing, the intra-cluster distribution divergence is bounded.  Specifically, for any cluster $\mathcal{S}_i^c$ and any clients $m, n \in \mathcal{S}_i^c$, we assume
\[
\|\boldsymbol{\mu}_m^c - \boldsymbol{\mu}_n^c\|_2 \leq \delta_\mu, \quad \|\boldsymbol{\Sigma}_m^c - \boldsymbol{\Sigma}_n^c\|_\mathrm{F} \leq \delta_\Sigma. 
\]
Moreover, let $\sigma_{\min}^2 \triangleq \min_{c\in\{1, 2, \cdots,C\},\ i\in\{1, 2, \cdots,K_\mathrm{node}\}} \lambda_{\min}(\boldsymbol{\Sigma}_i^c)$ denote a uniform lower bound on the minimum eigenvalue of the cluster-level covariances $\boldsymbol{\Sigma}_i^c$ (defined in Eq.~\eqref{eq6}), and assume $\sigma_{\min}^2>0$ and
\[
\delta_\Sigma + \delta_\mu^2 \leq \frac{\sigma_{\min}^2}{2}.
\]
For structural knowledge sharing, the intra-cluster spectral energy divergence is bounded.  Specifically, for any cluster $\mathcal{T}_j$ and any clients $m, n \in \mathcal{T}_j$, we assume
\[
d_{\mathrm{Chordal}}(\mathbf{Q}_m, \mathbf{Q}_n) \leq \epsilon_U.
\]
\end{assumption}

\begin{assumption}[Bounded Jacobians of Variational Parameters]\label{asm:bounded-variational}
The variational parameters are differentiable functions of local model parameters $\vecw \in \mathcal{W}$. In particular, for each class $c$ and client $m$, $q(\mathbf{Z}_m^c)=\mathcal{N}(\boldsymbol{\mu}_m^c(\vecw),\boldsymbol{\Sigma}_m^c(\vecw))$. We assume there exists a constant $L_q>0$ such that for all $\vecw\in\mathcal{W}$,
\[
\|\nabla_{\vecw} \boldsymbol{\mu}_m^c(\vecw)\|_2 \leq L_q, \qquad \|\nabla_{\vecw} \boldsymbol{\Sigma}_m^c(\vecw)\|_\mathrm{F} \leq L_q,
\]
where $\nabla_{\vecw}\boldsymbol{\mu}_m^c(\vecw)$ denotes the Jacobian of $\boldsymbol{\mu}_m^c(\vecw)$, and $\nabla_{\vecw}\boldsymbol{\Sigma}_m^c(\vecw)$ denotes the Jacobian of $\boldsymbol{\Sigma}_m^c(\vecw)$.
\end{assumption}

\begin{assumption}[Bounded and Lipschitz Spectral Coefficients]\label{asm:bounded-filter}
The learnable coefficients of spectral GNNs satisfy the following properties.
\begin{itemize}
\item (Boundedness) For all $m \in \mathcal{M}$ and $k \in \{0, 1, \cdots, K\}$, there exists $w_{\max}>0$ such that
\[
|w_m^k|\le w_{\max}.
\]
\item (Lipschitz Dependence) There exists a constant $L_w>0$ such that for any clients $m,n$ and any filter order $k\in\{0,1,\cdots,K\}$,
\[
|w_m^k - w_n^k| \le L_w\, d_{\mathrm{Chordal}}(\mathbf{Q}_m,\mathbf{Q}_n).
\]
\end{itemize}
\end{assumption}

\begin{lemma}[Cluster-level Jacobian Bounds]\label{lem:cluster-jacobians}
Under Assumption~\ref{asm:bounded-heterogeneity} and Assumption~\ref{asm:bounded-variational}, the cluster-level parameters $\{\boldsymbol{\mu}_i^c(\vecw),\boldsymbol{\Sigma}_i^c(\vecw)\}$ defined in Eq.~\eqref{eq6} are differentiable in $\vecw$. Moreover, for all $\vecw\in\mathcal{W}$,
\[
\|\nabla_{\vecw}\boldsymbol{\mu}_i^c(\vecw)\|_2 \le L_q, \qquad
\|\nabla_{\vecw}\boldsymbol{\Sigma}_i^c(\vecw)\|_\mathrm{F} \le (1+4\delta_\mu)L_q.
\]
In particular, $L_q$ can be chosen sufficiently large so that the uniform bound $\|\nabla_{\vecw}\boldsymbol{\Sigma}_i^c(\vecw)\|_\mathrm{F} \leq L_q$ holds.
\end{lemma}

\begin{proof}
According to Eq.~\eqref{eq6}, the cluster-level mean is given by
\[
\boldsymbol{\mu}_i^c(\vecw) =  \sum_{n \in \mathcal{S}_i^c} \omega_n^c \boldsymbol{\mu}_n^c(\vecw).
\]
Taking the Jacobian with respect to $\vecw$ and applying the triangle inequality, we obtain
\[
\|\nabla_{\vecw} \boldsymbol{\mu}_i^c(\vecw)\|_2
= \left\| \sum_{n \in \mathcal{S}_i^c} \omega_n^c \nabla_{\vecw} \boldsymbol{\mu}_n^c(\vecw) \right\|_2
\leq \sum_{n \in \mathcal{S}_i^c} \omega_n^c \|\nabla_{\vecw} \boldsymbol{\mu}_n^c(\vecw)\|_2
\leq L_q.
\]

In addition, for $\boldsymbol{\Sigma}_i^c(\vecw)$ in Eq.~\eqref{eq6}, we have
\[
\boldsymbol{\Sigma}_i^c(\vecw) = \sum_{n \in \mathcal{S}_i^c} \omega_n^c \boldsymbol{\Sigma}_n^c(\vecw)
+ \sum_{n \in \mathcal{S}_i^c} \omega_n^c \big(\boldsymbol{\mu}_n^c(\vecw) - \boldsymbol{\mu}_i^c(\vecw)\big)\big(\boldsymbol{\mu}_n^c(\vecw) - \boldsymbol{\mu}_i^c(\vecw)\big)^\top.
\]
Taking the Jacobian with respect to $\vecw$ and applying the triangle inequality, we obtain
\[
\begin{aligned}
\|\nabla_{\vecw}\boldsymbol{\Sigma}_i^c(\vecw)\|_\mathrm{F}
\le \sum_{n \in \mathcal{S}_i^c} \omega_n^c \|\nabla_{\vecw}\boldsymbol{\Sigma}_n^c(\vecw)\|_\mathrm{F} + \sum_{n \in \mathcal{S}_i^c} \omega_n^c 2 \|\boldsymbol{\mu}_n^c(\vecw) - \boldsymbol{\mu}_i^c(\vecw)\|_2 \cdot \|\nabla_{\vecw}(\boldsymbol{\mu}_n^c(\vecw) - \boldsymbol{\mu}_i^c(\vecw))\|_2.
\end{aligned}
\]
By Assumption~\ref{asm:bounded-heterogeneity}, $\|\boldsymbol{\mu}_n^c(\vecw) - \boldsymbol{\mu}_i^c(\vecw)\|_2 \le \delta_\mu$, and from the previous result, $\|\nabla_{\vecw}\boldsymbol{\mu}_n^c(\vecw)\|_2 \le L_q$, $\|\nabla_{\vecw}\boldsymbol{\mu}_i^c(\vecw)\|_2 \le L_q$. Therefore, we have
\[
\|\nabla_{\vecw}(\boldsymbol{\mu}_n^c(\vecw) - \boldsymbol{\mu}_i^c(\vecw))\|_2 \le \|\nabla_{\vecw}\boldsymbol{\mu}_n^c(\vecw)\|_2 + \|\nabla_{\vecw}\boldsymbol{\mu}_i^c(\vecw)\|_2 \le 2L_q.
\]
Combining the above, we obtain
\[
\|\nabla_{\vecw}\boldsymbol{\Sigma}_i^c(\vecw)\|_\mathrm{F} \le L_q + 2\delta_\mu \cdot 2L_q = (1 + 4\delta_\mu)L_q.
\]
\end{proof}

\subsection{Error Bound for Semantic Knowledge Sharing}\label{appendix:semantic-error}
We first analyze the error introduced by aligning local distributions with cluster-level distributions.  Recall that for each class $c$, the local distribution on client $m$ is $q(\mathbf{Z}_m^c) = \mathcal{N}(\boldsymbol{\mu}_m^c, \boldsymbol{\Sigma}_m^c)$, and the cluster-level distribution for cluster $\mathcal{S}_i^c$ is $q(\mathcal{S}_i^c) = \mathcal{N}(\boldsymbol{\mu}_i^c, \boldsymbol{\Sigma}_i^c)$.

\begin{lemma}[KL Divergence Bound]\label{lem:kl-bound}
Under Assumption~\ref{asm:bounded-heterogeneity}, for any client $m \in \mathcal{S}_i^c$, the KL divergence between the local distribution and the cluster-level distribution is bounded as
\begin{equation}\label{eq:kl-bound}
D_{\mathrm{KL}}\big(q(\mathbf{Z}_m^c) \| q(\mathcal{S}_i^c)\big) \leq \frac{1}{2\sigma_{\min}^2} \delta_\mu^2 + \frac{3d}{2\sigma_{\min}^2} (\delta_\Sigma + \delta_\mu^2),
\end{equation}
where $\sigma_{\min}^2>0$ is defined in Assumption~\ref{asm:bounded-heterogeneity}.
\end{lemma}

\begin{proof}
For two multivariate Gaussian distributions $\mathcal{N}(\boldsymbol{\mu}_m^c, \boldsymbol{\Sigma}_m^c)$ and $\mathcal{N}(\boldsymbol{\mu}_i^c, \boldsymbol{\Sigma}_i^c)$, the KL divergence admits the closed-form expression: 
\begin{equation}\label{eq:kl-gaussian}
D_{\mathrm{KL}} = \frac{1}{2} \left( \tr((\boldsymbol{\Sigma}_i^c)^{-1} \boldsymbol{\Sigma}_m^c) + (\boldsymbol{\mu}_i^c - \boldsymbol{\mu}_m^c)^\top (\boldsymbol{\Sigma}_i^c)^{-1} (\boldsymbol{\mu}_i^c - \boldsymbol{\mu}_m^c) - d + \ln \frac{|\boldsymbol{\Sigma}_i^c|}{|\boldsymbol{\Sigma}_m^c|} \right).
\end{equation}

We bound each term separately. For the first term, the cluster-level covariance $\boldsymbol{\Sigma}_i^c$ constructed in Eq.~\eqref{eq6} can be rewritten as
\[
\boldsymbol{\Sigma}_i^c =  \sum_{n \in \mathcal{S}_i^c} \omega_n^c \boldsymbol{\Sigma}_n^c + \sum_{n \in \mathcal{S}_i^c} \omega_n^c (\boldsymbol{\mu}_n^c - \boldsymbol{\mu}_i^c)(\boldsymbol{\mu}_n^c - \boldsymbol{\mu}_i^c)^\top.
\]
Therefore, using triangle inequality, Assumption~\ref{asm:bounded-heterogeneity}, and $\|\mathbf{v}\mathbf{v}^\top\|_\mathrm{F} = \|\mathbf{v}\|_2^2$, we have
\[
\|\boldsymbol{\Sigma}_m^c - \boldsymbol{\Sigma}_i^c\|_\mathrm{F} \leq \underbrace{\left\|\boldsymbol{\Sigma}_m^c - \sum_{n \in \mathcal{S}_i^c} \omega_n^c \boldsymbol{\Sigma}_n^c\right\|_\mathrm{F}}_{\leq \delta_\Sigma}
+ \underbrace{\left\|\sum_{n \in \mathcal{S}_i^c} \omega_n^c (\boldsymbol{\mu}_n^c - \boldsymbol{\mu}_i^c)(\boldsymbol{\mu}_n^c - \boldsymbol{\mu}_i^c)^\top\right\|_\mathrm{F}}_{\leq \delta_\mu^2}
\leq \delta_\Sigma + \delta_\mu^2.
\]
Consequently, we can bound the first term as
\[
	\tr((\boldsymbol{\Sigma}_i^c)^{-1} \boldsymbol{\Sigma}_m^c)
= \tr((\boldsymbol{\Sigma}_i^c)^{-1} (\boldsymbol{\Sigma}_m^c - \boldsymbol{\Sigma}_i^c)) + d
\leq \| (\boldsymbol{\Sigma}_i^c)^{-1} \|_\mathrm{F} \cdot \|\boldsymbol{\Sigma}_m^c - \boldsymbol{\Sigma}_i^c\|_\mathrm{F} + d
\leq \frac{\sqrt{d}}{\sigma_{\min}^2}(\delta_\Sigma + \delta_\mu^2) + d
\leq \frac{d(\delta_\Sigma + \delta_\mu^2)}{\sigma_{\min}^2} + d.
\]

For the second term, we first note that Assumption~\ref{asm:bounded-heterogeneity} together with the definition of cluster-level mean in Eq.~\eqref{eq6} implies
\[
\|\boldsymbol{\mu}_m^c - \boldsymbol{\mu}_i^c\|_2 = \left\|\sum_{n\in\mathcal{S}_i^c} \omega_n^c (\boldsymbol{\mu}_m^c-\boldsymbol{\mu}_n^c)\right\|_2 \leq \sum_{n\in\mathcal{S}_i^c} \omega_n^c \|\boldsymbol{\mu}_m^c-\boldsymbol{\mu}_n^c\|_2 \leq \delta_\mu.
\]
Therefore, we obtain
\[
(\boldsymbol{\mu}_i^c - \boldsymbol{\mu}_m^c)^\top (\boldsymbol{\Sigma}_i^c)^{-1} (\boldsymbol{\mu}_i^c - \boldsymbol{\mu}_m^c)
\leq \| (\boldsymbol{\Sigma}_i^c)^{-1} \|_2 \cdot \|\boldsymbol{\mu}_i^c - \boldsymbol{\mu}_m^c\|_2^2
= \frac{1}{\sigma_{\min}^2}\|\boldsymbol{\mu}_i^c - \boldsymbol{\mu}_m^c\|_2^2
\leq \frac{\delta_\mu^2}{\sigma_{\min}^2}.
\]

For the third term, we have
\[
\left| \ln \frac{|\boldsymbol{\Sigma}_i^c|}{|\boldsymbol{\Sigma}_m^c|} \right|
= \left| -\ln \det\big((\boldsymbol{\Sigma}_i^c)^{-1}\boldsymbol{\Sigma}_m^c\big)\right|
= \left|\ln \det\big(\mathbf{I} + (\boldsymbol{\Sigma}_i^c)^{-1}(\boldsymbol{\Sigma}_m^c-\boldsymbol{\Sigma}_i^c)\big)\right|.
\]
Let $\boldsymbol{\Delta} := \boldsymbol{\Sigma}_m^c - \boldsymbol{\Sigma}_i^c$ and define the symmetric matrix $\mathbf{H} := (\boldsymbol{\Sigma}_i^c)^{-\frac{1}{2}}\, \boldsymbol{\Delta}\,(\boldsymbol{\Sigma}_i^c)^{-\frac{1}{2}}.$ Afterwards, according to $\det(\mathbf{I}+\mathbf{A}\mathbf{B})=\det(\mathbf{I}+\mathbf{B}\mathbf{A})$, we have
\[
\det\big(\mathbf{I} + (\boldsymbol{\Sigma}_i^c)^{-1}\boldsymbol{\Delta}\big)
= \det\big(\mathbf{I} + (\boldsymbol{\Sigma}_i^c)^{-\frac{1}{2}}\boldsymbol{\Delta}(\boldsymbol{\Sigma}_i^c)^{-\frac{1}{2}}\big)
= \det(\mathbf{I}+\mathbf{H}).
\]
Let $\{\lambda_j\}_{j=1}^d$ be the eigenvalues of $\mathbf{H}$. Subsequently, we have
\[
\left|\ln\det(\mathbf{I}+\mathbf{H})\right|
= \left|\sum_{j=1}^{d}\ln(1+\lambda_j)\right|
\leq \sum_{j=1}^{d}\left|\ln(1+\lambda_j)\right|.
\]
Moreover, Assumption~\ref{asm:bounded-heterogeneity} implies
\[
\|\mathbf{H}\|_2
\leq \|(\boldsymbol{\Sigma}_i^c)^{-\frac{1}{2}}\|_2^2\,\|\boldsymbol{\Delta}\|_2
= \|(\boldsymbol{\Sigma}_i^c)^{-1}\|_2\,\|\boldsymbol{\Delta}\|_2
\leq \frac{1}{\sigma_{\min}^2}\,\|\boldsymbol{\Delta}\|_\mathrm{F}
\leq \frac{\delta_\Sigma+\delta_\mu^2}{\sigma_{\min}^2}
\leq \frac{1}{2},
\]
and thus $|\lambda_j|\le \|\mathbf{H}\|_2\le\frac{1}{2}$ for all $j$. Applying the inequality $|\ln(1+x)| \leq 2|x|$ for $|x| \leq \frac{1}{2}$, we obtain
\[
\left|\ln\det(\mathbf{I}+\mathbf{H})\right|
\leq 2\sum_{j=1}^{d}|\lambda_j|
\leq 2d\,\|\mathbf{H}\|_2
\leq \frac{2d(\delta_\Sigma + \delta_\mu^2)}{\sigma_{\min}^2}.
\]
Substituting the above three bounds into Eq.~\eqref{eq:kl-gaussian} completes the proof.
\end{proof}

\begin{lemma}[Gradient Error from Semantic Alignment]\label{lem:semantic-gradient-error}
The gradient error induced by semantic knowledge sharing in Eq.~\eqref{eq7} is bounded as
\begin{equation}\label{eq:semantic-gradient-bound}
\left\| \nabla_{\vecw} \mathcal{L}_{\mathrm{node}} \right\|_2 \leq C_1 \left( \delta_\mu + \delta_\mu^2 + \delta_\Sigma \right),
\end{equation}
where $C_1 = \frac{C \cdot K_{\mathrm{node}} \cdot C^\prime}{\sigma_{\min}^2}$ with $C^\prime > 0$ being a constant.
\end{lemma}

\begin{proof}
Recall that the semantic knowledge alignment loss is defined as
\[
\mathcal{L}_{\mathrm{node}} = \sum_{c=1}^{C} \sum_{i=1}^{K_{\mathrm{node}}} \sum_{m \in \mathcal{S}_i^c} D_{\mathrm{KL}}\big(q(\mathbf{Z}_m^c) \| q(\mathcal{S}_i^c)\big).
\]
Here, $q(\mathbf{Z}_m^c)=\mathcal{N}(\boldsymbol{\mu}_m^c(\vecw),\boldsymbol{\Sigma}_m^c(\vecw))$ and $q(\mathcal{S}_i^c)=\mathcal{N}(\boldsymbol{\mu}_i^c(\vecw),\boldsymbol{\Sigma}_i^c(\vecw))$, where the cluster-level parameters $\{\boldsymbol{\mu}_i^c,\boldsymbol{\Sigma}_i^c\}$ are defined in Eq.~\eqref{eq6}.

Taking the gradient with respect to $\vecw$, we have
\[
\nabla_{\vecw} \mathcal{L}_{\mathrm{node}} = \sum_{c=1}^{C} \sum_{i=1}^{K_{\mathrm{node}}} \sum_{m \in \mathcal{S}_i^c} \nabla_{\vecw} D_{\mathrm{KL}}\big(q(\mathbf{Z}_m^c) \| q(\mathcal{S}_i^c)\big).
\]
By using the triangle inequality, we have
\[
\|\nabla_{\vecw} \mathcal{L}_{\mathrm{node}}\|_2
\le \sum_{c=1}^{C} \sum_{i=1}^{K_{\mathrm{node}}} \sum_{m \in \mathcal{S}_i^c}
\left\| \nabla_{\vecw} D_{\mathrm{KL}}\big(q(\mathbf{Z}_m^c) \| q(\mathcal{S}_i^c)\big) \right\|_2.
\]
Recall the closed-form KL divergence between Gaussians in Eq.~\eqref{eq:kl-gaussian}. Differentiating each term and applying the chain rule, the gradient of $D_{\mathrm{KL}}\big(q(\mathbf{Z}_m^c) \| q(\mathcal{S}_i^c)\big)$ with respect to $\vecw$ is a linear combination of (i) $\nabla_{\vecw}\boldsymbol{\mu}_m^c, \nabla_{\vecw}\boldsymbol{\mu}_i^c$ and (ii) $\nabla_{\vecw}\boldsymbol{\Sigma}_m^c, \nabla_{\vecw}\boldsymbol{\Sigma}_i^c$, multiplied by matrices such as $(\boldsymbol{\Sigma}_i^c)^{-1}$ and $(\boldsymbol{\Sigma}_m^c)^{-1}$. We now make the above statement explicit by bounding the contribution of each term in Eq.~\eqref{eq:kl-gaussian}. For brevity, we fix $c,i,m$ and define
$\boldsymbol{\mu}_m := \boldsymbol{\mu}_m^c(\vecw)$, $\boldsymbol{\mu}_i := \boldsymbol{\mu}_i^c(\vecw)$, $\boldsymbol{\Sigma}_m := \boldsymbol{\Sigma}_m^c(\vecw)$, and $\boldsymbol{\Sigma}_i := \boldsymbol{\Sigma}_i^c(\vecw)$.
Let $\mathbf{d}_\mu := \boldsymbol{\mu}_i-\boldsymbol{\mu}_m$ and $\boldsymbol{\Delta} := \boldsymbol{\Sigma}_m-\boldsymbol{\Sigma}_i$.

\textbf{Step 1 (Uniform Inverse Bounds).}
By Assumption~\ref{asm:bounded-heterogeneity}, we have $\lambda_{\min}(\boldsymbol{\Sigma}_i) \ge \sigma_{\min}^2$, which implies
\[
\|\boldsymbol{\Sigma}_i^{-1}\|_2 \le \frac{1}{\sigma_{\min}^2}.
\]
Furthermore, since $\|\boldsymbol{\Delta}\|_2 \le \|\boldsymbol{\Delta}\|_\mathrm{F} \le \delta_\Sigma + \delta_\mu^2 \le \frac{\sigma_{\min}^2}{2}$, Weyl's inequality gives
\[
\lambda_{\min}(\boldsymbol{\Sigma}_m) \ge \lambda_{\min}(\boldsymbol{\Sigma}_i) - \|\boldsymbol{\Delta}\|_2 \ge \sigma_{\min}^2 - \frac{\sigma_{\min}^2}{2} = \frac{\sigma_{\min}^2}{2},
\]
and thus
\[
\|\boldsymbol{\Sigma}_m^{-1}\|_2 \le \frac{2}{\sigma_{\min}^2}.
\]

\textbf{Step 2 (Differential Form of Gaussian KL).}
From Eq.~\eqref{eq:kl-gaussian}, recall that the KL divergence between two Gaussians can be written as a sum of trace, quadratic, and log-determinant terms. To compute its gradient with respect to $\vecw$, we first write its total differential. Using the matrix calculus identities $\mathrm{d}\ln|\mathbf{A}| = \tr(\mathbf{A}^{-1}\mathrm{d}\mathbf{A})$ and $\mathrm{d}(\mathbf{A}^{-1}) = -\mathbf{A}^{-1}(\mathrm{d}\mathbf{A})\mathbf{A}^{-1}$, we obtain
\begin{align*}
2\,\mathrm{d}D_{\mathrm{KL}}
&= \tr\big(\boldsymbol{\Sigma}_i^{-1}\,\mathrm{d}\boldsymbol{\Sigma}_m\big)
- \tr\big(\boldsymbol{\Sigma}_i^{-1}(\mathrm{d}\boldsymbol{\Sigma}_i)\boldsymbol{\Sigma}_i^{-1}\boldsymbol{\Sigma}_m\big) \\
&\quad + 2\,\mathbf{d}_\mu^\top\boldsymbol{\Sigma}_i^{-1}(\mathrm{d}\boldsymbol{\mu}_i-\mathrm{d}\boldsymbol{\mu}_m)
- \mathbf{d}_\mu^\top\boldsymbol{\Sigma}_i^{-1}(\mathrm{d}\boldsymbol{\Sigma}_i)\boldsymbol{\Sigma}_i^{-1}\mathbf{d}_\mu \\
&\quad + \tr\big(\boldsymbol{\Sigma}_i^{-1}\,\mathrm{d}\boldsymbol{\Sigma}_i\big)
- \tr\big(\boldsymbol{\Sigma}_m^{-1}\,\mathrm{d}\boldsymbol{\Sigma}_m\big),
\end{align*}
where $\mathrm{d}$ denotes the total differential (\textit{i.e.}, the first-order variation) induced by an infinitesimal perturbation $\mathrm{d}\vecw$. Next, we regroup terms involving $\mathrm{d}\boldsymbol{\Sigma}_m$ and $\mathrm{d}\boldsymbol{\Sigma}_i$ and then obtain
\begin{equation}\label{eq:diff-kl}
\begin{aligned}
2\,\mathrm{d}D_{\mathrm{KL}}
&= \tr\Big(\big(\boldsymbol{\Sigma}_i^{-1}-\boldsymbol{\Sigma}_m^{-1}\big)\,\mathrm{d}\boldsymbol{\Sigma}_m\Big) \\
&\quad - \tr\Big(\big(\boldsymbol{\Sigma}_i^{-1}\boldsymbol{\Delta}\boldsymbol{\Sigma}_i^{-1} + \boldsymbol{\Sigma}_i^{-1}\mathbf{d}_\mu\mathbf{d}_\mu^\top\boldsymbol{\Sigma}_i^{-1}\big)\,\mathrm{d}\boldsymbol{\Sigma}_i\Big) \\
&\quad + 2\,\mathbf{d}_\mu^\top\boldsymbol{\Sigma}_i^{-1}(\mathrm{d}\boldsymbol{\mu}_i-\mathrm{d}\boldsymbol{\mu}_m),
\end{aligned}
\end{equation}
where we use the identity $\boldsymbol{\Sigma}_i^{-1}-\boldsymbol{\Sigma}_i^{-1}\boldsymbol{\Sigma}_m\boldsymbol{\Sigma}_i^{-1} = -\boldsymbol{\Sigma}_i^{-1}\boldsymbol{\Delta}\boldsymbol{\Sigma}_i^{-1}$ with $\boldsymbol{\Delta} = \boldsymbol{\Sigma}_m - \boldsymbol{\Sigma}_i$.

\textbf{Step 3 (Chain Rule and Term-by-Term Bounds).}
For any unit direction $\mathbf{u}$ in the parameter space (\textit{i.e.}, $\|\mathbf{u}\|_2=1$), set $\mathrm{d}\vecw=\mathbf{u}$. According to Assumption~\ref{asm:bounded-variational}, we have $\|\mathrm{d}\boldsymbol{\mu}_m\|_2 \le L_q$ and $\|\mathrm{d}\boldsymbol{\Sigma}_m\|_\mathrm{F} \le L_q$ for all clients. By Lemma~\ref{lem:cluster-jacobians}, the same bounds hold for cluster-level parameters, namely $\|\mathrm{d}\boldsymbol{\mu}_i\|_2 \le L_q$ and $\|\mathrm{d}\boldsymbol{\Sigma}_i\|_\mathrm{F} \le L_q$. In addition, according to Assumption~\ref{asm:bounded-heterogeneity}, $\|\mathbf{d}_\mu\|_2 \le \delta_\mu$ and $\|\boldsymbol{\Delta}\|_\mathrm{F} \le \delta_\Sigma + \delta_\mu^2$. Therefore, we can bound each term in Eq.~\eqref{eq:diff-kl} as follows.

\emph{(a) Mean-related term.}
By applying the Cauchy-Schwarz inequality, we have
\[
\big|\mathbf{d}_\mu^\top\boldsymbol{\Sigma}_i^{-1}(\mathrm{d}\boldsymbol{\mu}_i-\mathrm{d}\boldsymbol{\mu}_m)\big|
\le \|\boldsymbol{\Sigma}_i^{-1}\|_2\,\|\mathbf{d}_\mu\|_2\,(\|\mathrm{d}\boldsymbol{\mu}_i\|_2+\|\mathrm{d}\boldsymbol{\mu}_m\|_2)
\le \frac{2L_q}{\sigma_{\min}^2}\,\delta_\mu.
\]

\emph{(b) $\mathrm{d}\boldsymbol{\Sigma}_m$-related term.}
Note that
\[
\boldsymbol{\Sigma}_i^{-1}-\boldsymbol{\Sigma}_m^{-1} = \boldsymbol{\Sigma}_i^{-1}\boldsymbol{\Delta}\boldsymbol{\Sigma}_m^{-1},
\]
so
\[
\|\boldsymbol{\Sigma}_i^{-1}-\boldsymbol{\Sigma}_m^{-1}\|_2
\le \|\boldsymbol{\Sigma}_i^{-1}\|_2\,\|\boldsymbol{\Delta}\|_2\,\|\boldsymbol{\Sigma}_m^{-1}\|_2
\le \frac{2}{\sigma_{\min}^4}\,\|\boldsymbol{\Delta}\|_\mathrm{F}.
\]
Using $|\tr(\mathbf{A}^\top\mathbf{B})| \le \|\mathbf{A}\|_\mathrm{F}\|\mathbf{B}\|_\mathrm{F}$ and $\|\mathbf{A}\|_\mathrm{F} \le \sqrt{d}\|\mathbf{A}\|_2$, we obtain
\[
\left|\tr\Big((\boldsymbol{\Sigma}_i^{-1}-\boldsymbol{\Sigma}_m^{-1})\,\mathrm{d}\boldsymbol{\Sigma}_m\Big)\right|
\le \sqrt{d}\,\|\boldsymbol{\Sigma}_i^{-1}-\boldsymbol{\Sigma}_m^{-1}\|_2\,\|\mathrm{d}\boldsymbol{\Sigma}_m\|_\mathrm{F}
\le \frac{2\sqrt{d}\,L_q}{\sigma_{\min}^4}\,\|\boldsymbol{\Delta}\|_\mathrm{F}.
\]

\emph{(c) $\mathrm{d}\boldsymbol{\Sigma}_i$-related term.}
Similarly,
\[
\|\boldsymbol{\Sigma}_i^{-1}\boldsymbol{\Delta}\boldsymbol{\Sigma}_i^{-1}\|_2 \le \|\boldsymbol{\Sigma}_i^{-1}\|_2^2\,\|\boldsymbol{\Delta}\|_2 \le \frac{1}{\sigma_{\min}^4}\,\|\boldsymbol{\Delta}\|_\mathrm{F},
\qquad
\|\boldsymbol{\Sigma}_i^{-1}\mathbf{d}_\mu\mathbf{d}_\mu^\top\boldsymbol{\Sigma}_i^{-1}\|_2 \le \|\boldsymbol{\Sigma}_i^{-1}\|_2^2\,\|\mathbf{d}_\mu\|_2^2 \le \frac{\delta_\mu^2}{\sigma_{\min}^4}.
\]
Therefore,
\[
\left|\tr\Big((\boldsymbol{\Sigma}_i^{-1}\boldsymbol{\Delta}\boldsymbol{\Sigma}_i^{-1} + \boldsymbol{\Sigma}_i^{-1}\mathbf{d}_\mu\mathbf{d}_\mu^\top\boldsymbol{\Sigma}_i^{-1})\,\mathrm{d}\boldsymbol{\Sigma}_i\Big)\right|
\le \frac{\sqrt{d}\,L_q}{\sigma_{\min}^4}\big(\|\boldsymbol{\Delta}\|_\mathrm{F} + \delta_\mu^2\big).
\]

Combining (a)--(c) and using $\|\boldsymbol{\Delta}\|_\mathrm{F} \le \delta_\Sigma + \delta_\mu^2$, we conclude that there exists a constant $C^\prime > 0$ (absorbing $L_q$, $\sqrt{d}$, and $1/\sigma_{\min}^2$) such that
\[
\left\| \nabla_{\vecw} D_{\mathrm{KL}}\big(q(\mathbf{Z}_m^c) \| q(\mathcal{S}_i^c)\big) \right\|_2
\le \frac{C^\prime}{\sigma_{\min}^2}\big(\delta_\mu + \delta_\mu^2 + \delta_\Sigma\big).
\]
Finally, summing over all $c,i,m$ yields Eq.~\eqref{eq:semantic-gradient-bound}, where the counting factors are absorbed into $C_1$.
\end{proof}

\subsection{Error Bound for Structural Knowledge Sharing}\label{appendix:structural-error}
We now analyze the error introduced by aligning the spectral characteristics of local spectral GNNs with those of cluster-level spectral GNNs.

\begin{lemma}[Lipschitz Constant of the Spectral Filter]\label{lem:poly-filter-lipschitz}
For client $m$, consider the spectral GNN defined as
\begin{equation}\label{eq-polynomial-filter-lemma}
\mathbf{P}_m = \sum_{k=0}^K w_m^k\, \mathbf{H}_m^k,
\end{equation}
where $\mathbf{H}_m^k = (\mathbf{L}_m)^k \mathbf{X}_m$. 
Let $h_m(x) = \sum_{k=0}^K w_m^k x^k$ be the corresponding spectral filter for $x \in [0,2]$. Therefore, $h_m$ is Lipschitz continuous on $[0,2]$ with Lipschitz constant
\begin{equation}\label{eq:filter-lip-const}
L_{\mathrm{filter},m} \triangleq \sup_{x\in[0,2]} |h_m'(x)| \le \sum_{k=1}^K k\,|w_m^k|\,2^{k-1}.
\end{equation}
Consequently, controlling the magnitudes of polynomial coefficients (\textit{e.g.}, via regularization in Eq.~\eqref{eq12}) directly constrains the spectral Lipschitz constant of the learned filter.% elastic-net
\end{lemma}

\begin{proof}%用参数范数/正则化控制谱滤波器的Lipschitz常数.
For any $x\in[0,2]$, we have $h_m'(x)=\sum_{k=1}^K k\,w_m^k x^{k-1}$. Thus
\[
|h_m'(x)|\le \sum_{k=1}^K k\,|w_m^k|\,|x|^{k-1} \le \sum_{k=1}^K k\,|w_m^k|\,2^{k-1}.
\]
Taking the supremum over $x\in[0,2]$ yields Eq.~\eqref{eq:filter-lip-const}.
\end{proof}

\begin{remark}
Lemma~\ref{lem:poly-filter-lipschitz} characterizes the sensitivity of the spectral filter $h_m(x)$ with respect to perturbations in the spectral variable $x\in[0,2]$. In particular, bounding the polynomial coefficients controls $\sup_{x\in[0,2]}|h_m'(x)|$, which limits how rapidly the learned filter can vary over the spectrum.

Note that the bound in Eq.~\eqref{eq:filter-lip-const} grows exponentially with $K$ due to the $2^{k-1}$ term. However, in practice, $K$ is set to a moderate value (\textit{e.g.}, $K=6$). Therefore, in practice, the Lipschitz constant of the spectral filter is much smaller than the worst-case theoretical bound.
\end{remark}

\begin{lemma}[Coefficient Perturbation Bound]\label{lem:coeff-perturb-bound}
Fix a client $m$ and let
\[
h_m(x)=\sum_{k=0}^{K} w_m^k x^k,\qquad \bar h_m(x)=\sum_{k=0}^{K} \bar w_m^k x^k,\qquad x\in[0,2],
\]
where $\bar w_m^k$ is a perturbed (\textit{e.g.}, cluster-consensus) coefficient and $\Delta w_m^k \triangleq w_m^k-\bar w_m^k$.
Therefore, for any $x\in[0,2]$,
\begin{equation}\label{eq:coeff-perturb-filter}
|h_m(x)-\bar h_m(x)|
\le \sum_{k=0}^{K} |\Delta w_m^k|\,|x|^k
\le \sum_{k=0}^{K} |\Delta w_m^k|\,2^k.
\end{equation}
Equivalently,
\begin{equation}\label{eq:coeff-perturb-filter-l2}
\sup_{x\in[0,2]} |h_m(x)-\bar h_m(x)|
\le \left(\sum_{k=0}^{K} 4^k\right)^{\!\frac12}\|\Delta w_m\|_2
= \left(\frac{4^{K+1}-1}{3}\right)^{\!\frac12}\|\Delta w_m\|_2,
\end{equation}
where $\Delta w_m=[\Delta w_m^0, \Delta w_m^1, \cdots,\Delta w_m^K]^\top$.
Moreover, for the polynomial-filter spectral GNN output
\[
\mathbf{P}_m(w)=\sum_{k=0}^{K} w_m^k \mathbf{H}_m^k,\qquad 
\mathbf{P}_m(\bar w)=\sum_{k=0}^{K} \bar w_m^k \mathbf{H}_m^k,
\]
we have the deterministic bound
\begin{equation}\label{eq:coeff-perturb-output}
\|\mathbf{P}_m(w)-\mathbf{P}_m(\bar w)\|_\mathrm{F}
\le \sum_{k=0}^{K} |\Delta w_m^k|\,\|\mathbf{H}_m^k\|_\mathrm{F}.
\end{equation}
In particular, since $\mathbf{L}_m$ is the symmetric normalized Laplacian matrix, $\|\mathbf{L}_m\|_2\le 2$ and then $\|\mathbf{H}_m^k\|_\mathrm{F}\le 2^k\|\mathbf{X}_m\|_\mathrm{F}$. Therefore,
\[
\|\mathbf{P}_m(w)-\mathbf{P}_m(\bar w)\|_\mathrm{F}
\le \|\mathbf{X}_m\|_\mathrm{F}\left(\frac{4^{K+1}-1}{3}\right)^{\!\frac12}\|\Delta w_m\|_2.
\]
\end{lemma}

\begin{proof}
By linearity,
\[
h_m(x)-\bar h_m(x)=\sum_{k=0}^{K}(w_m^k-\bar w_m^k)x^k=\sum_{k=0}^{K}\Delta w_m^k x^k.
\]
Applying triangle inequality gives
\[
|h_m(x)-\bar h_m(x)|\le \sum_{k=0}^{K}|\Delta w_m^k|\,|x|^k \le \sum_{k=0}^{K}|\Delta w_m^k|\,2^k,
\]
which yields Eq.~\eqref{eq:coeff-perturb-filter}. For Eq.~\eqref{eq:coeff-perturb-filter-l2}, after applying Cauchy-Schwarz inequality, we have
\[
\sum_{k=0}^{K}|\Delta w_m^k|\,2^k \le \Big(\sum_{k=0}^{K}4^k\Big)^\frac{1}{2}\Big(\sum_{k=0}^{K}(\Delta w_m^k)^2\Big)^\frac{1}{2}.
\]
Finally,
\[
\mathbf{P}_m(w)-\mathbf{P}_m(\bar w)=\sum_{k=0}^{K}(w_m^k-\bar w_m^k)\mathbf{H}_m^k=\sum_{k=0}^{K}\Delta w_m^k\,\mathbf{H}_m^k,
\]
and triangle inequality gives Eq.~\eqref{eq:coeff-perturb-output}.
\end{proof}

\begin{remark}[Role of Lemma~\ref{lem:coeff-perturb-bound}]
Lemma~\ref{lem:coeff-perturb-bound} provides a deterministic forward-stability guarantee for polynomial-filter spectral GNNs: a small perturbation of the spectral coefficients implies a controlled change of the induced spectral filter $h_m$ and the resulting representation $\mathbf{P}_m$. This lemma is used to justify the interpretability of structural alignment: aligning the coefficients stabilizes the learned spectral responses across clients, which supports structural knowledge sharing.
\end{remark}

\begin{lemma}[Gradient Error from Structural Alignment]\label{lem:structural-gradient-error}
Let $\overline{w}_j^k$ denote the cluster mean coefficient defined in Eq.~\eqref{eq10}. Let $c(m)\in\{1,2,\cdots,K_{\mathrm{struct}}\}$ be the structural cluster index such that $m\in\mathcal{T}_{c(m)}$. Define the cluster-consensus coefficient vector $w^{\mathrm{cm}}$ by
\[
(w^{\mathrm{cm}})_m^k \triangleq \overline{w}_{c(m)}^k,\qquad m=1,2, \cdots,M,\ \ k=0,1,\cdots,K.
\]
Here $w^{\mathrm{cm}}$ is a notational device for analysis (an `instantaneous' within-cluster mean at the same iterate), rather than an explicit parameter substitution in the algorithm.
Let $w$ denote the collection of all spectral GNN coefficients $w_m^k$ for all $m=1,2,\cdots,M$ and $k=0,1,\cdots,K$.

\smallskip
\noindent\textbf{Gradient convention.} In this lemma, $\nabla F(\cdot)$ denotes the gradient of the population risk restricted to the coefficient coordinates (all other model parameters are fixed). We reuse $L_F$ to denote a smoothness constant on these coordinates.

Define
\[
e_{\mathrm{struct}}(w) \triangleq \nabla F(w^{\mathrm{cm}})-\nabla F(w) + \lambda_1\,g_{\ell_1}(w) + \lambda_2\, w,
\]
where $g_{\ell_1}(w)\in \partial\|w\|_1$ denotes a subgradient of the $\ell_1$ term in Eq.~\eqref{eq12} (restricted to the coefficient coordinates). Therefore, we have
\begin{equation}\label{eq:structural-gradient-bound}
\|e_{\mathrm{struct}}(w)\|_2 \leq C_2 (K+1)\epsilon_U + \lambda_1 C_3 + \lambda_2 C_4,
\end{equation}
where $\epsilon_U$ is defined in Assumption~\ref{asm:bounded-heterogeneity}.
\end{lemma}

\begin{proof}%参数对齐/正则化→Lipschitz常数有界→误差有界的理论链条。
Recall that $\mathcal{L}_{\mathrm{struct}} = \mathcal{L}_{\mathrm{align}} + \mathcal{L}_{\mathrm{reg}}$ is optimized via alignment and regularization, and the bound concerns the induced perturbation term $e_{\mathrm{struct}}(w)$ defined above.

\textbf{Step 1 (Alignment-induced Deviation).}
Fix any structural cluster $\mathcal{T}_j$ and any $m\in\mathcal{T}_j$. For each $k\in\{0, 1, \cdots,K\}$, by Assumption~\ref{asm:bounded-heterogeneity} and Assumption~\ref{asm:bounded-filter},
\[
|w_m^k-\overline{w}_j^k| \le \frac{1}{|\mathcal{T}_j|}\sum_{n\in\mathcal{T}_j}|w_m^k-w_n^k|
\le L_w\,\epsilon_U.
\]
Therefore, for each $m\in\mathcal{T}_j$,
\[
\|[w_m^0-\overline{w}_j^0, w_m^1-\overline{w}_j^1, \cdots,w_m^K-\overline{w}_j^K]\|_2 \le \|[w_m^0-\overline{w}_j^0, w_m^1-\overline{w}_j^1, \cdots,w_m^K-\overline{w}_j^K]\|_1 \le (K+1)L_w\,\epsilon_U,
\]
and hence
\[
\|w^{\mathrm{cm}}-w\|_2 \le \sqrt{M}\,(K+1)L_w\,\epsilon_U.
\]
By $L_F$-smoothness (Assumption~\ref{asm:smoothness}), we have
\[
\|\nabla F(w^{\mathrm{cm}})-\nabla F(w)\|_2 \le L_F\,\|w^{\mathrm{cm}}-w\|_2 \le L_F\sqrt{M}\,(K+1)L_w\,\epsilon_U.
\]

\textbf{Step 2 (Regularization Terms).}
Since each coordinate of an $\ell_1$ subgradient has magnitude at most $1$, we have $\|g_{\ell_1}(w)\|_2\le \sqrt{M(K+1)}$.
By Assumption~\ref{asm:bounded-filter}, $\|w\|_2\le \sqrt{M(K+1)}\,w_{\max}$ on the coefficient coordinates.

\textbf{Step 3 (Combination).}
Setting
\[
C_2\triangleq L_F\sqrt{M}\,L_w,\quad C_3\triangleq \sqrt{M(K+1)},\quad C_4\triangleq \sqrt{M(K+1)}\,w_{\max},
\]
and applying the triangle inequality completes the proof.
\end{proof}

\begin{remark}
The introduction of the cluster-consensus coefficients $w^{\mathrm{cm}}$ is a standard analytical device in clustered federated learning~\cite{carrillo2024fedcbo}. Although the algorithm minimizes the structural alignment loss, the error bound is established by comparing the current coefficients to their cluster means. The actual structural alignment loss ensures that $w_m^k$ remains close to $\overline{w}_j^k$, so the error induced by alignment is upper-bounded by the error at the consensus point.
\end{remark}

\subsection{Overall Convergence Analysis}\label{appendix:overall-convergence}
We now prove Theorem~\ref{thm:FedSSA-convergence} by combining the error bounds from semantic and structural knowledge sharing. 

\begin{proof}[Proof of Theorem~\ref{thm:FedSSA-convergence}]
Let $\vecg(\vecw^t)$ denote the local gradient used in FedSSA at iteration $t$, which incorporates both the task-specific gradient and the knowledge sharing gradients:
\[
\vecg(\vecw^t) = \nabla F(\vecw^t) + \nabla_{\vecw} \mathcal{L}_{\mathrm{node}} + \mathbf{e}_{\mathrm{struct}}(\vecw^t).
\]

By Lemma~\ref{lem:semantic-gradient-error} and Lemma~\ref{lem:structural-gradient-error}, the gradient error can be bounded as
\begin{equation}\label{eq:total-gradient-error}
\|\vecg(\vecw^t) - \nabla F(\vecw^t)\|_2 \leq C_1(\delta_\mu + \delta_\mu^2 + \delta_\Sigma) + C_2 (K+1)\epsilon_U + \lambda_1 C_3 + \lambda_2 C_4.
\end{equation}

The update rule of FedSSA is given by $\vecw^{t+1} = \vecw^t - \eta \vecg(\vecw^t)$. To analyze the convergence, we consider the distance to the optimal solution: 
\begin{equation}\label{eq:distance-decomposition}
\begin{aligned}
\|\vecw^{t+1} - \vecw^*\|_2 &= \|\vecw^t - \eta \vecg(\vecw^t) - \vecw^*\|_2 \\
&\leq \|\vecw^t - \eta \nabla F(\vecw^t) - \vecw^*\|_2 + \eta \|\vecg(\vecw^t) - \nabla F(\vecw^t)\|_2.
\end{aligned}
\end{equation}

For the first term, we use the co-coercivity property of strongly convex and smooth functions~\cite{yin2018byzantine}. By Assumption~\ref{asm:smoothness} and Assumption~\ref{asm:strong-convexity}, for any $\vecw, \vecw' \in \mathcal{W}$,
\begin{equation}\label{eq:co-coercivity}
\langle \vecw - \vecw', \nabla F(\vecw) - \nabla F(\vecw') \rangle \geq \frac{L_F \lambda_F}{L_F + \lambda_F} \|\vecw - \vecw'\|_2^2 + \frac{1}{L_F + \lambda_F} \|\nabla F(\vecw) - \nabla F(\vecw')\|_2^2.
\end{equation}

Setting $\vecw' = \vecw^*$ and noting that $\nabla F(\vecw^*) = \mathbf{0}$, we obtain
\begin{equation}\label{eq:inner-product-bound}
\langle \vecw^t - \vecw^*, \nabla F(\vecw^t) \rangle \geq \frac{L_F \lambda_F}{L_F + \lambda_F} \|\vecw^t - \vecw^*\|_2^2 + \frac{1}{L_F + \lambda_F} \|\nabla F(\vecw^t)\|_2^2.
\end{equation}

We now bound $\|\vecw^t - \eta \nabla F(\vecw^t) - \vecw^*\|_2^2$ as
\begin{equation}\label{eq:squared-distance}
\begin{aligned}
\|\vecw^t - \eta \nabla F(\vecw^t) - \vecw^*\|_2^2 &= \|\vecw^t - \vecw^*\|_2^2 - 2\eta \langle \vecw^t - \vecw^*, \nabla F(\vecw^t) \rangle + \eta^2 \|\nabla F(\vecw^t)\|_2^2 \\
&\leq \|\vecw^t - \vecw^*\|_2^2 - \frac{2\eta L_F \lambda_F}{L_F + \lambda_F} \|\vecw^t - \vecw^*\|_2^2 \\
&\quad - \frac{2\eta}{L_F + \lambda_F} \|\nabla F(\vecw^t)\|_2^2 + \eta^2 \|\nabla F(\vecw^t)\|_2^2.
\end{aligned}
\end{equation}

Choosing $\eta = \frac{1}{L_F}$, the coefficient of $\|\nabla F(\vecw^t)\|_2^2$ simplifies to
\[
-\frac{2}{L_F(L_F + \lambda_F)} + \frac{1}{L_F^2} = \frac{L_F + \lambda_F - 2L_F}{L_F^2(L_F + \lambda_F)} = \frac{\lambda_F - L_F}{L_F^2(L_F + \lambda_F)} \leq 0,
\]
since $\lambda_F \leq L_F$ for any smooth and strongly convex function. Therefore,
\begin{equation}\label{eq:contraction}
\|\vecw^t - \eta \nabla F(\vecw^t) - \vecw^*\|_2^2 \leq \left(1 - \frac{2\lambda_F}{L_F + \lambda_F}\right) \|\vecw^t - \vecw^*\|_2^2.
\end{equation}

Using the inequality $\sqrt{1-x} \leq 1 - \frac{x}{2}$ for $x \in [0, 1]$, we obtain
\begin{equation}\label{eq:linear-contraction}
\|\vecw^t - \eta\nabla F(\vecw^t) - \vecw^*\|_2 \leq \sqrt{1 - \frac{2\lambda_F}{L_F + \lambda_F}}\|\vecw^t - \vecw^*\|_2 \leq \left(1 - \frac{\lambda_F}{L_F + \lambda_F}\right)\|\vecw^t - \vecw^*\|_2.
\end{equation}

Substituting Eq.~\eqref{eq:total-gradient-error} and Eq.~\eqref{eq:linear-contraction} into Eq.~\eqref{eq:distance-decomposition}, we have
\begin{equation}\label{eq:iteration-bound}
\|\vecw^{t+1} - \vecw^*\|_2 \leq \left(1 - \frac{\lambda_F}{L_F + \lambda_F}\right) \|\vecw^t - \vecw^*\|_2 + \frac{\mathcal{E}}{L_F}.
\end{equation}

Let $\rho = 1 - \frac{\lambda_F}{L_F + \lambda_F} \in (0, 1)$. Unrolling the recursion in Eq.~\eqref{eq:iteration-bound} yields
\begin{equation}\label{eq:unrolled}
\begin{aligned}
\|\vecw^T - \vecw^*\|_2 &\leq \rho^T \|\vecw^0 - \vecw^*\|_2 + \frac{\mathcal{E}}{L_F} \sum_{t=0}^{T-1} \rho^t \\
&\leq \rho^T \|\vecw^0 - \vecw^*\|_2 + \frac{\mathcal{E}}{L_F} \cdot \frac{1}{1 - \rho} \\
&= \rho^T \|\vecw^0 - \vecw^*\|_2 + \frac{\mathcal{E}}{L_F} \cdot \frac{L_F + \lambda_F}{\lambda_F} \\
&= \left(1 - \frac{\lambda_F}{L_F + \lambda_F}\right)^T \|\vecw^0 - \vecw^*\|_2 + \frac{L_F + \lambda_F}{\lambda_F L_F} \mathcal{E}.
\end{aligned}
\end{equation}

This completes the proof of Theorem~\ref{thm:FedSSA-convergence}.
\end{proof}
\subsection{Proof of Corollary~\ref{cor:convergence-rate}}\label{appendix:convergence-rate}
\begin{proof}[Proof of Corollary~\ref{cor:convergence-rate}]
By Theorem~\ref{thm:FedSSA-convergence}, let
\[
\rho \triangleq 1-\frac{\lambda_F}{L_F+\lambda_F}\in(0,1).
\]
Therefore, we have
\[
\|\vecw^T-\vecw^*\|_2 \le \rho^T \|\vecw^0-\vecw^*\|_2
+ \frac{L_F+\lambda_F}{\lambda_F L_F}\mathcal{E}.
\]

If
\[
T \ge \frac{L_F+\lambda_F}{\lambda_F}\log\frac{\|\vecw^0-\vecw^*\|_2}{\xi},
\]
then using $\log(1-x)\le -x$ for $x\in(0,1)$, we have
\[
\log\rho=\log\Bigl(1-\frac{\lambda_F}{L_F+\lambda_F}\Bigr)
\le -\frac{\lambda_F}{L_F+\lambda_F}.
\]
Let $a \triangleq \frac{\lambda_F}{L_F+\lambda_F}$. Therefore, we have $T\log\rho \le -aT$. Moreover, since
\[
T \ge \frac{1}{a}\log\frac{\|\vecw^0-\vecw^*\|_2}{\xi},
\]
multiplying both sides by $-a$ yields $-aT \le \log\frac{\xi}{\|\vecw^0-\vecw^*\|_2}$. Therefore,
\[
T\log\rho \le \log\frac{\xi}{\|\vecw^0-\vecw^*\|_2}.
\]
Exponentiating both sides yields
\[
\rho^T\|\vecw^0-\vecw^*\|_2 \le \xi.
\]
Substituting this into the bound from Theorem~\ref{thm:FedSSA-convergence} gives
\[
\|\vecw^T - \vecw^*\|_2 \le \xi + \frac{L_F+\lambda_F}{\lambda_F L_F}\mathcal{E}.
\]
\end{proof}

\subsection{Discussion on the Error Bound}\label{appendix:error-discussion}
Theorem~\ref{thm:FedSSA-convergence} demonstrates that FedSSA converges linearly to a neighborhood of $\vecw^*$, whose radius is characterized by the error floor $\mathcal{E}$. In our bound, $\mathcal{E}$ consists of a semantic term and a structural term.

\textbf{Semantic Term}
The term $C_1(\delta_\mu + \delta_\mu^2 + \delta_\Sigma)$ is induced by semantic alignment, namely matching local class-wise Gaussian posteriors to cluster-level distributions. This term is controlled by the intra-cluster heterogeneity parameters (\textit{i.e.}, $\delta_\mu$ and $\delta_\Sigma$) in Assumption~\ref{asm:bounded-heterogeneity}, which quantifies the within-cluster discrepancy of class-wise means and covariances. Consequently, a tighter semantic clustering (\textit{i.e.}, smaller within-cluster variation) leads to a smaller semantic error.

\textbf{Structural Term}
The term $C_2(K+1)\epsilon_U + \lambda_1 C_3 + \lambda_2 C_4$ arises from structural alignment and regularization. The alignment component scales linearly with the cluster tightness $\epsilon_U$ and filter order (\textit{i.e.}, $K+1$). The regularization components are linear in $\lambda_1$ and $\lambda_2$ under the coefficient boundedness condition in Assumption~\ref{asm:bounded-filter}. Therefore, decreasing $\epsilon_U$, $\lambda_1$, and $\lambda_2$ reduces the structural error.

\section{Efficiency Analysis}\label{Ap_efficiency_analysis}
In this section, we analyze the space complexity and time complexity of our proposed FedSSA on client side and server side, respectively.

\subsection{Space Complexity of Our Proposed FedSSA}
\subsubsection{Client Side}
Since the local model deployed on each client is a spectral GNN (\textit{i.e.}, UniFilter~\cite{huanguniversal}), the space complexity of our FedSSA on client side is determined by spectral GNN. Specifically, each local model consists of $K$ orders of bases, and one Multi-Layer Perceptron (MLP). In other words, the space complexity of each client consists of two parts, namely $K$ orders of bases and one MLP. On one hand, the space complexities of $K$ bases are $\mathcal{O}(K \times n_m \times d)$. On the other hand, the space complexity of one MLP is $\mathcal{O}(d^2)$. Consequently, the overall space complexity of our method on client side is $\mathcal{O}( K \times n_m \times d  + d^2)$.

\subsubsection{Server Side}\label{complexity_algorithm}
On server side, the space complexity of our method is determined by three parts, namely storing uploaded coefficients from $M$ clients, storing inferred distributions from $M$ clients, and storing spectral energy measures from $M$ clients. First, the space complexity of storing uploaded coefficients from $M$ clients is $\mathcal{O}(M \times K)$ when considering there are $K$ orders of bases. Second, the space complexity of storing inferred distributions is $\mathcal{O}(M \times C \times d^2)$. Third, the space complexity of storing spectral energy measures is $\mathcal{O}(M \times K \times d)$. Consequently, the total space complexity of our method on server side is $\mathcal{O}( M \times K + M \times C \times d^2 + M \times K \times d)$. After simplification, it can be written as $\mathcal{O}\big( M  \times (K + Cd^2 + Kd)\big)$.

\subsection{Time Complexity of Our Proposed FedSSA}
\subsubsection{Client Side}
The time complexity of our FedSSA on each client is determined by our proposed VGAE, which is actually a spectral GNN. Specifically, it is made up of $K$ orders of bases and one MLP. First, the time complexity of $K$ orders of bases is $\mathcal{O}(K \times e_m \times d)$, where $e_m$ denotes the number of edges. Second, the time complexity of one MLP is $\mathcal{O}(n_m \times d^2)$. Furthermore, the time complexities of semantic knowledge alignment and spectral knowledge alignment are $\mathcal{O}(C \times d)$ and $\mathcal{O}(K)$, respectively. Consequently, the overall time complexity of our method on client side is $\mathcal{O}(K \times e_m \times d + n_m \times d^2 + C \times d + K)$. After simplification, it can be written as $\mathcal{O}\big(d \times (K e_m  + n_md+ C)+K\big)$.
\subsubsection{Server Side}
On server side, there are two operations, namely clustering via inferred distributions (\textit{a.k.a.}, semantic clustering) and clustering via spectral energies (\textit{a.k.a.}, structural clustering). First, the time complexity of semantic clustering is $\mathcal{O}(M \times C \times K_\mathrm{node} \times d)$. Second, the time complexity of structural clustering is $\mathcal{O}(M^2 d K  + MK_\mathrm{struct})$, where the first term accounts for computing all pairwise Chordal distances between spectral energy matrices, and the second term accounts for k-means clustering. Consequently, the total time complexity of our FedSSA on server side is $\mathcal{O}(M \times C \times K_\mathrm{node} \times d + M(Md K + K_\mathrm{struct}))$. After simplification, it can be written as $\mathcal{O}\big(M (C K_\mathrm{node} d + Md K + K_\mathrm{struct})\big)$.

\subsection{Efficiency Comparison with Baseline Methods}
As shown in \cref{table_efficiency_compare}, we present efficiency comparisons on the space complexity and time complexity of our proposed FedSSA with those of existing methods. We can observe that the complexities of our method are generally comparable to those of existing typical methods. This validates that our proposed FedSSA is efficient in real-world applications, which is confirmed by our experimental results in Section~\ref{time_communication}.

\begin{table*}[t]
  \centering
  \scriptsize
  \caption{The space and time complexity of different methods on client side and server side. Here $M$, $K$, $d$, $n_m$, $e_m$, and $C$ denote the number of clients, orders, dimensions, nodes, edges, and classes, respectively.}
    \label{table_efficiency_compare}
    \renewcommand{\arraystretch}{1.1}
    \begin{tabular}{lcccc}
      \toprule
      \rowcolor{gray!50}
      \textbf{Method} & \textbf{Client Space} & \textbf{Server Space} & \textbf{Client Time} & \textbf{Server Time} \\ 
      \midrule
      FedAvg~\cite{mcmahan2017communication}   & $d + d^2$  & $M(1 + d^2)$  & $e_m d + n_m d^2$  & $M$ \\
      \rowcolor{gray!20}
      FedProx~\cite{MLSYS2020_1f5fe839}        & $d + 2d^2$ & $M(1 + d^2)$  & $e_m d + n_m d^2 + d^2$  & $M$ \\
      FedPer~\cite{Arivazhagan2019}            & $d + 2d^2$ & $M(1 + d^2)$  & $e_m d + n_m d^2 + d^2$  & $M$ \\
      \rowcolor{gray!20}
      GCFL~\cite{NEURIPS2021_9c6947bd}         & $d + d^2$  & $M(1 + 2d^2)$ & $e_m d + n_m d^2$  & $M + M^2(\log M + d^2)$ \\
      FedGNN~\cite{wu2021fedgnn}               & $d + 2d^2$ & $2M(1 + 2d^2)$& $e_m d + n_m d^2 + d$  & $M$ \\
      \rowcolor{gray!20}
      FedSage+\cite{NEURIPS2021_34adeb8e}      & $n_m d + 3d^2$ & $M(1 + 3d^2)$ & $e_m d + n_m d^2$  & $M$ \\
      FED-PUB~\cite{baek2023personalized}      & $n_m d + d^2$  & $M(d^2 + M)$  & $e_m d + n_m d^2$  & $Md(M + d)$ \\
      \rowcolor{gray!20}
      FedGTA~\cite{li2023fedgta}               & $d + d^2 + n_m C$ & $M(1 + d^2 + n_m C)$ & $e_m(d + n_m C) + n_m(d^2 + C)$ & $M(1 + n_m C)$ \\
      AdaFGL~\cite{li2024adafgl}               & $n_m d + 2d^2$ & $M(1 + d^2)$  & $e_m d + e_m n_m + n_m d^2$ & $Md + Mn_m d$ \\
      \rowcolor{gray!20}
      FedTAD~\cite{zhu2024fedtad}              & $n_m d + 2d^2$ & $M(1 + d^2)$  & $e_m d + n_m d^2$  & $n_m d(d + n_m + 2MC)$ \\
      FedIIH~\cite{wentao2025fediih}           & $Kn_m d + d^2$ & $MK(d^2 + M)$ & $K(e_m d + n_m d^2)$ & $MKd(Mn_m + n_m d)$ \\
      \hline
      \rowcolor{yellow!30}
      FedSSA (Ours) & $Kn_m d + d^2$ & $M(K + Cd^2 + Kd)$ & $d (K e_m  + n_md+ C)+K$ & $M (C K_\mathrm{node} d + Md K + K_\mathrm{struct})$ \\
      \bottomrule
\end{tabular}
\end{table*}

\section{Implementation Details}
\label{implementation_details}
This section provides the details of our experimental setup, which includes local training objective, computational platform, graph datasets, subgraph partitioning, baseline methods, and training procedures.

\subsection{Local Training Objective}\label{ap_local_objective}
For each client $m$, the overall local training objective integrates cross-entropy loss, negative ELBO loss, semantic alignment loss, and structural alignment loss. Specifically, we minimize
\begin{equation}
\mathcal{L}_m = \mathcal{L}_\mathrm{ce}^{(m)} + \mathcal{L}_\mathrm{vgae}^{(m)} + \mathcal{L}_\mathrm{node}^{(m)} + \mathcal{L}_\mathrm{struct}^{(m)},
\end{equation}
where $\mathcal{L}_\mathrm{ce}^{(m)}$ is cross-entropy loss for node classification, $\mathcal{L}_\mathrm{vgae}^{(m)}$ is negative ELBO loss (see Eq.~\eqref{eq1}), $\mathcal{L}_\mathrm{node}^{(m)}$ is semantic alignment loss, and $\mathcal{L}_\mathrm{struct}^{(m)}$ is structural alignment loss.

\subsection{Experimental Platform}
\label{experimental_platform}
All experiments are conducted on a Linux server with a 2.90\,GHz Intel Xeon Gold 6326 CPU, 64\,GB RAM, and two NVIDIA GeForce RTX 4090 GPUs with 48\,GB memory. Our FedSSA is implemented in Python 3.8.8, PyTorch 1.12.0, and PyTorch Geometric (PyG) 2.5.1.

\subsection{Datasets}
\label{dataset_info}
We evaluate our proposed FedSSA on eleven widely used benchmark datasets, which includes six homophilic and five heterophilic graph datasets. On one hand, homophilic graph datasets consist of four citation networks (\textit{i.e.}, \textit{Cora}, \textit{CiteSeer}, \textit{PubMed}, and \textit{ogbn-arxiv}) and two Amazon product co-purchasing graphs (\textit{i.e.}, \textit{Amazon-Computer} and \textit{Amazon-Photo}). On the other hand, heterophilic graph datasets include \textit{Roman-empire}, \textit{Amazon-ratings}, \textit{Minesweeper}, \textit{Tolokers}, and \textit{Questions}~\cite{platonov2023a}. The statistical information of all datasets is summarized in Table~\ref{datasets_statistics}. For \textit{Minesweeper}, \textit{Tolokers}, and \textit{Questions}, which are binary classification tasks, we utilize Area Under the ROC Curve (AUC) as the evaluation metric following previous work~\cite{platonov2023a, wentao2025fediih}. For other multi-class datasets, classification accuracy is used as the evaluation metric.

\begin{table*}[t]
    \centering
    \scriptsize
    \caption{Statistical information of eleven used graph datasets.}
    \label{datasets_statistics}
    \renewcommand{\arraystretch}{1.1}
    \begin{tabular}{lrrrr}
    \toprule
    \textbf{Dataset} & \textbf{\# Nodes} & \textbf{\# Edges} & \textbf{\# Classes} & \textbf{\# Features} \\ 
    \midrule
    \rowcolor{gray!50}
    \multicolumn{5}{l}{\textbf{Homophilic Graphs}} \\
    \textit{Cora}            & 2,708    & 5,429     & 7  & 1,433 \\
    \textit{CiteSeer}        & 3,327    & 4,732     & 6  & 3,703 \\
    \textit{PubMed}          & 19,717   & 44,324    & 3  & 500   \\
    \textit{Amazon-Computer} & 13,752   & 491,722   & 10 & 767   \\
    \textit{Amazon-Photo}    & 7,650    & 238,162   & 8  & 745   \\
    \textit{ogbn-arxiv}      & 169,343  & 1,166,243 & 40 & 128   \\
    \hline
    \rowcolor{gray!50}
    \multicolumn{5}{l}{\textbf{Heterophilic Graphs}} \\
    \textit{Roman-empire}    & 22,662   & 32,927    & 18 & 300   \\
    \textit{Amazon-ratings}  & 24,492   & 93,050    & 5  & 300   \\
    \textit{Minesweeper}     & 10,000   & 39,402    & 2  & 7     \\
    \textit{Tolokers}        & 11,758   & 519,000   & 2  & 10    \\
    \textit{Questions}       & 48,921   & 153,540   & 2  & 301   \\
    \bottomrule
    \end{tabular}
\end{table*}

For all datasets except \textit{ogbn-arxiv}, we randomly sample 20\% of nodes for training, 40\% of nodes for validation, and 40\% of nodes for testing. For \textit{ogbn-arxiv}, which contains more than 0.1 million nodes and 1 million edges, we follow previous work~\cite{baek2023personalized, wentao2025fediih} and utilize only 5\% of nodes for training, 47.5\% of nodes for validation, and 47.5\% of nodes for testing.

\subsection{Subgraph Partitioning}
\label{subgraph_partitioning_detail}
Following practical scenarios and prior work~\cite{baek2023personalized, wentao2025fediih}, we consider two subgraph partitioning schemes, namely non-overlapping and overlapping. In the non-overlapping setting, the global node set $\mathcal{V}$ is partitioned into $M$ disjoint subsets $\{\mathcal{V}_m\}_{m=1}^M$ such that $\cup_{m=1}^{M} \mathcal{V}_m = \mathcal{V}$ and $\mathcal{V}_m \cap \mathcal{V}_n = \emptyset$ for $m \neq n$. Any partitioning scheme that does not satisfy this condition is referred to as overlapping. Specifically, procedures for both partitioning schemes are described below.

\subsubsection{Non-overlapping Partitioning}
Given $M$ clients, we generate $M$ non-overlapping subgraphs by applying METIS graph partitioning algorithm~\cite{karypis1997metis} to the global graph. Each client is then assigned a unique subgraph corresponding to one partition output by METIS.

\subsubsection{Overlapping Partitioning}
Given $M$ clients, we first employ METIS to partition the global graph into $\lfloor M/5 \rfloor$ subgraphs, where $\lfloor \cdot \rfloor$ denotes the floor function. For each generated subgraph, we randomly sample half of its nodes and their associated edges. Meanwhile, we repeat this sampling process five times to generate five distinct but overlapping subgraphs. Consequently, the total number of overlapping subgraphs matches the number of clients.

\subsection{Baseline Methods}
\label{baseline_methods_info}
We compare our proposed FedSSA with eleven baseline methods, which includes one classic Federated Learning (FL) method (\textit{i.e.}, FedAvg~\cite{mcmahan2017communication}), two personalized FL methods (\textit{i.e.}, FedProx~\cite{MLSYS2020_1f5fe839} and FedPer~\cite{Arivazhagan2019}), three general Graph Federated Learning (GFL) methods (\textit{i.e.}, GCFL~\cite{NEURIPS2021_9c6947bd}, FedGNN~\cite{wu2021fedgnn}, and FedSage+~\cite{NEURIPS2021_34adeb8e}), and five personalized GFL approaches (\textit{i.e.}, FED-PUB~\cite{baek2023personalized}, FedGTA~\cite{li2023fedgta}, AdaFGL~\cite{li2024adafgl}, FedTAD~\cite{zhu2024fedtad}, and FedIIH~\cite{wentao2025fediih}). In addition, we introduce a local training method, where each client trains independently without federated aggregation. The details of these baseline methods are summarized as follows.

\textbf{FedAvg}~\cite{mcmahan2017communication}: A foundational FL method in which clients train local models independently and periodically send updates to a central server. The server aggregates the received parameters by averaging and broadcasts the aggregated global model back to all clients.

\textbf{FedProx}~\cite{MLSYS2020_1f5fe839}: A personalized FL method that adds a proximal term to local objective, which penalizes the divergence between local model parameters and global model parameters. This regularization stabilizes local updates and enables clients to learn personalized models while leveraging global information.

\textbf{FedPer}~\cite{Arivazhagan2019}: A personalized FL method that aggregates the backbone network parameters across clients during federated aggregation, while keeping the classification layer parameters personalized and updated locally on each client.

\textbf{GCFL}~\cite{NEURIPS2021_9c6947bd}: A representative GFL method, which is originally designed for vertical GFL scenarios such as molecular property prediction~\cite{yu2026atom}. Specifically, GCFL employs a bi-partitioning strategy that recursively divides clients into two disjoint groups based on the similarity of their gradients. Actually, this procedure is similar to clustered FL~\cite{sattler2021clustered}. After that, model aggregation is performed only within each group.

\textbf{FedGNN}~\cite{wu2021fedgnn}: A general GFL method that enhances local performance by exchanging node embeddings across clients. To be specific, when nodes in different clients have identical neighborhoods, FedGNN transfers corresponding node embeddings to expand local subgraphs, which leads to enriched local information via cross-client node embeddings.

\textbf{FedSage+}~\cite{NEURIPS2021_34adeb8e}: A GFL baseline that reconstructs missing edges between subgraphs. Specifically, each client receives node representations from other clients and computes gradients based on the distance between local node features and received representations. These gradients are then sent back to other clients and used to train their corresponding neighbor generator, which facilitates the reconstruction of missing edges.

\textbf{FED-PUB}~\cite{baek2023personalized}: A personalized GFL method that performs personalized model aggregation. It estimates inter-subgraph similarity by evaluating local model outputs on a test graph. Depending on these similarity scores, it carries out weighted aggregation of model parameters. Furthermore, each client learns a personalized sparse mask to select and update only partially aggregated parameters, which are relevant to its local subgraph.

\textbf{FedGTA}~\cite{li2023fedgta}: A personalized GFL method that estimates inter-subgraph similarity levels. Specifically, it consists of three steps. First, each client computes topology-aware local smoothing confidence and mixed moments of neighbor features. Second, in each communication round, these computed results are then uploaded to server along with local model parameters. Third, server performs weighted aggregation tailored to each client based on these estimated similarities.

\textbf{AdaFGL}~\cite{li2024adafgl}: A personalized GFL method that employs a decoupled two-step personalization strategy. In the first stage, standard federated training is conducted to obtain a global knowledge extractor via aggregation in the final round. In the second stage, each client performs personalized training on its local subgraph using extracted federated knowledge.

\textbf{FedTAD}~\cite{zhu2024fedtad}: A personalized GFL method that introduces a generator to synthesize pseudo graphs for data-free knowledge distillation. This approach enables effective knowledge transfer from local models to the global model, thereby mitigating the negative impact of heterogeneity.

\textbf{FedIIH}~\cite{wentao2025fediih}: A personalized GFL method that simultaneously models both inter-heterogeneity and intra-heterogeneity. On one hand, inter-heterogeneity is captured from a multi-level global perspective via hierarchical variational inference, which facilitates accurate estimation of inter-subgraph similarity via graph data distributions. On the other hand, intra-heterogeneity is addressed by disentangling each subgraph into multiple latent factors, which allows fine-grained personalization.

\textbf{Local}: A non-federated baseline method in which each client trains its local model independently by only using local data. In this method, there is no federated aggregation or collaboration.

\subsection{Training Details}
For all baseline methods except FedAvg, FedSage+, FedGTA, FedIIH, and our proposed FedSSA, we employ a two-layer Graph Convolutional Network (GCN)~\cite{kipf2017semisupervised} followed by a linear classifier as the network architecture. Meanwhile, the hyperparameters of each baseline method are set according to the configurations in their original papers. For FedSage+, we adopt GraphSage~\cite{NIPS2017_5dd9db5e} as the encoder and train a missing neighbor generator to mend missing edges among subgraphs. In addition, FedGTA employs a Graph Attention Multi-Layer Perceptron (GAMLP)~\cite{35346783539121} as its backbone together with a linear classifier. Besides, for FedIIH, we utilize a node feature projection layer from DisenGCN~\cite{pmlrv97ma19a} to extract node representations, which are then classified by a Multi-Layer Perceptron (MLP). In contrast, for our proposed FedSSA, we utilize a spectral GNN (\textit{i.e.}, UniFilter~\cite{huanguniversal}) to extract node representations, which are subsequently fed into an MLP for node classification. To ensure a fair comparison, we employ UniFilter as the backbone for both local training and FedAvg. Moreover, for ease of implementation, we employ k-means for both semantic clustering and structural clustering.

\section{Pseudocode of Our Proposed FedSSA}
In this section, we show the pseudocode of our proposed FedSSA for clients and server in Algorithm~\ref{alg:client} and Algorithm~\ref{alg:server}, respectively. 

\begin{algorithm}[tb]
  \caption{\textbf{FedSSA} Client Algorithm}
  \label{alg:client}
  \begin{algorithmic}
    \STATE {\bfseries Input:} Number of local training epochs $E$; number of classes $C$; the order of bases $K$; local subgraph $\mathcal{G}_m$; local node features $\mathbf{X}_m$; local labels $\mathbf{Y}_m$; received cluster-level representative distributions $\{\mathcal{N}(\boldsymbol{\mu}_{i}^c, \boldsymbol{\Sigma}_{i}^c)\}_{c=1}^C$ from server; received cluster-level coefﬁcients $\{\overline{w}_j^k\}_{k=0}^K$ from server.
    \STATE {\bfseries Output:} Predicted labels for unlabeled nodes in local subgraph $\mathcal{G}_m$.
    \STATE Download cluster-level representative distributions $\{\mathcal{N}(\boldsymbol{\mu}_{i}^c, \boldsymbol{\Sigma}_{i}^c)\}_{c=1}^C$ from server;
    \STATE Download cluster-level coefﬁcients $\{\overline{w}_j^k\}_{k=0}^K$ from server;
    \FOR{each local epoch $e$ from 1 {\bfseries to} $E$}
    \STATE \textcolor{blue}{\# Semantic Knowledge Alignment}
    \STATE Employ VGAE to infer class-wise latent distributions $q(\mathbf{Z}_m^c)=\mathcal{N}(\boldsymbol{\mu}_m^c, \boldsymbol{\Sigma}_m^c)$ for each class $c$ via Eq.~\eqref{eq3}, where $c = 1, 2, \cdots, C$;
    \STATE Compute semantic knowledge alignment loss $\mathcal{L}_\mathrm{node}$ via Eq.~\eqref{eq7};
    \STATE \textcolor{blue}{\# Structural Knowledge Alignment}
    \STATE Compute spectral energy measure $\mathbf{S}_m$ via Eq.~\eqref{eq8} and Eq.~\eqref{eq-spectral-energy};
    \STATE Compute structural knowledge sharing loss $\mathcal{L}_\mathrm{struct}$ via Eq.~\eqref{eq11} and Eq.~\eqref{eq12};
    \STATE Update local model parameters by minimizing the overall local objective (see Appendix~\ref{ap_local_objective} for details);
    \ENDFOR
    \STATE Upload coefﬁcients $\{w_m^k\}_{k=0}^K$ to server;
    \STATE Upload inferred class-wise latent distributions $\{q(\mathbf{Z}_m^c)\}_{c=1}^C$ to server;
    \STATE Upload computed spectral energy measure $\mathbf{S}_m$ to server;
    \STATE Predict labels for unlabeled nodes in local subgraph $\mathcal{G}_m$.
  \end{algorithmic}
\end{algorithm}

\begin{algorithm}[tb]
  \caption{\textbf{FedSSA} Server Algorithm}
  \label{alg:server}
  \begin{algorithmic}
    \STATE {\bfseries Input:} Number of communication rounds $R$; number of clients $M$; number of classes $C$; the order of bases $K$; coefﬁcients $\{w_m^k\}_{k=0}^K$ from client $m$; inferred class-wise latent distributions $\{q(\mathbf{Z}_m^c)\}_{c=1}^C$ from client $m$; spectral energy measure $\mathbf{S}_m$ from client $m$.
    \STATE {\bfseries Output:} Cluster-level representative distributions $\{\mathcal{N}(\boldsymbol{\mu}_{i}^c, \boldsymbol{\Sigma}_{i}^c)\}_{c=1}^C$; cluster-level coefﬁcients $\{\overline{w}_j^k\}_{k=0}^K$.
    \FOR{each communication round $r$ from 1 {\bfseries to} $R$}
    \FOR{client $m \in \{1, 2, \cdots, M\}$ {\bfseries in parallel}}
    \STATE Perform Algorithm~\ref{alg:client} on client $m$;
    \STATE Receive coefﬁcients $\{w_m^k\}_{k=0}^K$ from client $m$;
    \STATE Receive inferred class-wise latent distributions $\{q(\mathbf{Z}_m^c)=\mathcal{N}(\boldsymbol{\mu}_m^c, \boldsymbol{\Sigma}_m^c)\}_{c=1}^C$ from client $m$;
    \STATE Receive computed spectral energy measure $\mathbf{S}_m$ from client $m$;
    \ENDFOR
    \STATE \textcolor{blue}{\# Semantic Clustering}
    \FOR{each class $c$ from 1 {\bfseries to} $C$}
        \STATE Employ reparameterization trick to obtain $\tilde{\mathbf{Z}}_m^c$, where ${m\in\mathcal{M}}$, $\tilde{\mathbf{Z}}_m^c=\boldsymbol{\mu}_m^c + (\boldsymbol{\Sigma}_m^c)^{\frac{1}{2}} \bm{\epsilon}_m^c$;
        \STATE Cluster $\{\tilde{\mathbf{Z}}_m^c \mid m \in \mathcal{M}\}$ into $K_\mathrm{node}$ clusters to obtain $\{\mathcal{S}_i^c\}_{i=1}^{K_\mathrm{node}}$ via Eq.~\eqref{eq4};
        \STATE Construct cluster-level representative distributions $\{\mathcal{N}(\boldsymbol{\mu}_{i}^c, \boldsymbol{\Sigma}_{i}^c)\}_{i=1}^{K_\mathrm{node}}$ via Eq.~\eqref{eq6};
        \FOR{each client $m \in \mathcal{M}$}
            \STATE Find $i$ such that $m \in \mathcal{S}_i^c$ and record the semantic cluster index of client $m$ as $i$;
            \STATE Send $\mathcal{N}(\boldsymbol{\mu}_{i}^c, \boldsymbol{\Sigma}_{i}^c)$ to client $m$;
        \ENDFOR
    \ENDFOR
    \STATE \textcolor{blue}{\# Structural Clustering}
    \STATE Compute orthonormal bases $\mathbf{Q}_m$ from $\mathbf{S}_m$ via QR decomposition, where $m \in \mathcal{M}$;
    \STATE Measure pairwise Chordal distances $d_{\mathrm{Chordal}}(\mathbf{Q}_m,\mathbf{Q}_n)$ via Eq.~\eqref{Chordal_distance};
    \STATE Cluster clients into $K_\mathrm{struct}$ clusters to obtain $\{\mathcal{T}_j\}_{j=1}^{K_\mathrm{struct}}$ via Eq.~\eqref{eq9};
    \FOR{each cluster $j$ from 1 {\bfseries to} $K_\mathrm{struct}$}
        \STATE Compute cluster-level coefficients $\{\overline{w}_j^k\}_{k=0}^K$ via Eq.~\eqref{eq10};
    \ENDFOR
    \FOR{each client $m \in \mathcal{M}$}
        \STATE Find $j$ such that $m \in \mathcal{T}_j$ and record the structural cluster index of client $m$ as $j$;
        \STATE Send $\{\overline{w}_j^k\}_{k=0}^K$ to client $m$.
    \ENDFOR
    \ENDFOR
  \end{algorithmic}
\end{algorithm}

\section{Additional Experiments}
In this section, we provide additional experiments. First, we provide additional experimental results on eleven datasets under overlapping partitioning setting. Second, we provide supplementary ablation studies on eleven datasets under both non-overlapping and overlapping partitioning settings. Third, we present additional convergence curves on other datasets under two representative settings, namely non-overlapping with 10 clients and overlapping with 30 clients. Fourth, we provide additional sensitivity analysis on hyperparameters. Fifth, we present additional case studies on other datasets to illustrate the effectiveness of our FedSSA in mitigating heterogeneity among clients. Finally, we analyze computational efficiency by reporting time consumption (seconds) per communication round. 

\subsection{Additional Experiments on Overlapping Subgraph Partitioning Setting}
\label{additional_tables}
Due to space limitations, we only present experimental results under non-overlapping partitioning setting in the main paper. Here, we provide additional experimental results on the same eleven datasets under overlapping partitioning setting in~\cref{table2} and~\cref{table4}. We can observe that our proposed FedSSA consistently outperforms all baseline methods on all datasets with varying numbers of clients. For example, in~\cref{table4}, the average accuracy of FedSSA is 63.90\%, which is 1.40\% higher than the second-best method (\textit{i.e.}, FedIIH). This further validates the effectiveness of our proposed FedSSA in handling graph data heterogeneity in GFL scenarios.

  \begin{table*}[t]
      \centering
          \scriptsize
          \caption{Accuracy (\%) of methods on six \textbf{homophilic} graph datasets under \textbf{overlapping} subgraph partitioning setting.}
          \label{table2}
    \renewcommand{\arraystretch}{0.9} % 调整行间距
           \scalebox{0.8}{
      \begin{tabular}{lcccccccccc}
      \hline
      \rowcolor{gray!50}
      \multicolumn{1}{c}{} & \multicolumn{3}{c}{Cora}                                                    & \multicolumn{3}{c}{CiteSeer}                                                      & \multicolumn{3}{c}{PubMed}                                                  & -              \\ \cline{2-11} 
      Methods              & 10 Clients              & 30 Clients              & 50 Clients              & 10 Clients                    & 30 Clients              & 50 Clients              & 10 Clients              & 30 Clients              & 50 Clients              & -              \\ \hline
      \rowcolor{gray!20}
      Local                & 73.98$\pm$0.25          & 71.65$\pm$0.12          & 76.63$\pm$0.10          & 65.12$\pm$0.08                & 64.54$\pm$0.42          & 66.68$\pm$0.44          & 82.32$\pm$0.07          & 80.72$\pm$0.16          & 80.54$\pm$0.11          & -              \\ \hline
      FedAvg~\cite{mcmahan2017communication}               & 76.48$\pm$0.36          & 53.99$\pm$0.98          & 53.99$\pm$4.53          & 69.48$\pm$0.15                & 66.15$\pm$0.64          & 66.51$\pm$1.00          & 82.67$\pm$0.11          & 82.05$\pm$0.12          & 80.24$\pm$0.35          & -              \\
      \rowcolor{gray!20}
      FedProx~\cite{MLSYS2020_1f5fe839}              & 77.85$\pm$0.50          & 51.38$\pm$1.74          & 56.27$\pm$9.04          & 69.39$\pm$0.35                & 66.11$\pm$0.75          & 66.53$\pm$0.43          & 82.63$\pm$0.17          & 82.13$\pm$0.13          & 80.50$\pm$0.46          & -              \\
      FedPer~\cite{Arivazhagan2019}               & 78.73$\pm$0.31          & 74.18$\pm$0.24          & 74.42$\pm$0.37          & 69.81$\pm$0.28                & 65.19$\pm$0.81          & 67.64$\pm$0.44          & 85.31$\pm$0.06          & 84.35$\pm$0.38          & 83.94$\pm$0.10          & -              \\
      \rowcolor{gray!20}
      GCFL~\cite{NEURIPS2021_9c6947bd}                 & 78.84$\pm$0.26          & 73.41$\pm$0.27          & 76.63$\pm$0.16          & 69.48$\pm$0.39                & 64.92$\pm$0.18          & 65.98$\pm$0.30          & 83.59$\pm$0.25          & 80.77$\pm$0.12          & 81.36$\pm$0.11          & -              \\
      FedGNN~\cite{wu2021fedgnn}               & 70.63$\pm$0.83          & 61.38$\pm$2.33          & 56.91$\pm$0.82          & 68.72$\pm$0.39                & 59.98$\pm$1.52          & 58.98$\pm$0.98          & 84.25$\pm$0.07          & 82.02$\pm$0.22          & 81.85$\pm$0.10          & -              \\
      \rowcolor{gray!20}
      FedSage+\cite{NEURIPS2021_34adeb8e}             & 77.52$\pm$0.46          & 51.99$\pm$0.42          & 55.48$\pm$11.5          & 68.75$\pm$0.48                & 65.97$\pm$0.02          & 65.93$\pm$0.30          & 82.77$\pm$0.08          & 82.14$\pm$0.11          & 80.31$\pm$0.68          & -              \\
      FED-PUB~\cite{baek2023personalized}              & 79.60$\pm$0.12          & 75.40$\pm$0.54          & \underline{77.84$\pm$0.23}          & 70.58$\pm$0.20                & 68.33$\pm$0.45          & 69.21$\pm$0.30          & 85.70$\pm$0.08          & 85.16$\pm$0.10          & 84.84$\pm$0.12          & -              \\
      \rowcolor{gray!20}
      FedGTA~\cite{li2023fedgta}               & 76.42$\pm$0.62          & 75.63$\pm$0.33          & 77.69$\pm$0.14          & 70.43$\pm$0.08 & 71.71$\pm$0.33          & 69.19$\pm$0.32          & 85.34$\pm$0.42          & 84.99$\pm$0.05          & 84.47$\pm$0.06          & -              \\
      AdaFGL~\cite{li2024adafgl}               & 78.50$\pm$0.19          & 75.80$\pm$0.23          & 74.41$\pm$0.00          & 72.63$\pm$0.15 & 68.18$\pm$0.31          & 62.90$\pm$0.75          & 85.58$\pm$0.23          & 85.85$\pm$0.41          & 84.45$\pm$0.07          & -              \\
	    \rowcolor{gray!20}
      FedTAD~\cite{zhu2024fedtad}              & 79.29$\pm$0.78          & 60.92$\pm$2.17          & 68.08$\pm$0.44          & \textbf{73.47$\pm$0.16}            & 67.74$\pm$0.57          & 63.51$\pm$0.68          & 82.98$\pm$0.20          & 82.11$\pm$0.15          & 81.63$\pm$0.19          & -              \\ 
      FedIIH~\cite{wentao2025fediih}        & \underline{80.57$\pm$0.23}  & \underline{76.82$\pm$0.24} & \textbf{78.58$\pm$0.25} & \underline{73.16$\pm$0.18}       & \underline{72.27$\pm$0.21}    & \underline{69.56$\pm$0.11} & \underline{85.87$\pm$0.03} & \underline{86.65$\pm$0.11} & \textbf{85.65$\pm$0.12} & -              \\ \hline
      \rowcolor{yellow!30}
      FedSSA (Ours)                         & \textbf{81.02$\pm$0.09}    & \textbf{76.89$\pm$0.06} & \textbf{78.58$\pm$0.15}          & 72.19$\pm$0.08                & \textbf{72.38$\pm$0.13}    & \textbf{69.68$\pm$0.10} & \textbf{86.40$\pm$0.09} & \textbf{86.83$\pm$0.06} & \underline{85.20$\pm$0.04} & -              \\ \hline
      \rowcolor{gray!50}
                           & \multicolumn{3}{c}{Amazon-Computer}                                         & \multicolumn{3}{c}{Amazon-Photo}                                                  & \multicolumn{3}{c}{ogbn-arxiv}                                              & Avg.            \\ \cline{2-11} 
      Methods              & 10 Clients              & 30 Clients              & 50 Clients              & 10 Clients                    & 30 Clients              & 50 Clients              & 10 Clients              & 30 Clients              & 50 Clients              & All              \\ \hline
      \rowcolor{gray!20}
      Local                & 88.50$\pm$0.20          & 86.66$\pm$0.00          & 87.04$\pm$0.02          & 92.17$\pm$0.12                & 90.16$\pm$0.12          & 90.42$\pm$0.15          & 62.52$\pm$0.07          & 61.32$\pm$0.04          & 60.04$\pm$0.04          & 76.72          \\ \hline
      FedAvg~\cite{mcmahan2017communication}               & 88.99$\pm$0.19          & 83.37$\pm$0.47          & 76.34$\pm$0.12          & 92.91$\pm$0.07                & 89.30$\pm$0.22          & 74.19$\pm$0.57          & 63.56$\pm$0.02          & 59.72$\pm$0.06          & 60.94$\pm$0.24          & 73.38          \\
      \rowcolor{gray!20}
      FedProx~\cite{MLSYS2020_1f5fe839}              & 88.84$\pm$0.20          & 83.84$\pm$0.89          & 76.60$\pm$0.47          & 92.67$\pm$0.19                & 89.17$\pm$0.40          & 72.36$\pm$2.06          & 63.52$\pm$0.11          & 59.86$\pm$0.16          & 61.12$\pm$0.04          & 73.38          \\
      FedPer~\cite{Arivazhagan2019}               & 89.30$\pm$0.04          & 87.99$\pm$0.23          & 88.22$\pm$0.27          & 92.88$\pm$0.24                & 91.23$\pm$0.16          & 90.92$\pm$0.38          & 63.97$\pm$0.08          & 62.29$\pm$0.04          & 61.24$\pm$0.11          & 78.42          \\
      \rowcolor{gray!20}
      GCFL~\cite{NEURIPS2021_9c6947bd}                 & 89.01$\pm$0.22          & 87.24$\pm$0.09          & 87.02$\pm$0.22          & 92.45$\pm$0.10                & 90.58$\pm$0.11          & 90.54$\pm$0.08          & 63.24$\pm$0.02          & 61.66$\pm$0.10          & 60.32$\pm$0.01          & 77.61          \\
      FedGNN~\cite{wu2021fedgnn}               & 88.15$\pm$0.09          & 87.00$\pm$0.10          & 83.96$\pm$0.88          & 91.47$\pm$0.11                & 87.91$\pm$1.34          & 78.90$\pm$6.46          & 63.08$\pm$0.19          & 60.09$\pm$0.04          & 60.51$\pm$0.11          & 73.66          \\
      \rowcolor{gray!20}
      FedSage+\cite{NEURIPS2021_34adeb8e}             & 89.24$\pm$0.15          & 81.33$\pm$1.20          & 76.72$\pm$0.39          & 92.76$\pm$0.05                & 88.69$\pm$0.99          & 72.41$\pm$1.36          & 63.24$\pm$0.02          & 59.90$\pm$0.12          & 60.95$\pm$0.09          & 73.12          \\
      FED-PUB~\cite{baek2023personalized}              & 89.98$\pm$0.08          & 89.15$\pm$0.06          & 88.76$\pm$0.14          & 93.22$\pm$0.07                & 92.01$\pm$0.07          & 91.71$\pm$0.11          & 64.18$\pm$0.04          & 63.34$\pm$0.12          & 62.55$\pm$0.12          & 79.53          \\
      \rowcolor{gray!20}
      FedGTA~\cite{li2023fedgta}               & 90.10$\pm$0.18          & 88.79$\pm$0.27          & 88.15$\pm$0.21          & 93.13$\pm$0.14                & 92.49$\pm$0.06          & 91.77$\pm$0.06          & 55.98$\pm$0.09          & 56.76$\pm$0.07          & 57.89$\pm$0.09          & 78.39          \\
      AdaFGL~\cite{li2024adafgl}        & 80.49$\pm$0.00          & 80.42$\pm$0.00          & 82.12$\pm$0.00          & 89.24$\pm$0.00          & 88.34$\pm$0.00          & 87.68$\pm$0.00          & 56.81$\pm$0.06                   & 55.17$\pm$0.00                   & 54.82$\pm$0.00                         & 75.74          \\ 
      \rowcolor{gray!20}
      FedTAD~\cite{zhu2024fedtad}        & 79.09$\pm$5.63          & 79.48$\pm$0.85          & 77.05$\pm$0.07          & 81.94$\pm$3.09          & 86.58$\pm$1.75          & 84.38$\pm$1.33          & 58.45$\pm$0.15                   & 57.75$\pm$0.54                   & 56.52$\pm$0.14                         & 73.39          \\
      FedIIH~\cite{wentao2025fediih}        & \underline{90.15$\pm$0.04} & \underline{89.56$\pm$0.19}                 & \textbf{89.99$\pm$0.00} & \underline{93.38$\pm$0.00}                & \underline{94.17$\pm$0.04} & \underline{93.25$\pm$0.16} & \underline{66.69$\pm$0.09}          & \underline{66.10$\pm$0.03}          & \underline{65.67$\pm$0.06} & \underline{81.01} \\ \hline
      \rowcolor{yellow!30}
      FedSSA (Ours)                       & \textbf{90.41$\pm$0.07} & \textbf{89.65$\pm$0.12} & \underline{89.34$\pm$0.11} & \textbf{93.72$\pm$0.15}       & \textbf{94.44$\pm$0.11} & \textbf{93.36$\pm$0.14} & \textbf{67.44$\pm$0.12} & \textbf{66.38$\pm$0.13} & \textbf{65.73$\pm$0.07} & \textbf{81.09} \\ \hline
      \end{tabular}
      }
  \end{table*}

  \begin{table*}[t]
      \centering
      \scriptsize
      \caption{Comparisons on five \textbf{heterophilic} graph datasets under \textbf{overlapping} subgraph partitioning setting. Accuracy (\%) is reported for \textit{Roman-empire} and \textit{Amazon-ratings}, and AUC (\%) is reported for \textit{Minesweeper}, \textit{Tolokers}, and \textit{Questions}.}
      \label{table4}
    \renewcommand{\arraystretch}{0.9} % 调整行间距
       \scalebox{0.8}{
  \begin{tabular}{lcccccccccc}
  \hline
  \rowcolor{gray!50}
  \textbf{}     & \multicolumn{3}{c}{Roman-empire}                                            & \multicolumn{3}{c}{Amazon-ratings}                                          & \multicolumn{3}{c}{Minesweeper}                                                   & -              \\ \cline{2-11} 
  Methods       & 10 Clients              & 30 Clients              & 50 Clients              & 10 Clients              & 30 Clients              & 50 Clients              & 10 Clients              & 30 Clients              & 50 Clients                    & -              \\ \hline
  \rowcolor{gray!20}
  Local         & 39.47$\pm$0.03          & 34.43$\pm$0.14          & 31.28$\pm$0.18          & 41.43$\pm$0.04          & 41.81$\pm$0.14          & 42.57$\pm$0.12          & 67.98$\pm$0.07          & 64.39$\pm$0.10          & 62.73$\pm$0.23                & -              \\ \hline
  FedAvg~\cite{mcmahan2017communication}        & 40.89$\pm$0.25          & 38.66$\pm$0.08          & 36.71$\pm$0.20          & 39.86$\pm$0.06          & 41.40$\pm$0.02          & 41.02$\pm$0.16          & 69.06$\pm$0.07          & 67.95$\pm$0.04          & 66.89$\pm$0.08                & -              \\
  \rowcolor{gray!20}
  FedProx~\cite{MLSYS2020_1f5fe839}       & 36.63$\pm$0.14          & 35.31$\pm$0.17          & 33.61$\pm$0.59          & 37.53$\pm$0.09          & 37.43$\pm$0.08          & 37.40$\pm$0.07          & 68.27$\pm$0.05          & 66.75$\pm$0.19          & 66.03$\pm$0.16                & -              \\
  FedPer~\cite{Arivazhagan2019}        & 23.66$\pm$3.27          & 23.27$\pm$3.09          & 22.23$\pm$3.58          & 32.33$\pm$4.23          & 31.58$\pm$0.54          & 34.48$\pm$2.25          & 61.85$\pm$1.02          & 60.13$\pm$1.38          & 60.06$\pm$3.61                & -              \\
  \rowcolor{gray!20}
  GCFL~\cite{NEURIPS2021_9c6947bd}          & 39.97$\pm$0.89          & 38.63$\pm$0.49          & 36.87$\pm$0.31          & 39.54$\pm$0.41          & 42.12$\pm$0.11          & 41.27$\pm$0.22          & 69.16$\pm$0.13          & 68.02$\pm$0.10          & 66.93$\pm$1.34               & -              \\
  FedGNN~\cite{wu2021fedgnn}        & 37.46$\pm$0.12          & 36.47$\pm$0.24          & 34.92$\pm$0.26          & 36.58$\pm$0.16          & 36.77$\pm$0.12          & 36.95$\pm$0.15          & 68.59$\pm$0.21          & 67.30$\pm$0.17          & 66.41$\pm$0.23                & -              \\
  \rowcolor{gray!20}
  FedSage+\cite{NEURIPS2021_34adeb8e}      & 57.48$\pm$0.00          & 42.55$\pm$0.00          & 37.13$\pm$0.00          & 36.86$\pm$0.00          & 36.71$\pm$0.00          & 37.03$\pm$0.00          & \textbf{76.64$\pm$0.00} & \underline{70.56$\pm$0.00} &  \underline{70.34$\pm$0.00} & -              \\
  FED-PUB~\cite{baek2023personalized}       & 43.80$\pm$0.25          & 40.46$\pm$0.16          & 37.73$\pm$0.09          & 42.25$\pm$0.25    &  42.30$\pm$0.06    & \underline{42.88$\pm$0.34} & 69.11$\pm$0.13          & 67.76$\pm$0.24          & 67.52$\pm$0.14                & -              \\
  \rowcolor{gray!20}
  FedGTA~\cite{li2023fedgta}        & 59.86$\pm$0.04    & 58.32$\pm$0.09    & 57.57$\pm$0.21    & 40.81$\pm$0.24          & 39.44$\pm$0.06          & 39.37$\pm$0.04          & 70.64$\pm$0.40          & 67.99$\pm$1.60          & 67.20$\pm$1.35                & -              \\
  AdaFGL~\cite{li2024adafgl}        & 64.44$\pm$0.03          & 61.77$\pm$0.02          & 59.55$\pm$0.01          & 39.39$\pm$0.05          & 41.19$\pm$0.15          & 40.71$\pm$0.25          & 69.07$\pm$0.72                   & 68.34$\pm$1.82                   & 66.80$\pm$1.31                         & -          \\
  \rowcolor{gray!20}
  FedTAD~\cite{zhu2024fedtad}        & 44.14$\pm$0.13          & 41.94$\pm$0.18          & 40.82$\pm$0.01          & 39.53$\pm$0.17          & 40.69$\pm$0.13          & 40.58$\pm$0.26          & 69.27$\pm$0.33                   & 68.43$\pm$0.05                   & 67.23$\pm$0.08                         & -          \\
  FedIIH~\cite{wentao2025fediih} & \underline{65.48$\pm$0.12} & \underline{63.32$\pm$0.06} & \underline{62.42$\pm$0.10}          & \underline{42.63$\pm$0.02}           & \underline{42.40$\pm$0.05}          & 42.65$\pm$0.21              & 69.35$\pm$0.25    &  68.09$\pm$0.26    &  67.37$\pm$0.14 & -         \\ \hline
  \rowcolor{yellow!30}
 FedSSA (Ours) & \textbf{65.66$\pm$0.07} & \textbf{63.80$\pm$0.09} & \textbf{62.75$\pm$0.14} & \textbf{42.83$\pm$0.06} & \textbf{42.52$\pm$0.10} & \textbf{42.97$\pm$0.15}    & \underline{75.65$\pm$0.11}    & \textbf{72.60$\pm$0.06}    & \textbf{71.31$\pm$0.05} & -              \\ \hline
 \rowcolor{gray!50}
                & \multicolumn{3}{c}{Tolokers}                                                & \multicolumn{3}{c}{Questions}                                               & \multicolumn{4}{c}{Avg.}                                                                           \\ \cline{2-11} 
  Methods       & 10 Clients              & 30 Clients              & 50 Clients              & 10 Clients              & 30 Clients              & 50 Clients              & 10 Clients              & 30 Clients              & 50 Clients                    & All            \\ \hline
  \rowcolor{gray!20}
  Local         & 73.83$\pm$0.03          & 69.01$\pm$0.31          & 66.63$\pm$0.20          & 63.17$\pm$0.02          & 57.17$\pm$0.08          & 56.13$\pm$0.02          & 57.18                   & 53.36                   & 51.87                         & 54.14          \\ \hline
  FedAvg~\cite{mcmahan2017communication}        & 72.99$\pm$0.40          & 58.51$\pm$0.27          & 55.47$\pm$0.42          & 62.80$\pm$0.63          & 58.88$\pm$0.18          & 60.78$\pm$0.27          & 57.12                   & 53.08                   & 52.17                         & 54.12          \\
  \rowcolor{gray!20}
  FedProx~\cite{MLSYS2020_1f5fe839}       & 54.49$\pm$1.69          & 45.59$\pm$0.41          & 41.49$\pm$0.45          & 52.53$\pm$0.34          & 51.54$\pm$0.41          & 50.72$\pm$0.40          & 49.89                   & 47.32                   & 45.85                         & 47.69          \\
  FedPer~\cite{Arivazhagan2019}        & 39.60$\pm$0.11          & 59.44$\pm$0.79          & 41.92$\pm$0.06          & 61.31$\pm$0.29          & 53.41$\pm$1.53          & 50.29$\pm$0.10          & 43.75                   & 45.57                   & 41.80                         & 43.70          \\
  \rowcolor{gray!20}
  GCFL~\cite{NEURIPS2021_9c6947bd}          & 70.61$\pm$0.55          & 59.72$\pm$0.50          & 57.64$\pm$0.71          & 62.84$\pm$0.60          & 59.46$\pm$0.68          & 60.24$\pm$0.41          & 56.42                   & 53.59                   & 52.59                         & 54.20          \\
  FedGNN~\cite{wu2021fedgnn}        & 56.21$\pm$1.20          & 46.85$\pm$0.31          & 42.18$\pm$0.45          & 53.25$\pm$0.15          & 51.90$\pm$0.15          & 51.22$\pm$0.14          & 50.42                   & 47.86                   & 46.34                         & 48.20          \\
  \rowcolor{gray!20}
  FedSage+\cite{NEURIPS2021_34adeb8e}      & \textbf{74.54$\pm$0.00} &  70.88$\pm$0.00    &  69.61$\pm$0.00    & 64.22$\pm$0.00          &  65.34$\pm$0.00    &  62.76$\pm$0.00    & 61.95             & 57.21             & 55.37              & 58.18    \\
  FED-PUB~\cite{baek2023personalized}       & 74.17$\pm$0.29    & 70.35$\pm$0.54          & 66.80$\pm$0.85          &  65.39$\pm$2.44    & 58.38$\pm$1.19          & 60.73$\pm$0.74          & 58.94                   & 55.85                   & 55.13                         & 56.64          \\
  \rowcolor{gray!20}
  FedGTA~\cite{li2023fedgta}        & 70.34$\pm$1.53          & 59.59$\pm$1.10          & 56.12$\pm$1.48          & 63.20$\pm$1.27          & 58.11$\pm$1.06          & 60.99$\pm$0.77          & 60.97                   & 56.69                   & 56.25                         & 57.97          \\
  AdaFGL~\cite{li2024adafgl}        & 70.01$\pm$1.91          & 58.94$\pm$1.11          & 56.25$\pm$1.35          & 61.90$\pm$0.80          & 58.93$\pm$2.18          & 60.68$\pm$0.63          & 60.96                   & 57.83                   & 56.80                         & 58.53          \\
  \rowcolor{gray!20}
  FedTAD~\cite{zhu2024fedtad}        & 69.34$\pm$1.26          & 62.11$\pm$0.27          & 56.39$\pm$0.52          & 61.96$\pm$0.54         & 59.24$\pm$0.36          & 60.24$\pm$0.92          & 56.85                   & 54.48                   & 53.05                         & 54.79          \\ 
  FedIIH~\cite{wentao2025fediih}     & 71.67$\pm$0.02          & \underline{71.69$\pm$0.12}          & \underline{69.99$\pm$0.03}      & \underline{68.79$\pm$0.09}   & \textbf{66.98$\pm$0.04} & \underline{64.73$\pm$0.35}          & \underline{63.58}          & \underline{62.50}          & \underline{61.43}                & \underline{62.50} \\ \hline
  \rowcolor{yellow!30}
  FedSSA (Ours)        & \underline{74.33$\pm$0.09}          & \textbf{72.27$\pm$0.07} & \textbf{71.03$\pm$0.09}         & \textbf{69.39$\pm$0.05}    & \underline{66.43$\pm$0.13}  & \textbf{64.94$\pm$0.12} & \textbf{65.57}          & \textbf{63.52}          & \textbf{62.60}                & \textbf{63.90} \\ \hline
      \end{tabular}
       }
\end{table*}

\subsection{Additional Ablation Studies}\label{ap_ablation_study}
To further evaluate the contribution of each key component in our proposed FedSSA, we carry out ablation studies on eleven datasets under both non-overlapping and overlapping partitioning settings. Specifically, since there are two essential components in our FedSSA (\textit{i.e.}, sharing semantic knowledge and sharing structural knowledge), we employ `Semantic' and `Structural' to represent them, respectively. Based on the inclusion or exclusion of these components, we consider four different combinations, which are shown in~\cref{table_ablation}. We can observe that performances consistently degrade across all datasets when any individual component is removed. This demonstrates that each component is essential and contributes significantly to the overall effectiveness of our proposed FedSSA.

\begin{table*}[t]
    \centering
    \scriptsize
    \caption{Ablation studies are conducted under both non-overlapping and overlapping partitioning settings on eleven datasets.}
    \label{table_ablation}
    \renewcommand{\arraystretch}{0.9} % 调整行间距
    \scalebox{0.9}{
    \begin{tabular}{cccccccc}
      \hline
  \rowcolor{gray!50}
           &               & \multicolumn{6}{c}{Cora}                                                                                                                                                                                                                                                                                                                                                                                                       \\ \hline
           Semantic  & Structural & \begin{tabular}[c]{@{}c@{}}non-overlapping\\ 5 clients\end{tabular} & \begin{tabular}[c]{@{}c@{}}non-overlapping\\ 10 clients\end{tabular} & \begin{tabular}[c]{@{}c@{}}non-overlapping\\ 20 clients\end{tabular} & \begin{tabular}[c]{@{}c@{}}overlapping\\ 10 clients\end{tabular} & \begin{tabular}[c]{@{}c@{}}overlapping\\ 30 clients\end{tabular} & \begin{tabular}[c]{@{}c@{}}overlapping\\ 50 clients\end{tabular} \\ \hline
  \textcolor{teal}{\CheckmarkBold}  & \textcolor{red}{\XSolidBrush}      & 80.07$\pm$0.50 ($\downarrow$ 4.60)                                  & 78.52$\pm$0.69 ($\downarrow$ 3.80)                                   & 80.44$\pm$0.11 ($\downarrow$ 3.69)                                   & 76.91$\pm$0.89 ($\downarrow$ 4.11)                               & 72.33$\pm$0.39 ($\downarrow$ 4.56)                                   & 75.11$\pm$0.13 ($\downarrow$ 3.47)                               \\ \hline
  \textcolor{red}{\XSolidBrush} & \textcolor{teal}{\CheckmarkBold}       & 79.32$\pm$0.32 ($\downarrow$ 5.35)                                 & 78.30$\pm$1.11 ($\downarrow$ 4.02)                                   & 79.89$\pm$0.09 ($\downarrow$ 4.24)                                   & 75.99$\pm$1.00 ($\downarrow$ 5.03)                               & 72.17$\pm$0.51 ($\downarrow$ 4.72)                                   & 74.80$\pm$0.21 ($\downarrow$ 3.78)                               \\ \hline
  \textcolor{red}{\XSolidBrush} & \textcolor{red}{\XSolidBrush}      & 78.85$\pm$0.15 ($\downarrow$ 5.82)                                 & 78.03$\pm$1.04 ($\downarrow$ 4.29)                                   & 79.63$\pm$0.56 ($\downarrow$ 4.50)                                   & 75.39$\pm$0.55 ($\downarrow$ 5.63)                               & 71.84$\pm$0.74 ($\downarrow$ 5.05)                                   & 74.65$\pm$0.19 ($\downarrow$ 3.93)                               \\ \hline
  \textcolor{teal}{\CheckmarkBold}  & \textcolor{teal}{\CheckmarkBold}       & \textbf{84.67$\pm$0.05}                                             & \textbf{82.32$\pm$0.04}                                              & \textbf{84.13$\pm$0.09}                                              & \textbf{81.02$\pm$0.09}                                          & \textbf{76.89$\pm$0.06}                                              & \textbf{78.58$\pm$0.15}                                          \\ \hline
    \rowcolor{gray!50}
             &               & \multicolumn{6}{c}{CiteSeer}                                                                                                                                                                                                                                                                                                                                                                                                       \\ \hline
             Semantic  & Structural & \begin{tabular}[c]{@{}c@{}}non-overlapping\\ 5 clients\end{tabular} & \begin{tabular}[c]{@{}c@{}}non-overlapping\\ 10 clients\end{tabular} & \begin{tabular}[c]{@{}c@{}}non-overlapping\\ 20 clients\end{tabular} & \begin{tabular}[c]{@{}c@{}}overlapping\\ 10 clients\end{tabular} & \begin{tabular}[c]{@{}c@{}}overlapping\\ 30 clients\end{tabular} & \begin{tabular}[c]{@{}c@{}}overlapping\\ 50 clients\end{tabular} \\ \hline
    \textcolor{teal}{\CheckmarkBold}  & \textcolor{red}{\XSolidBrush}      & 69.76$\pm$1.28 ($\downarrow$ 3.30)                                  & 74.08$\pm$0.23 ($\downarrow$ 3.57)                                   & 71.48$\pm$0.94 ($\downarrow$ 2.61)                                   & 68.11$\pm$0.54 ($\downarrow$ 4.08)                               & 68.36$\pm$0.08 ($\downarrow$ 4.02)                                   & 66.38$\pm$1.13 ($\downarrow$ 3.30)                               \\ \hline
    \textcolor{red}{\XSolidBrush} & \textcolor{teal}{\CheckmarkBold}       & 69.41$\pm$0.19 ($\downarrow$ 3.65)                                 & 73.05$\pm$0.68 ($\downarrow$ 4.60)                                   & 70.82$\pm$1.27 ($\downarrow$ 3.27)                                   & 67.72$\pm$1.09 ($\downarrow$ 4.47)                               & 67.42$\pm$0.11 ($\downarrow$ 4.96)                                   & 65.84$\pm$0.36 ($\downarrow$ 3.84)                               \\ \hline
    \textcolor{red}{\XSolidBrush} & \textcolor{red}{\XSolidBrush}      & 69.20$\pm$0.08 ($\downarrow$ 3.86)                                 & 72.80$\pm$0.44 ($\downarrow$ 4.85)                                   & 69.51$\pm$1.18 ($\downarrow$ 4.58)                                   & 66.78$\pm$0.58 ($\downarrow$ 5.41)                               & 67.14$\pm$0.51 ($\downarrow$ 5.24)                                   & 65.63$\pm$0.30 ($\downarrow$ 4.05)                               \\ \hline
    \textcolor{teal}{\CheckmarkBold}  & \textcolor{teal}{\CheckmarkBold}       & \textbf{73.06$\pm$0.08}                                             & \textbf{77.65$\pm$0.10}                                              & \textbf{74.09$\pm$0.09}                                              & \textbf{72.19$\pm$0.08}                                          & \textbf{72.38$\pm$0.13}                                              & \textbf{69.68$\pm$0.10}                                          \\ \hline
    \rowcolor{gray!50}
             &               & \multicolumn{6}{c}{PubMed}                                                                                                                                                                                                                                                                                                                                                                                               \\ \hline
             Semantic  & Structural & \begin{tabular}[c]{@{}c@{}}non-overlapping\\ 5 clients\end{tabular} & \begin{tabular}[c]{@{}c@{}}non-overlapping\\ 10 clients\end{tabular} & \begin{tabular}[c]{@{}c@{}}non-overlapping\\ 20 clients\end{tabular} & \begin{tabular}[c]{@{}c@{}}overlapping\\ 10 clients\end{tabular} & \begin{tabular}[c]{@{}c@{}}overlapping\\ 30 clients\end{tabular} & \begin{tabular}[c]{@{}c@{}}overlapping\\ 50 clients\end{tabular} \\ \hline
    \textcolor{teal}{\CheckmarkBold}  & \textcolor{red}{\XSolidBrush}      & 85.05$\pm$0.30 ($\downarrow$ 3.06)                                  & 84.73$\pm$0.18 ($\downarrow$ 3.05)                                   & 84.31$\pm$0.14 ($\downarrow$ 3.06)                                   & 83.54$\pm$0.15 ($\downarrow$ 2.86)                               & 82.47$\pm$0.73 ($\downarrow$ 4.36)                                   & 81.73$\pm$0.30 ($\downarrow$ 3.47)                               \\ \hline
    \textcolor{red}{\XSolidBrush} & \textcolor{teal}{\CheckmarkBold}       & 84.64$\pm$0.56 ($\downarrow$ 3.47)                                 & 83.31$\pm$0.32 ($\downarrow$ 4.47)                                   & 83.65$\pm$0.24 ($\downarrow$ 3.72)                                   & 82.40$\pm$0.03 ($\downarrow$ 4.00)                               & 81.70$\pm$0.53 ($\downarrow$ 5.13)                                   & 81.22$\pm$0.76 ($\downarrow$ 3.98)                               \\ \hline
    \textcolor{red}{\XSolidBrush} & \textcolor{red}{\XSolidBrush}      & 84.37$\pm$0.16 ($\downarrow$ 3.74)                                 & 83.12$\pm$0.20 ($\downarrow$ 4.66)                                   & 83.24$\pm$0.28 ($\downarrow$ 4.13)                                   & 81.79$\pm$0.29 ($\downarrow$ 4.61)                               & 80.95$\pm$0.73 ($\downarrow$ 5.88)                                   & 80.72$\pm$0.65 ($\downarrow$ 4.48)                               \\ \hline
    \textcolor{teal}{\CheckmarkBold}  & \textcolor{teal}{\CheckmarkBold}       & \textbf{88.11$\pm$0.07}                                             & \textbf{87.78$\pm$0.13}                                              & \textbf{87.37$\pm$0.14}                                              & \textbf{86.40$\pm$0.09}                                          & \textbf{86.83$\pm$0.06}                                             & \textbf{85.20$\pm$0.04}                                          \\ \hline

    \rowcolor{gray!50}
             &               & \multicolumn{6}{c}{Amazon-Computer}                                                                                                                                                                                                                                                                                                                                                                                               \\ \hline
             Semantic  & Structural & \begin{tabular}[c]{@{}c@{}}non-overlapping\\ 5 clients\end{tabular} & \begin{tabular}[c]{@{}c@{}}non-overlapping\\ 10 clients\end{tabular} & \begin{tabular}[c]{@{}c@{}}non-overlapping\\ 20 clients\end{tabular} & \begin{tabular}[c]{@{}c@{}}overlapping\\ 10 clients\end{tabular} & \begin{tabular}[c]{@{}c@{}}overlapping\\ 30 clients\end{tabular} & \begin{tabular}[c]{@{}c@{}}overlapping\\ 50 clients\end{tabular} \\ \hline
    \textcolor{teal}{\CheckmarkBold}  & \textcolor{red}{\XSolidBrush}      & 88.53$\pm$0.10 ($\downarrow$ 2.56)                                  & 87.61$\pm$0.22 ($\downarrow$ 3.69)                                   & 88.88$\pm$0.84 ($\downarrow$ 1.74)                                   & 87.62$\pm$0.27 ($\downarrow$ 2.79)                               & 87.06$\pm$0.11 ($\downarrow$ 2.59)                                   & 86.29$\pm$0.33 ($\downarrow$ 3.05)                               \\ \hline
    \textcolor{red}{\XSolidBrush} & \textcolor{teal}{\CheckmarkBold}       & 88.18$\pm$0.74 ($\downarrow$ 2.91)                                 & 86.95$\pm$0.79 ($\downarrow$ 4.35)                                   & 88.42$\pm$0.73 ($\downarrow$ 2.20)                                   & 87.40$\pm$0.40 ($\downarrow$ 3.01)                               & 86.32$\pm$0.20 ($\downarrow$ 3.33)                                   & 85.32$\pm$0.71 ($\downarrow$ 4.02)                               \\ \hline
    \textcolor{red}{\XSolidBrush} & \textcolor{red}{\XSolidBrush}      & 87.72$\pm$0.13 ($\downarrow$ 3.37)                                 & 86.74$\pm$0.53 ($\downarrow$ 4.56)                                   & 88.32$\pm$0.31 ($\downarrow$ 2.30)                                   & 87.21$\pm$0.23 ($\downarrow$ 3.20)                               & 86.24$\pm$0.55 ($\downarrow$ 3.41)                                   & 85.13$\pm$0.19 ($\downarrow$ 4.21)                               \\ \hline
    \textcolor{teal}{\CheckmarkBold}  & \textcolor{teal}{\CheckmarkBold}       & \textbf{91.09$\pm$0.07}                                             & \textbf{91.30$\pm$0.13}                                              & \textbf{90.62$\pm$0.09}                                              & \textbf{90.41$\pm$0.07}                                          & \textbf{89.65$\pm$0.12}                                             & \textbf{89.34$\pm$0.11}                                          \\ \hline

    \rowcolor{gray!50}
             &               & \multicolumn{6}{c}{Amazon-Photo}                                                                                                                                                                                                                                                                                                                                                                                               \\ \hline
             Semantic  & Structural & \begin{tabular}[c]{@{}c@{}}non-overlapping\\ 5 clients\end{tabular} & \begin{tabular}[c]{@{}c@{}}non-overlapping\\ 10 clients\end{tabular} & \begin{tabular}[c]{@{}c@{}}non-overlapping\\ 20 clients\end{tabular} & \begin{tabular}[c]{@{}c@{}}overlapping\\ 10 clients\end{tabular} & \begin{tabular}[c]{@{}c@{}}overlapping\\ 30 clients\end{tabular} & \begin{tabular}[c]{@{}c@{}}overlapping\\ 50 clients\end{tabular} \\ \hline
    \textcolor{teal}{\CheckmarkBold}  & \textcolor{red}{\XSolidBrush}      & 90.65$\pm$0.09 ($\downarrow$ 3.21)                                  & 91.38$\pm$0.90 ($\downarrow$ 3.24)                                   & 90.60$\pm$0.15 ($\downarrow$ 3.16)                                   & 90.21$\pm$0.60 ($\downarrow$ 3.51)                               & 91.61$\pm$0.72 ($\downarrow$ 2.83)                                   & 90.78$\pm$0.77 ($\downarrow$ 2.58)                               \\ \hline
    \textcolor{red}{\XSolidBrush} & \textcolor{teal}{\CheckmarkBold}       & 90.37$\pm$0.45 ($\downarrow$ 3.49)                                 & 91.01$\pm$0.24 ($\downarrow$ 3.61)                                   & 90.24$\pm$0.44 ($\downarrow$ 3.52)                                   & 89.84$\pm$0.50 ($\downarrow$ 3.88)                               & 91.21$\pm$0.18 ($\downarrow$ 3.23)                                   & 90.31$\pm$0.09 ($\downarrow$ 3.05)                               \\ \hline
    \textcolor{red}{\XSolidBrush} & \textcolor{red}{\XSolidBrush}      & 89.63$\pm$0.29 ($\downarrow$ 4.23)                                 & 90.89$\pm$0.08 ($\downarrow$ 3.73)                                   & 89.85$\pm$0.43 ($\downarrow$ 3.91)                                   & 89.62$\pm$0.27 ($\downarrow$ 4.10)                               & 90.35$\pm$0.52 ($\downarrow$ 4.09)                                   & 89.89$\pm$0.38 ($\downarrow$ 3.47)                               \\ \hline
    \textcolor{teal}{\CheckmarkBold}  & \textcolor{teal}{\CheckmarkBold}       & \textbf{93.86$\pm$0.15}                                             & \textbf{94.62$\pm$0.14}                                              & \textbf{93.76$\pm$0.07}                                              & \textbf{93.72$\pm$0.15}                                          & \textbf{94.44$\pm$0.11}                                             & \textbf{93.36$\pm$0.14}                                          \\ \hline

    \rowcolor{gray!50}
    &               & \multicolumn{6}{c}{ogbn-arxiv}                                                                                                                                                                                                                                                                                                                                                                                               \\ \hline
    Semantic  & Structural & \begin{tabular}[c]{@{}c@{}}non-overlapping\\ 5 clients\end{tabular} & \begin{tabular}[c]{@{}c@{}}non-overlapping\\ 10 clients\end{tabular} & \begin{tabular}[c]{@{}c@{}}non-overlapping\\ 20 clients\end{tabular} & \begin{tabular}[c]{@{}c@{}}overlapping\\ 10 clients\end{tabular} & \begin{tabular}[c]{@{}c@{}}overlapping\\ 30 clients\end{tabular} & \begin{tabular}[c]{@{}c@{}}overlapping\\ 50 clients\end{tabular} \\ \hline
\textcolor{teal}{\CheckmarkBold}  & \textcolor{red}{\XSolidBrush}      & 67.29$\pm$0.50 ($\downarrow$ 3.57)                                  & 66.26$\pm$0.67 ($\downarrow$ 3.21)                                   & 65.53$\pm$0.06 ($\downarrow$ 3.24)                                   & 63.16$\pm$0.14 ($\downarrow$ 4.28)                               & 63.64$\pm$0.86 ($\downarrow$ 2.74)                                   & 62.86$\pm$0.34 ($\downarrow$ 2.87)                               \\ \hline
\textcolor{red}{\XSolidBrush} & \textcolor{teal}{\CheckmarkBold}       & 67.19$\pm$0.17 ($\downarrow$ 3.67)                                 & 66.13$\pm$0.50 ($\downarrow$ 3.34)                                   & 65.40$\pm$0.49 ($\downarrow$ 3.37)                                   & 62.99$\pm$0.38 ($\downarrow$ 4.45)                               & 63.07$\pm$0.83 ($\downarrow$ 3.31)                                   & 62.13$\pm$0.36 ($\downarrow$ 3.60)                               \\ \hline
\textcolor{red}{\XSolidBrush} & \textcolor{red}{\XSolidBrush}      & 66.86$\pm$0.59 ($\downarrow$ 4.00)                                 & 65.96$\pm$0.12 ($\downarrow$ 3.51)                                   & 64.79$\pm$0.08 ($\downarrow$ 3.98)                                   & 62.17$\pm$0.57 ($\downarrow$ 5.27)                               & 62.39$\pm$0.55 ($\downarrow$ 3.99)                                   & 61.65$\pm$0.08 ($\downarrow$ 4.08)                               \\ \hline
\textcolor{teal}{\CheckmarkBold}  & \textcolor{teal}{\CheckmarkBold}       & \textbf{70.86$\pm$0.13}                                             & \textbf{69.47$\pm$0.08}                                              & \textbf{68.77$\pm$0.13}                                              & \textbf{67.44$\pm$0.12}                                          & \textbf{66.38$\pm$0.13}                                             & \textbf{65.73$\pm$0.07}                                          \\ \hline
  \rowcolor{gray!50}
           &               & \multicolumn{6}{c}{Roman-empire}                                                                                                                                                                                                                                                                                                                                                                                               \\ \hline
           Semantic  & Structural & \begin{tabular}[c]{@{}c@{}}non-overlapping\\ 5 clients\end{tabular} & \begin{tabular}[c]{@{}c@{}}non-overlapping\\ 10 clients\end{tabular} & \begin{tabular}[c]{@{}c@{}}non-overlapping\\ 20 clients\end{tabular} & \begin{tabular}[c]{@{}c@{}}overlapping\\ 10 clients\end{tabular} & \begin{tabular}[c]{@{}c@{}}overlapping\\ 30 clients\end{tabular} & \begin{tabular}[c]{@{}c@{}}overlapping\\ 50 clients\end{tabular} \\ \hline
  \textcolor{teal}{\CheckmarkBold}  & \textcolor{red}{\XSolidBrush}      & 63.15$\pm$0.64 ($\downarrow$ 5.52)                                  & 62.67$\pm$0.14 ($\downarrow$ 4.14)                                   & 62.53$\pm$0.20 ($\downarrow$ 2.61)                                   & 59.85$\pm$0.54 ($\downarrow$ 5.81)                               & 60.28$\pm$0.14 ($\downarrow$ 3.52)                                   & 58.70$\pm$0.49 ($\downarrow$ 4.05)                               \\ \hline
  \textcolor{red}{\XSolidBrush} & \textcolor{teal}{\CheckmarkBold}       & 62.86$\pm$0.51 ($\downarrow$ 5.81)                                 & 62.46$\pm$0.13 ($\downarrow$ 4.35)                                   & 61.97$\pm$0.15 ($\downarrow$ 3.17)                                   & 60.48$\pm$0.77 ($\downarrow$ 5.18)                               & 59.60$\pm$0.17 ($\downarrow$ 4.20)                                   & 58.37$\pm$0.13 ($\downarrow$ 4.38)                               \\ \hline
  \textcolor{red}{\XSolidBrush} & \textcolor{red}{\XSolidBrush}      & 62.37$\pm$0.74 ($\downarrow$ 6.30)                                 & 61.76$\pm$0.42 ($\downarrow$ 5.05)                                   & 61.63$\pm$0.18 ($\downarrow$ 3.51)                                   & 59.22$\pm$0.84 ($\downarrow$ 6.44)                               & 59.16$\pm$0.47 ($\downarrow$ 4.64)                                   & 58.13$\pm$0.45 ($\downarrow$ 4.62)                               \\ \hline
  \textcolor{teal}{\CheckmarkBold}  & \textcolor{teal}{\CheckmarkBold}       & \textbf{68.67$\pm$0.10}                                             & \textbf{66.81$\pm$0.09}                                              & \textbf{65.14$\pm$0.15}                                              & \textbf{65.66$\pm$0.07}                                          & \textbf{63.80$\pm$0.09}                                              & \textbf{62.75$\pm$0.14}                                          \\ \hline
\rowcolor{gray!50}
&               & \multicolumn{6}{c}{Amazon-ratings}                                                                                                                                                                                                                                                                                                                                                                                               \\ \hline
Semantic  & Structural & \begin{tabular}[c]{@{}c@{}}non-overlapping\\ 5 clients\end{tabular} & \begin{tabular}[c]{@{}c@{}}non-overlapping\\ 10 clients\end{tabular} & \begin{tabular}[c]{@{}c@{}}non-overlapping\\ 20 clients\end{tabular} & \begin{tabular}[c]{@{}c@{}}overlapping\\ 10 clients\end{tabular} & \begin{tabular}[c]{@{}c@{}}overlapping\\ 30 clients\end{tabular} & \begin{tabular}[c]{@{}c@{}}overlapping\\ 50 clients\end{tabular} \\ \hline
\textcolor{teal}{\CheckmarkBold}  & \textcolor{red}{\XSolidBrush}      & 41.78$\pm$0.12 ($\downarrow$ 3.40)                                  & 41.29$\pm$0.20 ($\downarrow$ 3.82)                                   & 42.57$\pm$0.16 ($\downarrow$ 3.56)                                   & 39.46$\pm$0.17 ($\downarrow$ 3.37)                               & 39.19$\pm$0.28 ($\downarrow$ 3.33)                                   & 40.67$\pm$0.40 ($\downarrow$ 2.30)                               \\ \hline
\textcolor{red}{\XSolidBrush} & \textcolor{teal}{\CheckmarkBold}       & 40.94$\pm$0.33 ($\downarrow$ 4.24)                                 & 41.23$\pm$0.37 ($\downarrow$ 3.88)                                   & 42.28$\pm$0.27 ($\downarrow$ 3.85)                                   & 39.40$\pm$0.24 ($\downarrow$ 3.43)                               & 38.89$\pm$0.25 ($\downarrow$ 3.63)                                   & 40.51$\pm$0.15 ($\downarrow$ 2.46)                               \\ \hline
\textcolor{red}{\XSolidBrush} & \textcolor{red}{\XSolidBrush}      & 40.52$\pm$0.22 ($\downarrow$ 4.66)                                 & 41.05$\pm$0.21 ($\downarrow$ 4.06)                                   & 42.05$\pm$0.13 ($\downarrow$ 4.08)                                   & 38.62$\pm$0.13 ($\downarrow$ 4.21)                               & 38.78$\pm$0.19 ($\downarrow$ 3.74)                                   & 40.03$\pm$0.70 ($\downarrow$ 2.94)                               \\ \hline
\textcolor{teal}{\CheckmarkBold}  & \textcolor{teal}{\CheckmarkBold}       & \textbf{45.18$\pm$0.14}                                             & \textbf{45.11$\pm$0.15}                                              & \textbf{46.13$\pm$0.05}                                              & \textbf{42.83$\pm$0.06}                                          & \textbf{42.52$\pm$0.10}                                             & \textbf{42.97$\pm$0.15}                                          \\ \hline

\rowcolor{gray!50}
&               & \multicolumn{6}{c}{Minesweeper}                                                                                                                                                                                                                                                                                                                                                                                               \\ \hline
Semantic  & Structural & \begin{tabular}[c]{@{}c@{}}non-overlapping\\ 5 clients\end{tabular} & \begin{tabular}[c]{@{}c@{}}non-overlapping\\ 10 clients\end{tabular} & \begin{tabular}[c]{@{}c@{}}non-overlapping\\ 20 clients\end{tabular} & \begin{tabular}[c]{@{}c@{}}overlapping\\ 10 clients\end{tabular} & \begin{tabular}[c]{@{}c@{}}overlapping\\ 30 clients\end{tabular} & \begin{tabular}[c]{@{}c@{}}overlapping\\ 50 clients\end{tabular} \\ \hline
\textcolor{teal}{\CheckmarkBold}  & \textcolor{red}{\XSolidBrush}      & 80.90$\pm$0.71 ($\downarrow$ 1.36)                                  & 80.13$\pm$0.07 ($\downarrow$ 2.03)                                   & 79.87$\pm$0.12 ($\downarrow$ 2.73)                                   & 72.62$\pm$0.13 ($\downarrow$ 3.03)                               & 70.46$\pm$0.30 ($\downarrow$ 2.14)                                   & 68.60$\pm$0.44 ($\downarrow$ 2.71)                               \\ \hline
\textcolor{red}{\XSolidBrush} & \textcolor{teal}{\CheckmarkBold}       & 80.72$\pm$0.33 ($\downarrow$ 1.54)                                 & 79.49$\pm$0.39 ($\downarrow$ 2.67)                                   & 79.78$\pm$0.37 ($\downarrow$ 2.82)                                   & 72.15$\pm$0.52 ($\downarrow$ 3.50)                               & 70.09$\pm$0.38 ($\downarrow$ 2.51)                                   & 68.64$\pm$0.10 ($\downarrow$ 2.67)                               \\ \hline
\textcolor{red}{\XSolidBrush} & \textcolor{red}{\XSolidBrush}      & 79.56$\pm$0.18 ($\downarrow$ 2.70)                                 & 78.84$\pm$0.44 ($\downarrow$ 3.32)                                   & 79.27$\pm$0.46 ($\downarrow$ 3.33)                                   & 71.26$\pm$0.38 ($\downarrow$ 4.39)                               & 69.37$\pm$0.26 ($\downarrow$ 3.23)                                   & 67.31$\pm$0.39 ($\downarrow$ 4.00)                               \\ \hline
\textcolor{teal}{\CheckmarkBold}  & \textcolor{teal}{\CheckmarkBold}       & \textbf{82.26$\pm$0.14}                                             & \textbf{82.16$\pm$0.08}                                              & \textbf{82.60$\pm$0.11}                                              & \textbf{75.65$\pm$0.11}                                          & \textbf{72.60$\pm$0.06}                                             & \textbf{71.31$\pm$0.05}                                          \\ \hline

\rowcolor{gray!50}
&               & \multicolumn{6}{c}{Tolokers}                                                                                                                                                                                                                                                                                                                                                                                               \\ \hline
Semantic  & Structural & \begin{tabular}[c]{@{}c@{}}non-overlapping\\ 5 clients\end{tabular} & \begin{tabular}[c]{@{}c@{}}non-overlapping\\ 10 clients\end{tabular} & \begin{tabular}[c]{@{}c@{}}non-overlapping\\ 20 clients\end{tabular} & \begin{tabular}[c]{@{}c@{}}overlapping\\ 10 clients\end{tabular} & \begin{tabular}[c]{@{}c@{}}overlapping\\ 30 clients\end{tabular} & \begin{tabular}[c]{@{}c@{}}overlapping\\ 50 clients\end{tabular} \\ \hline
\textcolor{teal}{\CheckmarkBold}  & \textcolor{red}{\XSolidBrush}      & 71.88$\pm$0.42 ($\downarrow$ 3.94)                                  & 70.33$\pm$0.58 ($\downarrow$ 3.63)                                   & 66.27$\pm$0.28 ($\downarrow$ 4.61)                                   & 71.71$\pm$0.31 ($\downarrow$ 2.62)                               & 68.36$\pm$0.42 ($\downarrow$ 3.91)                                   & 68.96$\pm$0.40 ($\downarrow$ 2.07)                               \\ \hline
\textcolor{red}{\XSolidBrush} & \textcolor{teal}{\CheckmarkBold}       & 71.37$\pm$0.14 ($\downarrow$ 4.45)                                 & 70.08$\pm$0.46 ($\downarrow$ 3.88)                                   & 66.14$\pm$0.42 ($\downarrow$ 4.74)                                   & 71.47$\pm$0.22 ($\downarrow$ 2.86)                               & 68.16$\pm$0.64 ($\downarrow$ 4.11)                                   & 68.45$\pm$0.31 ($\downarrow$ 2.58)                               \\ \hline
\textcolor{red}{\XSolidBrush} & \textcolor{red}{\XSolidBrush}      & 70.11$\pm$0.26 ($\downarrow$ 5.71)                                 & 69.61$\pm$0.32 ($\downarrow$ 4.35)                                   & 65.59$\pm$0.37 ($\downarrow$ 5.29)                                   & 70.59$\pm$0.19 ($\downarrow$ 3.74)                               & 67.09$\pm$0.46 ($\downarrow$ 5.18)                                   & 68.14$\pm$0.61 ($\downarrow$ 2.89)                               \\ \hline
\textcolor{teal}{\CheckmarkBold}  & \textcolor{teal}{\CheckmarkBold}       & \textbf{75.82$\pm$0.05}                       & \textbf{73.96$\pm$0.10}                                              & \textbf{70.88$\pm$0.12}                                              & \textbf{74.33$\pm$0.09}                                          & \textbf{72.27$\pm$0.07}                                             & \textbf{71.03$\pm$0.09}                                          \\ \hline

\rowcolor{gray!50}
&               & \multicolumn{6}{c}{Questions}                                                                                                                                                                                                                                                                                                                                                                                               \\ \hline
Semantic  & Structural & \begin{tabular}[c]{@{}c@{}}non-overlapping\\ 5 clients\end{tabular} & \begin{tabular}[c]{@{}c@{}}non-overlapping\\ 10 clients\end{tabular} & \begin{tabular}[c]{@{}c@{}}non-overlapping\\ 20 clients\end{tabular} & \begin{tabular}[c]{@{}c@{}}overlapping\\ 10 clients\end{tabular} & \begin{tabular}[c]{@{}c@{}}overlapping\\ 30 clients\end{tabular} & \begin{tabular}[c]{@{}c@{}}overlapping\\ 50 clients\end{tabular} \\ \hline
\textcolor{teal}{\CheckmarkBold}  & \textcolor{red}{\XSolidBrush}      & 66.94$\pm$0.23 ($\downarrow$ 2.57)                                  & 65.44$\pm$0.63 ($\downarrow$ 3.25)                                   & 62.33$\pm$0.59 ($\downarrow$ 3.41)                                   & 67.82$\pm$0.47 ($\downarrow$ 1.57)                               & 64.55$\pm$0.12 ($\downarrow$ 1.88)                                   & 61.63$\pm$0.19 ($\downarrow$ 3.31)                               \\ \hline
\textcolor{red}{\XSolidBrush} & \textcolor{teal}{\CheckmarkBold}       & 66.72$\pm$0.45 ($\downarrow$ 2.79)                                 & 65.37$\pm$0.26 ($\downarrow$ 3.32)                                   & 61.51$\pm$0.32 ($\downarrow$ 4.23)                                   & 67.53$\pm$0.26 ($\downarrow$ 1.86)                               & 64.29$\pm$0.35 ($\downarrow$ 2.14)                                   & 61.48$\pm$0.17 ($\downarrow$ 3.46)                               \\ \hline
\textcolor{red}{\XSolidBrush} & \textcolor{red}{\XSolidBrush}      & 65.70$\pm$0.18 ($\downarrow$ 3.81)                                 & 64.89$\pm$0.44 ($\downarrow$ 3.80)                                   & 61.23$\pm$0.22 ($\downarrow$ 4.51)                                   & 66.33$\pm$0.53 ($\downarrow$ 3.06)                               & 63.62$\pm$0.31 ($\downarrow$ 2.81)                                   & 60.49$\pm$0.37 ($\downarrow$ 4.45)                               \\ \hline
\textcolor{teal}{\CheckmarkBold}  & \textcolor{teal}{\CheckmarkBold}       & \textbf{69.51$\pm$0.15}                       & \textbf{68.69$\pm$0.18}                                              & \textbf{65.74$\pm$0.07}                                              & \textbf{69.39$\pm$0.05}                                          & \textbf{66.43$\pm$0.13}                                             & \textbf{64.94$\pm$0.12}                                          \\ \hline
    \end{tabular}
    }
  \end{table*}

\subsection{Additional Convergence Curves}\label{ap_convergence_curves}
To further evaluate the convergence of our proposed FedSSA and the compared methods, we present convergence curves on six datasets under non-overlapping partitioning setting (see~\cref{fig_additional_convergence1}), and on eight datasets under overlapping partitioning setting (see~\cref{fig_additional_convergence2}). We can observe that our proposed FedSSA converges stably. In contrast, the convergence curves of typical methods such as FedGTA (\textit{e.g.}, Fig.~\ref{fig_additional_convergence2}(d)) exhibit pronounced instability. This instability arises because FedGTA relies on dynamically estimated similarity levels to guide federated aggregation. However, under strong client heterogeneity, local models can undergo substantial changes between communication rounds, leading to rapidly fluctuating similarity levels. These fluctuations, in turn, cause an inconsistent aggregation process and hinder stable knowledge transfer across clients. Consequently, the convergence process becomes more erratic, resulting in the observed oscillations in the convergence curves.

\begin{figure*}[!t]
  \centering
  \subfloat[\footnotesize{\textit{Cora}}]{\includegraphics[width=0.33\columnwidth]{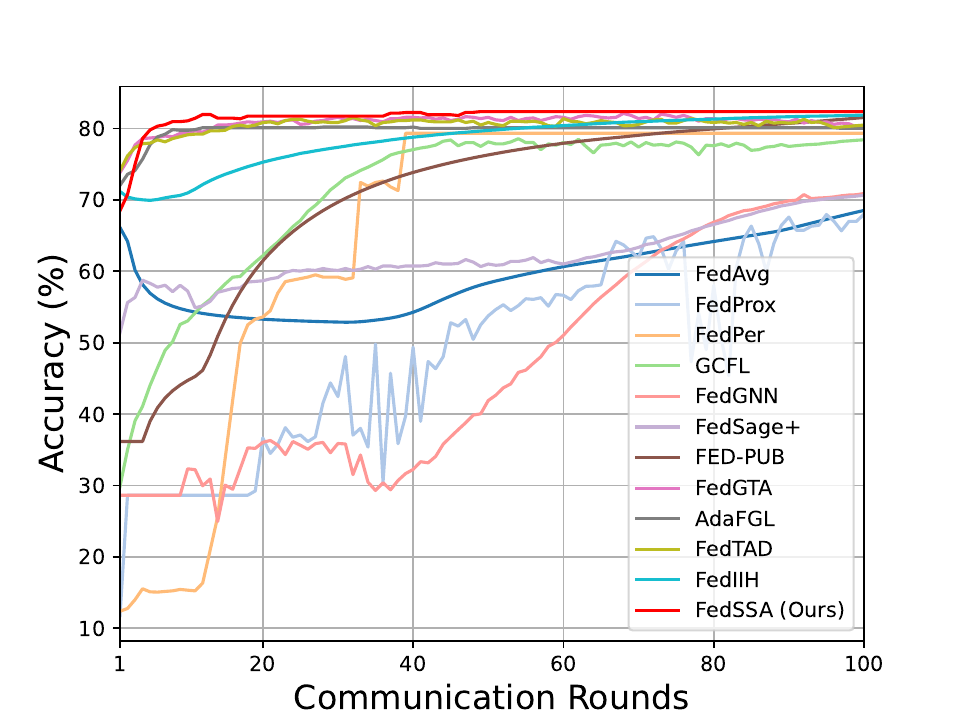}\label{fig_additional_convergence1_1}}
  \hfill
  \subfloat[\footnotesize{\textit{CiteSeer}}]{\includegraphics[width=0.33\columnwidth]{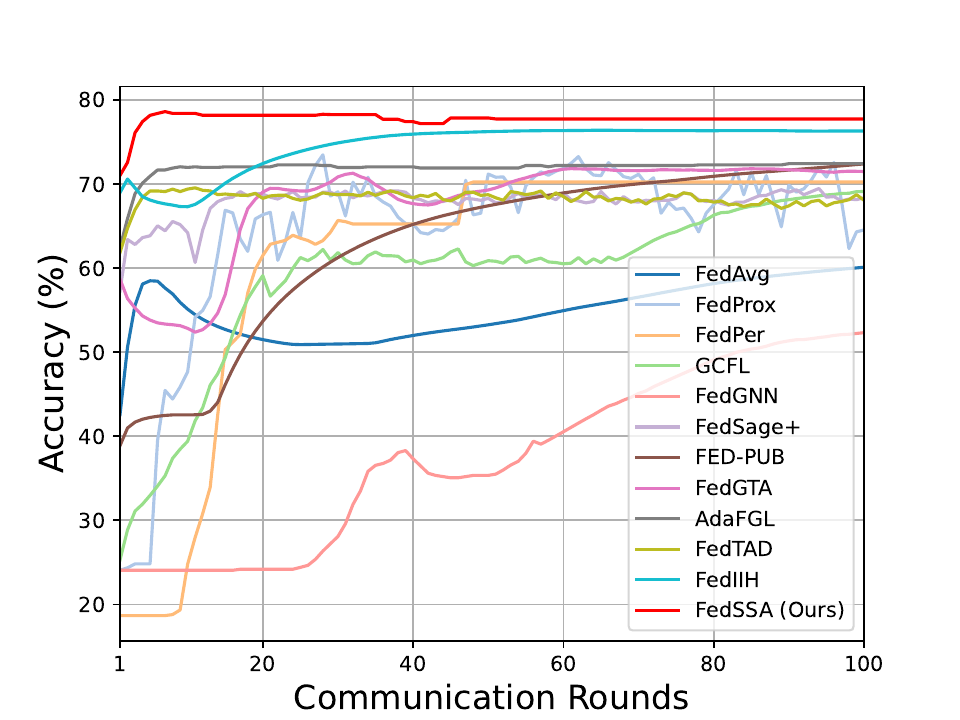}\label{fig_additional_convergence1_2}}
  \hfill
  \subfloat[\footnotesize{\textit{PubMed}}]{\includegraphics[width=0.33\columnwidth]{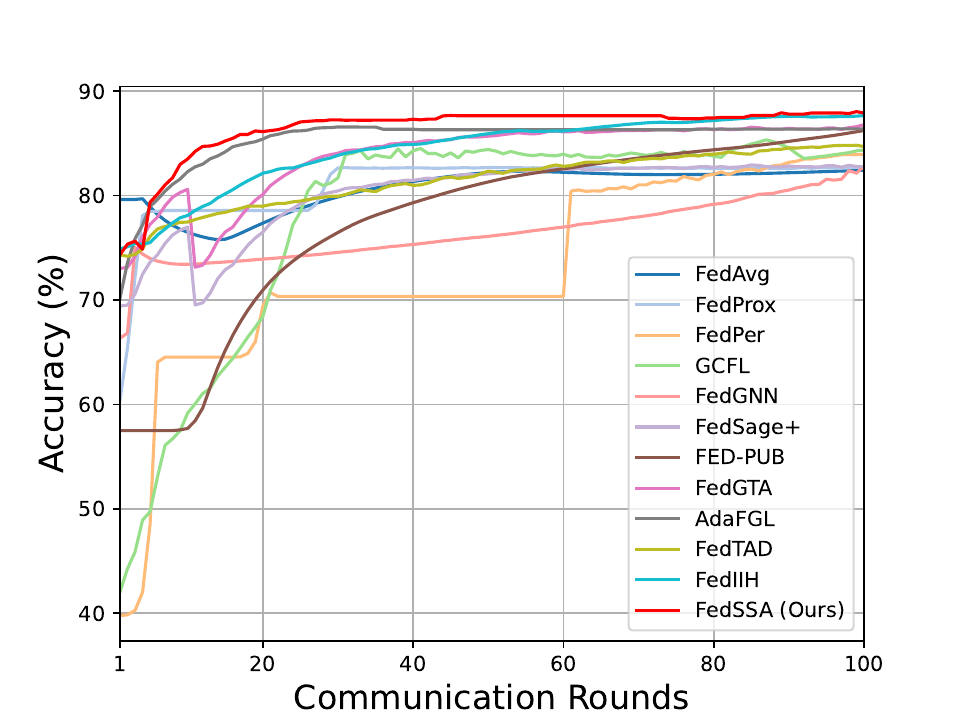}\label{fig_additional_convergence1_3}}
  \hfill
  \subfloat[\footnotesize{\textit{Roman-empire}}]{\includegraphics[width=0.33\columnwidth]{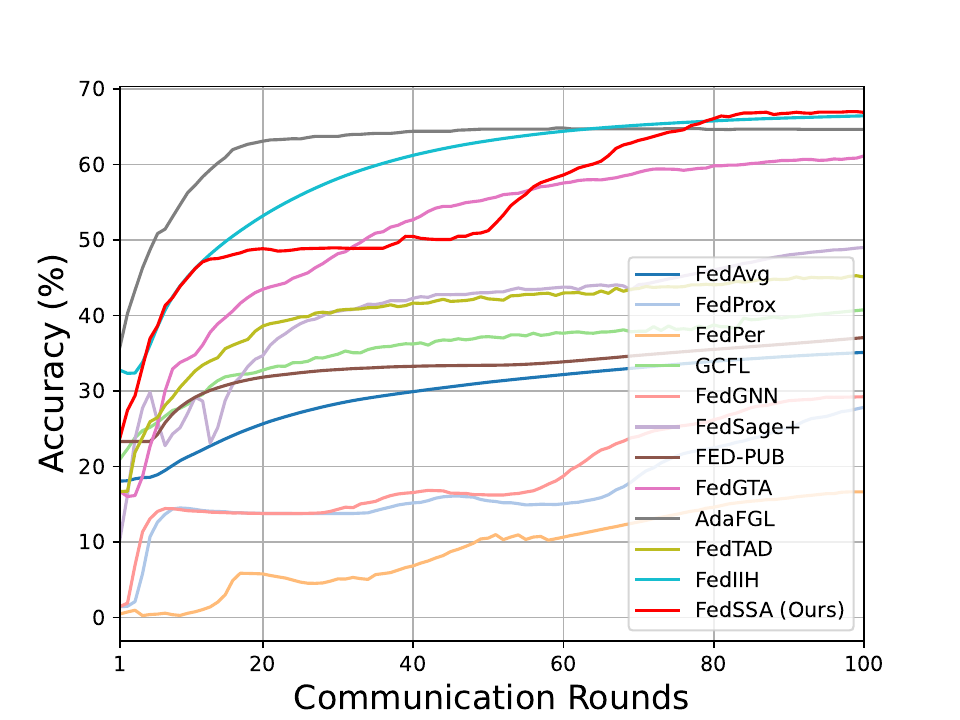}\label{fig_additional_convergence1_4}}
  \hfill
  \subfloat[\footnotesize{\textit{Tolokers}}]{\includegraphics[width=0.33\columnwidth]{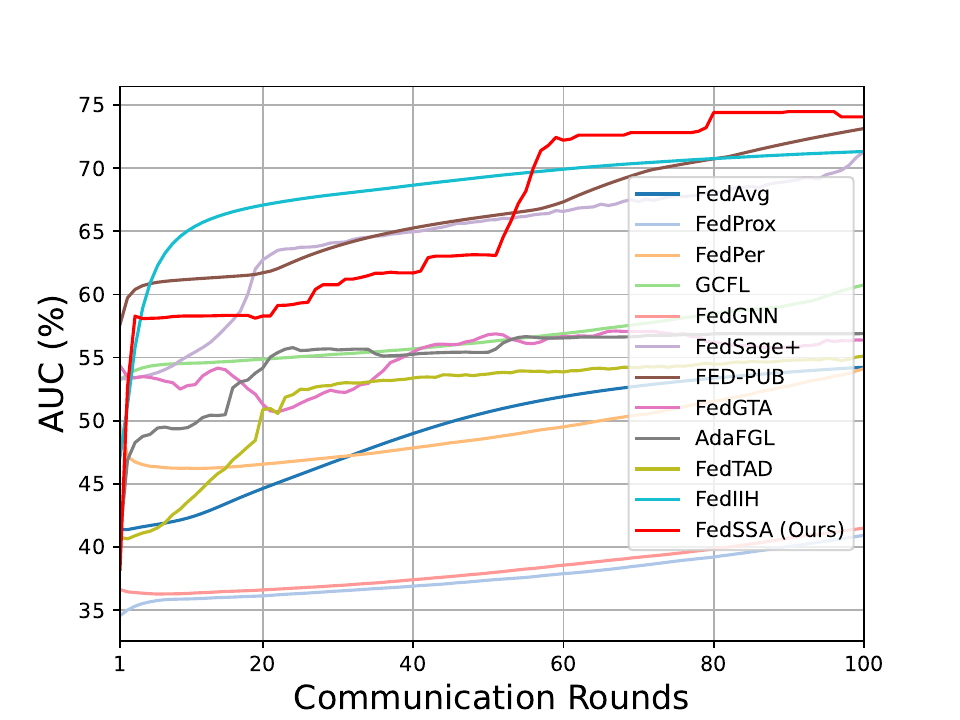}\label{fig_additional_convergence1_5}}
  \hfill
  \subfloat[\footnotesize{\textit{Questions}}]{\includegraphics[width=0.33\columnwidth]{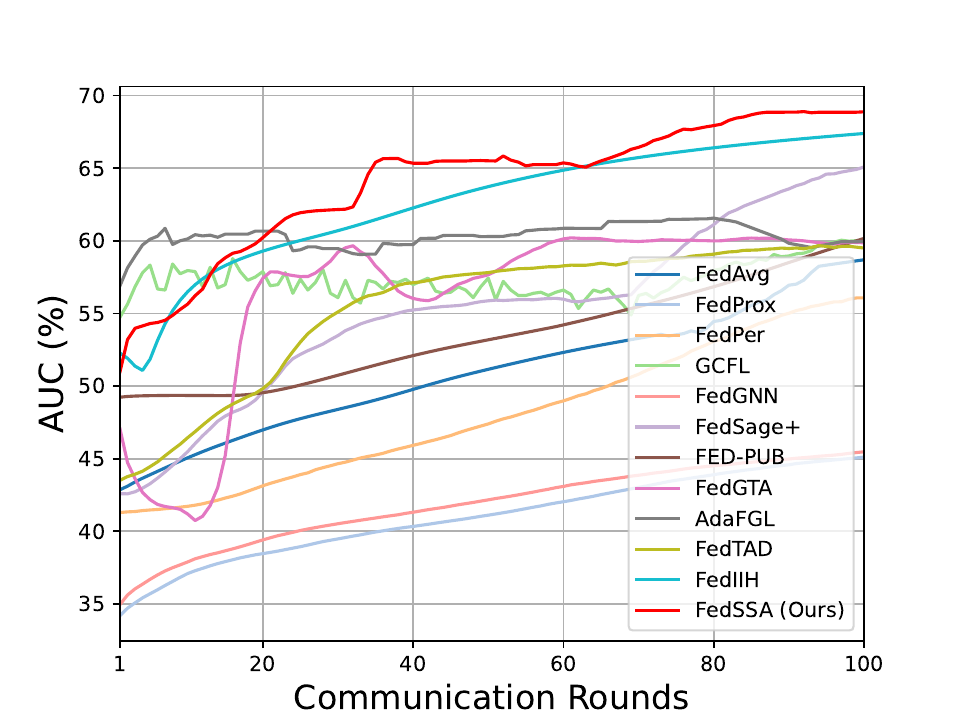}\label{fig_additional_convergence1_6}}
  \caption{Convergence curves on six datasets under non-overlapping partitioning setting with 10 clients.}
  \label{fig_additional_convergence1}
\end{figure*}

\begin{figure*}[!t]
  \centering
  \subfloat[\footnotesize{\textit{Cora}}]{\includegraphics[width=0.25\columnwidth]{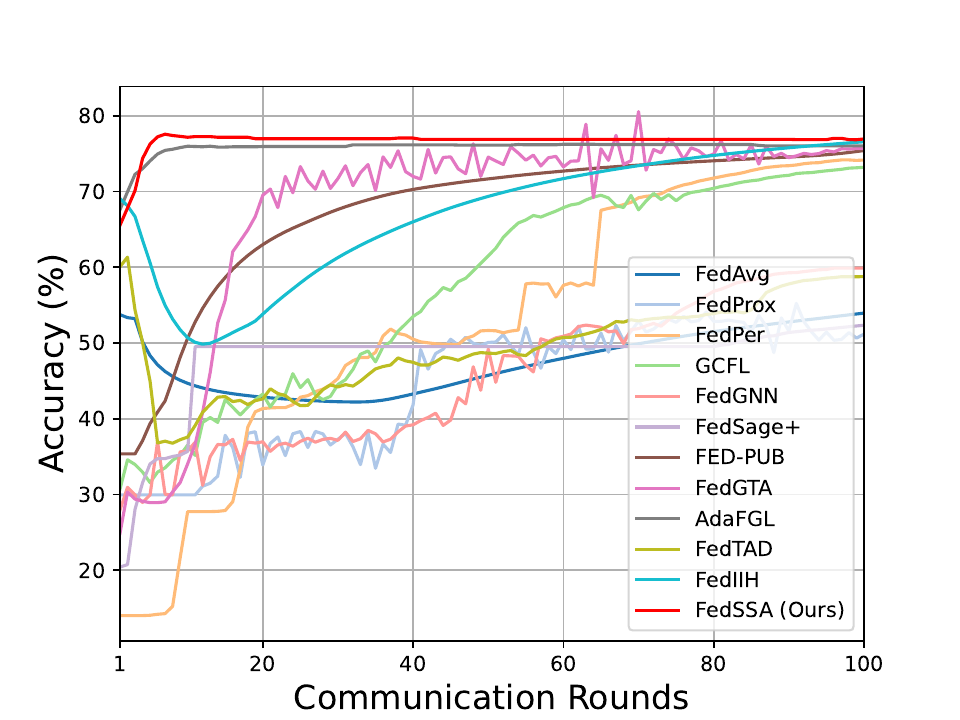}\label{fig_additional_convergence2_1}}
  \hfill
  \subfloat[\footnotesize{\textit{CiteSeer}}]{\includegraphics[width=0.25\columnwidth]{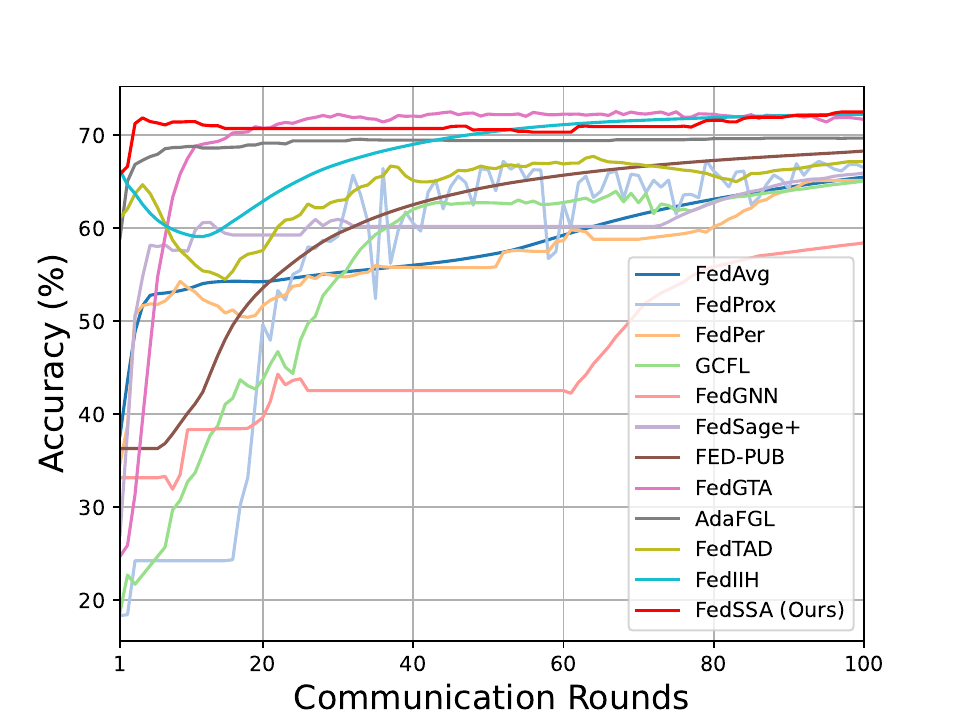}\label{fig_additional_convergence2_2}}
  \hfill
  \subfloat[\footnotesize{\textit{PubMed}}]{\includegraphics[width=0.25\columnwidth]{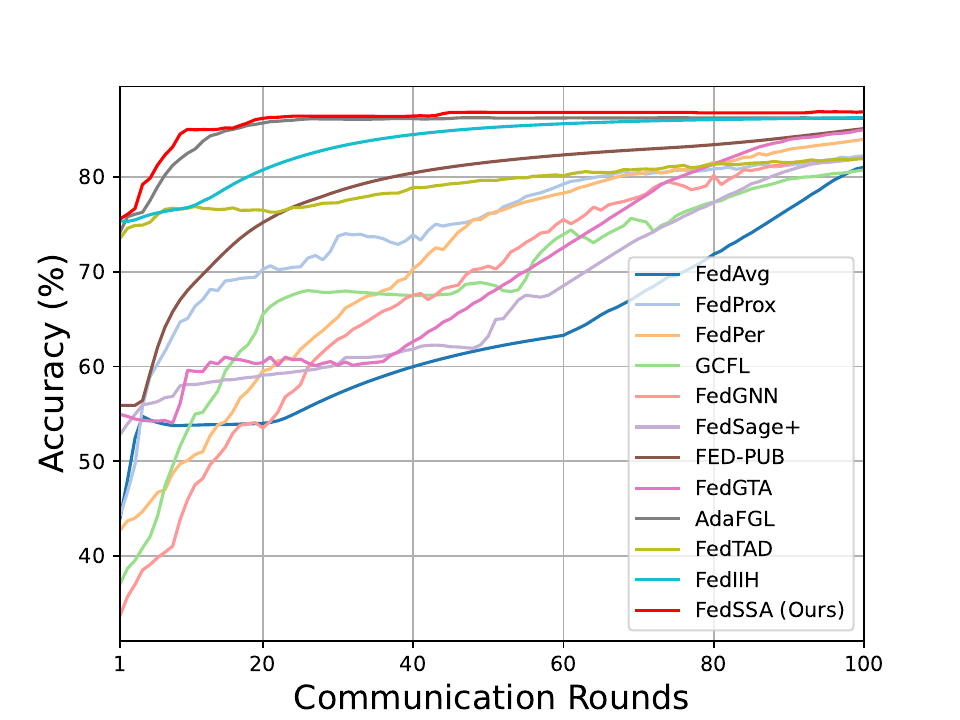}\label{fig_additional_convergence2_3}}
  \hfill
  \subfloat[\footnotesize{\textit{ogbn-arxiv}}]{\includegraphics[width=0.25\columnwidth]{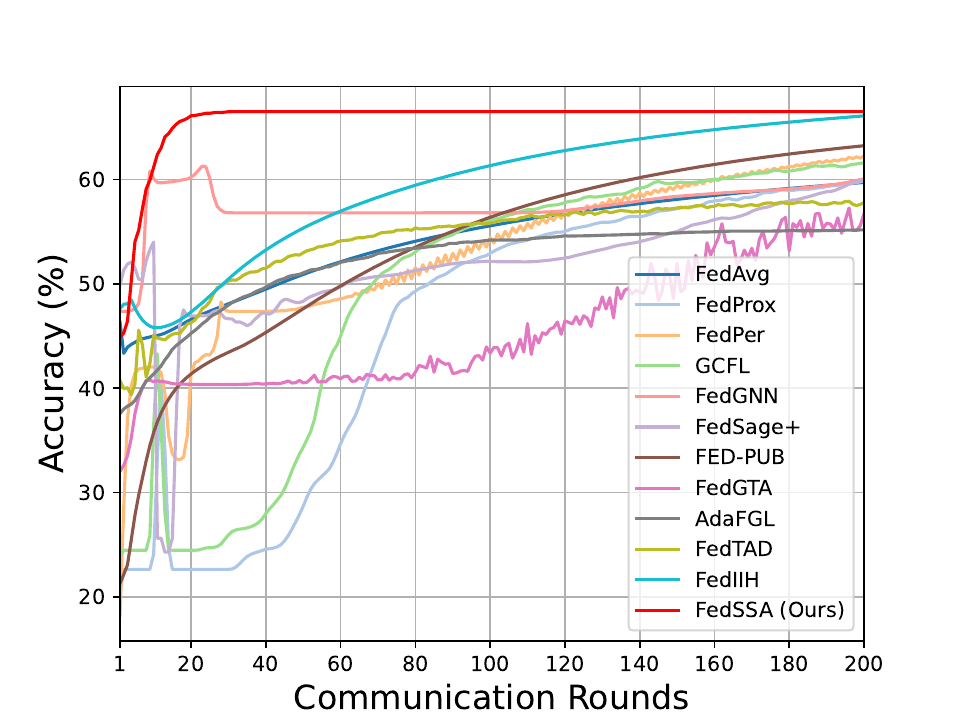}\label{fig_additional_convergence2_4}}
  \hfill
  \subfloat[\footnotesize{\textit{Roman-empire}}]{\includegraphics[width=0.25\columnwidth]{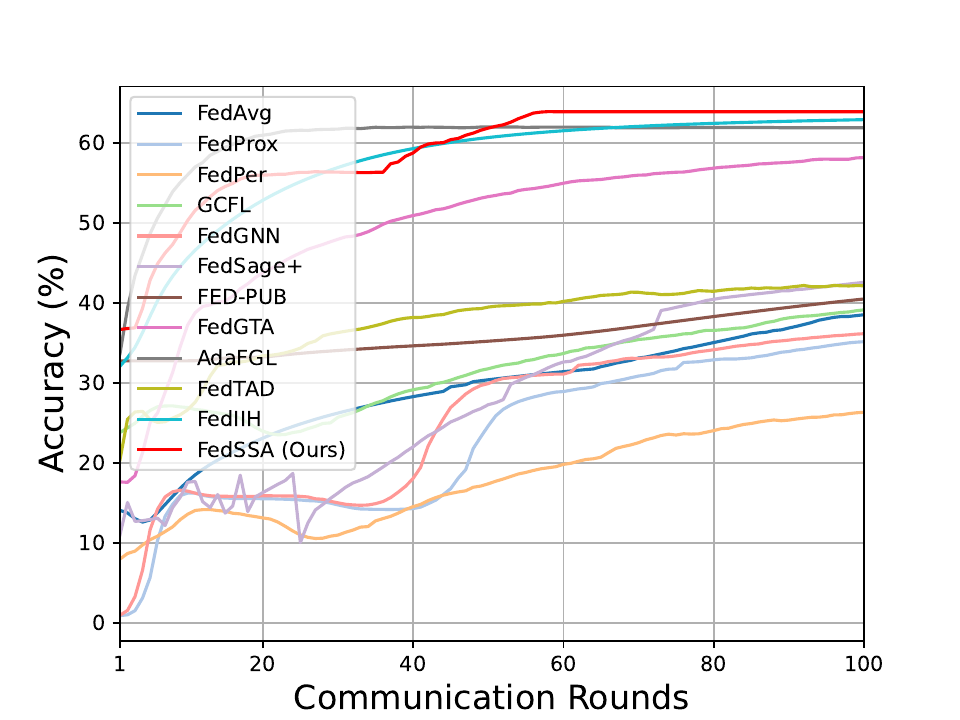}\label{fig_additional_convergence2_5}}
  \hfill
  \subfloat[\footnotesize{\textit{Minesweeper}}]{\includegraphics[width=0.25\columnwidth]{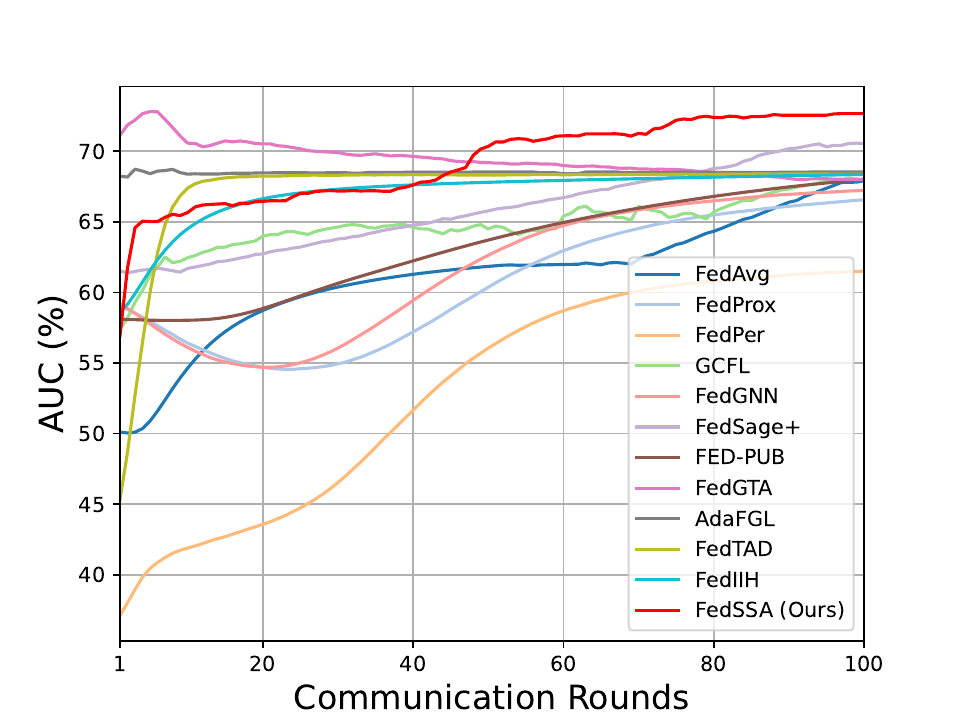}\label{fig_additional_convergence2_6}}
  \hfill
  \subfloat[\footnotesize{\textit{Tolokers}}]{\includegraphics[width=0.25\columnwidth]{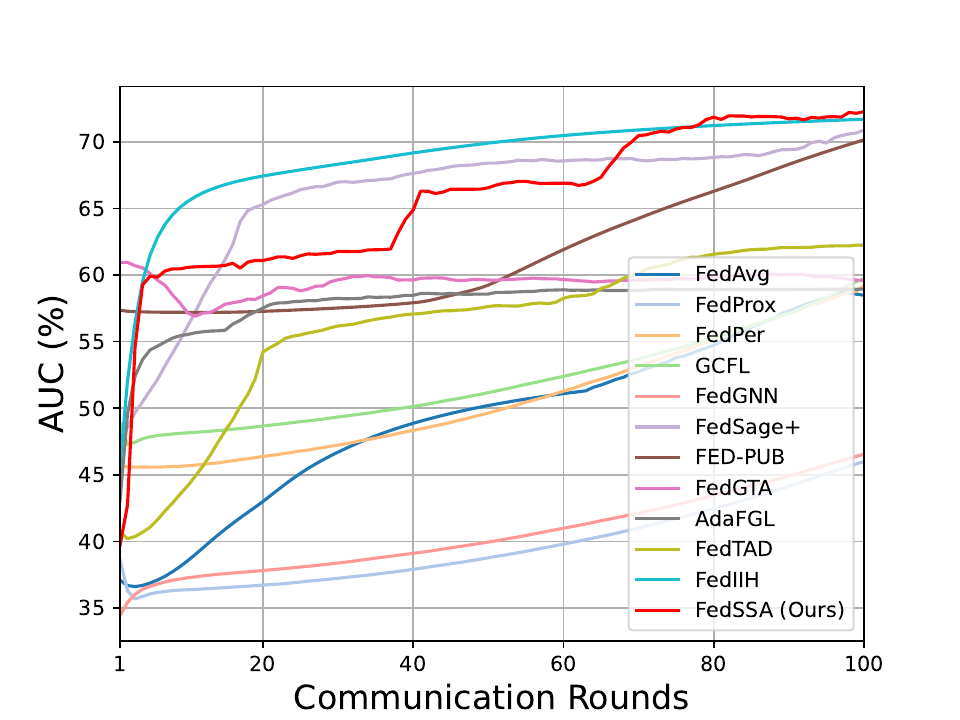}\label{fig_additional_convergence2_7}}
  \hfill
  \subfloat[\footnotesize{\textit{Questions}}]{\includegraphics[width=0.25\columnwidth]{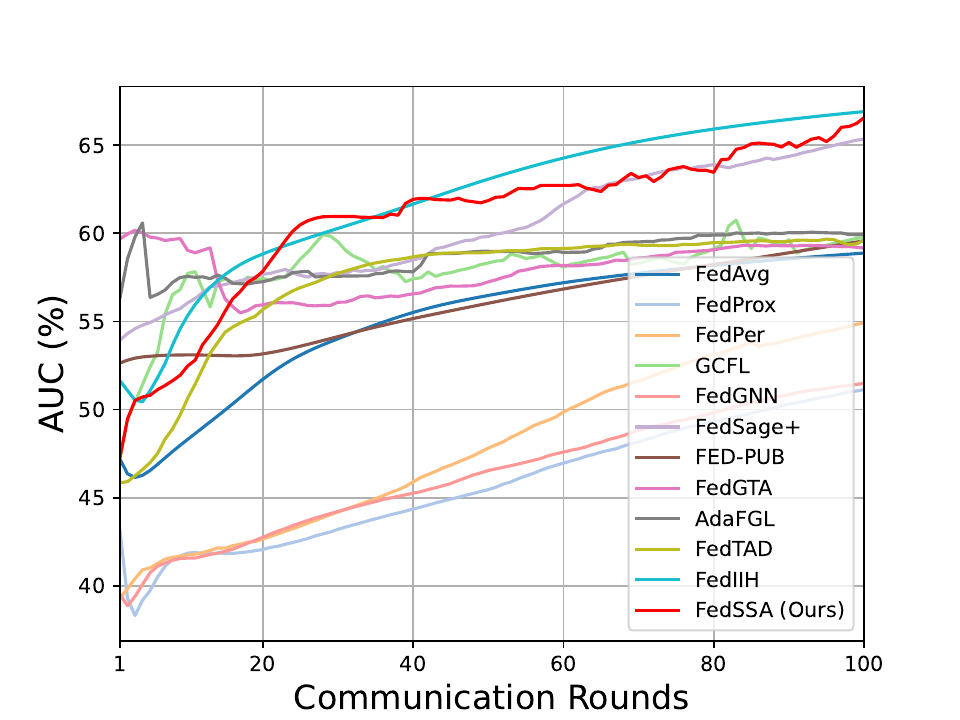}\label{fig_additional_convergence2_8}}
  \caption{Convergence curves on eight datasets under overlapping partitioning settings with 30 clients.}
  \label{fig_additional_convergence2}
\end{figure*}

\subsection{Additional Sensitivity Analysis on Hyperparameters}\label{ap_sensitivity_analysis}
To further analyze the sensitivity of our proposed FedSSA to hyperparameters, we conduct additional sensitivity analyses on \textit{Cora} and \textit{Roman-empire} datasets. Specifically, we examine the impact of four key hyperparameters, namely the number of node clusters $K_\mathrm{node}$, the number of structural clusters $K_\mathrm{struct}$, and regularization parameters $\lambda_1$ and $\lambda_2$. Figure~\ref{fig_additional_sensitivity1} and Figure~\ref{fig_additional_sensitivity2} present accuracy curves with variance bars under different values of hyperparameters. Experimental results show that our proposed FedSSA exhibits stable performance across a wide range of hyperparameters, which indicates that our FedSSA is not sensitive to the variation of hyperparameters.

\begin{figure}[]
  \centering
  \subfloat[\footnotesize{\textit{Cora} under different $K_\mathrm{node}$}]{\includegraphics[width=0.2\columnwidth]{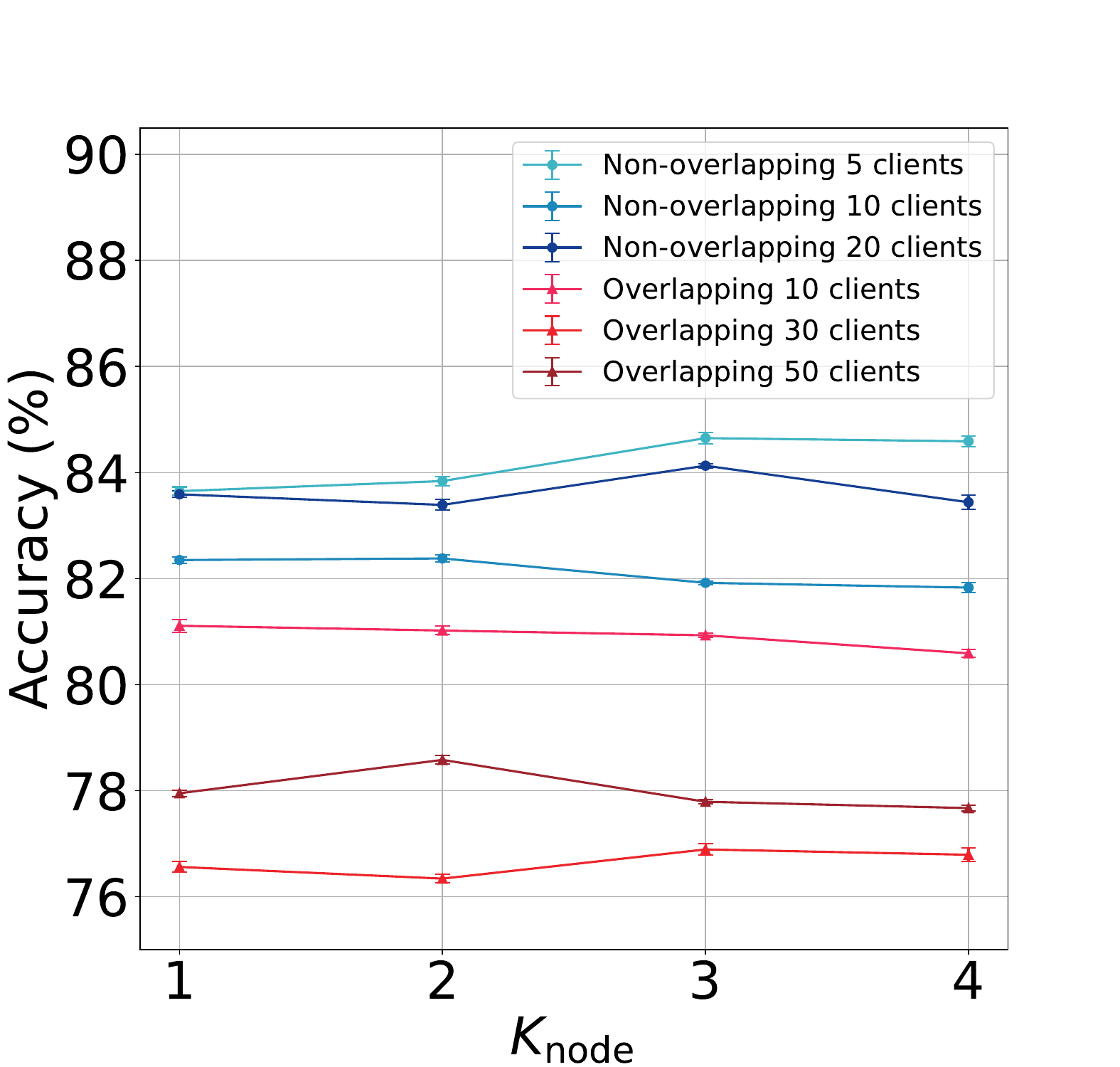}\label{fig_additional_sensitivity1_1}}
  \hfill
  \subfloat[\footnotesize{\textit{Cora} under different $K_\mathrm{struct}$}]{\includegraphics[width=0.2\columnwidth]{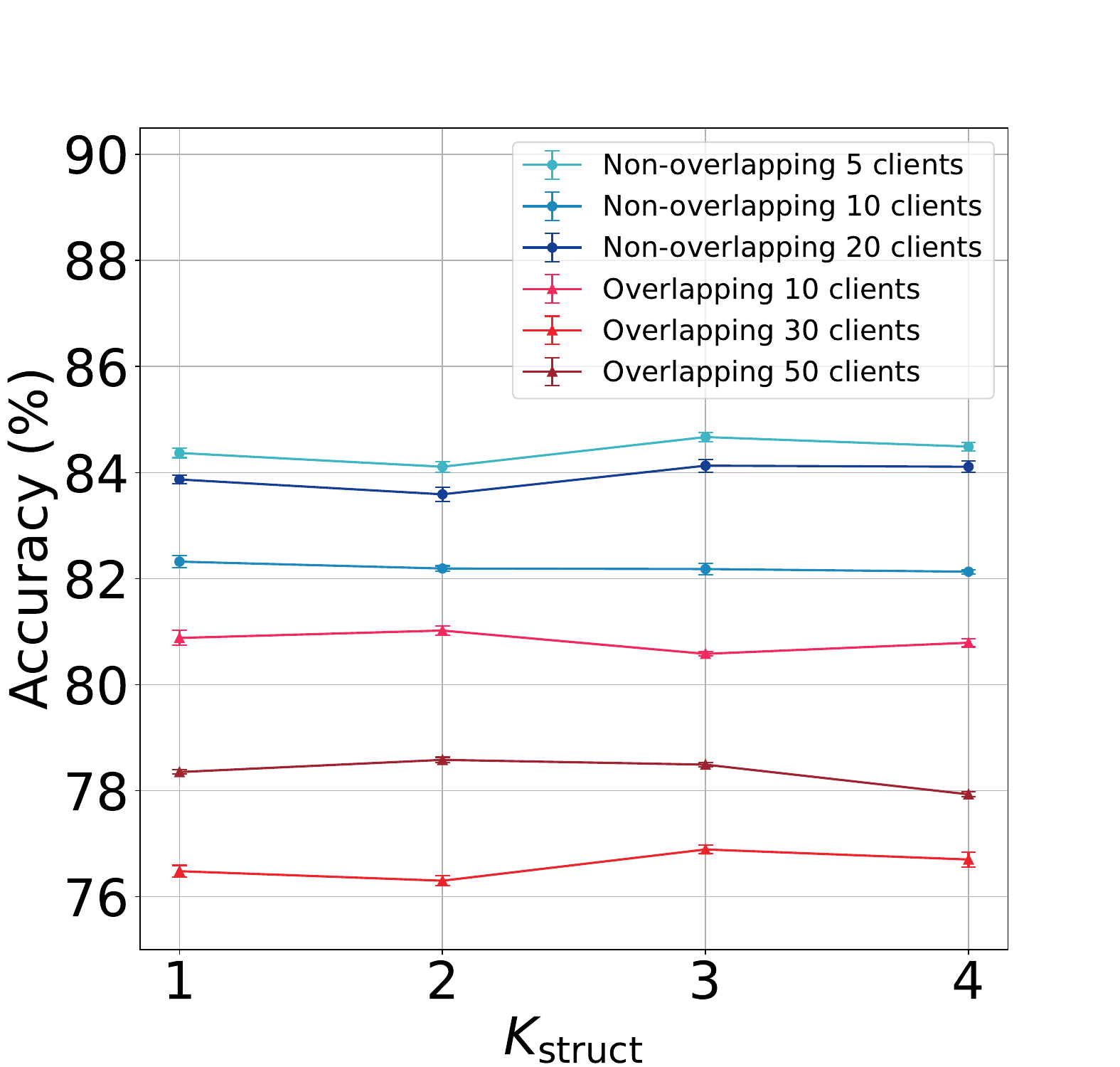}\label{fig_additional_sensitivity1_2}}
  \hfill
  \subfloat[\footnotesize{\textit{Cora} under different $\lambda_1$}]{\includegraphics[width=0.2\columnwidth]{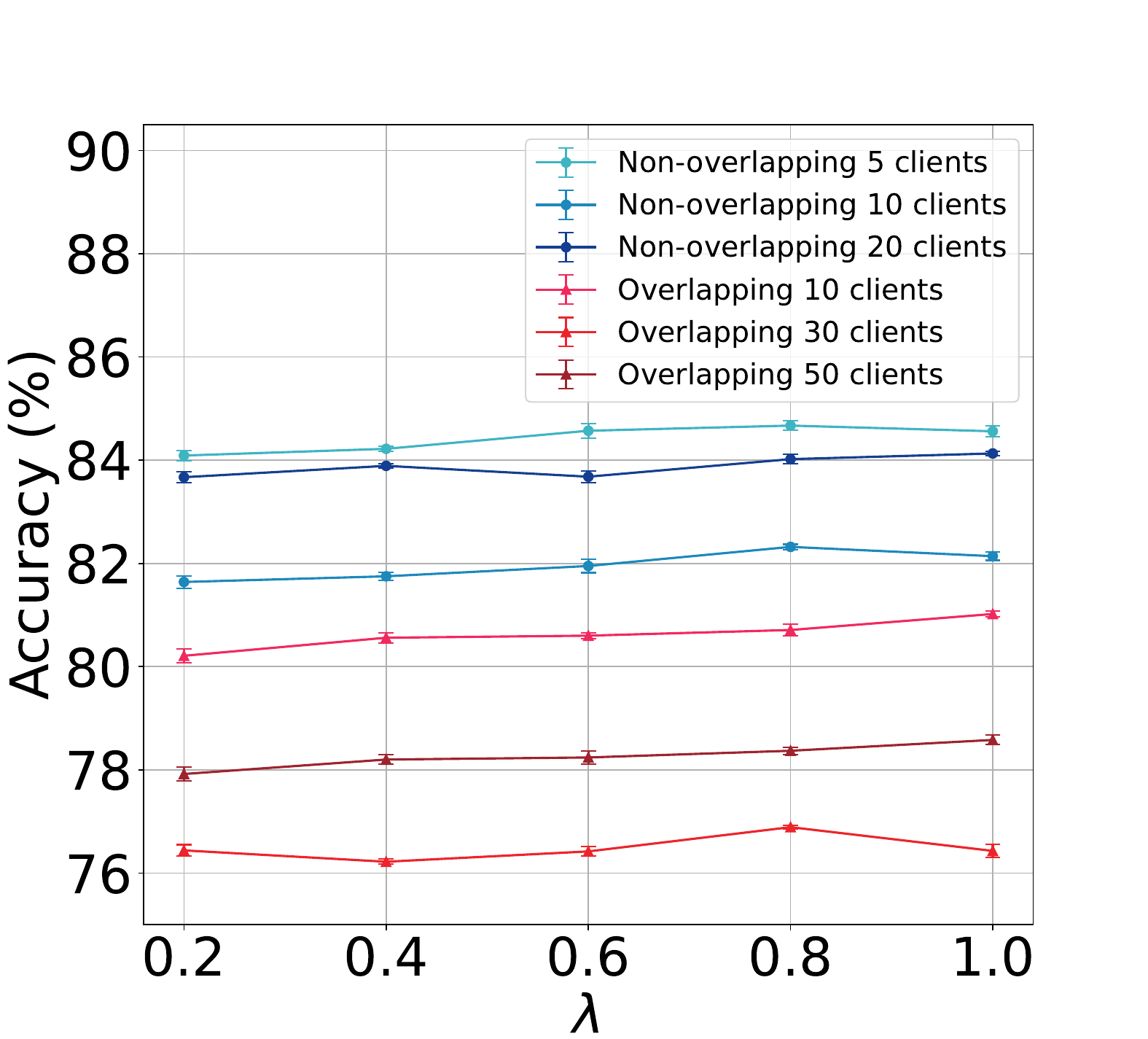}\label{fig_additional_sensitivity1_3}}
  \hfill
  \subfloat[\footnotesize{\textit{Cora} under different $\lambda_2$}]{\includegraphics[width=0.2\columnwidth]{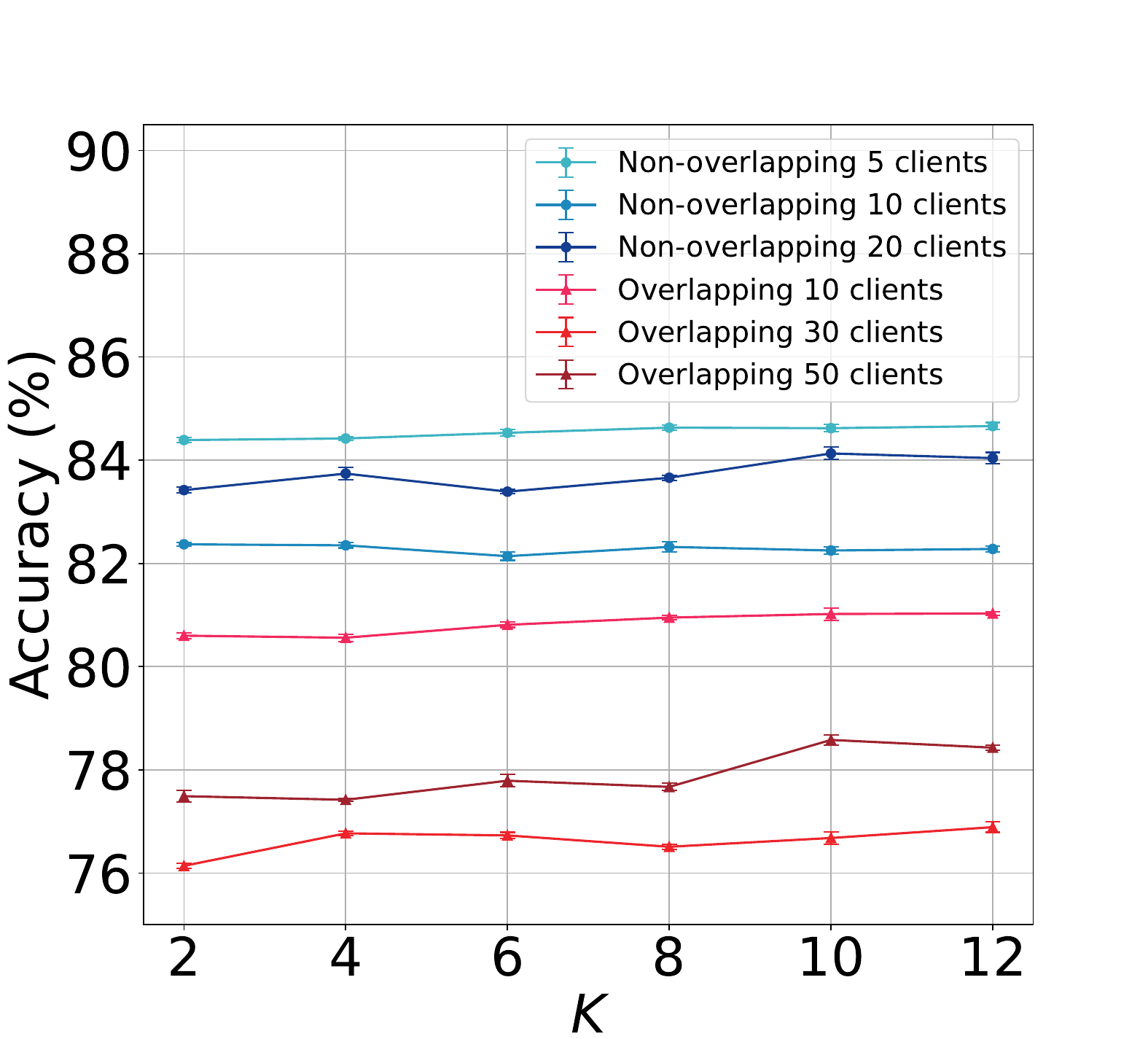}\label{fig_additional_sensitivity1_4}}
  \caption{Accuracy curves with variance bars on \textit{Cora} dataset under different values of $K_\mathrm{node}$, $K_\mathrm{struct}$, $\lambda_1$, and $\lambda_2$.}
  \label{fig_additional_sensitivity1}
\end{figure}

\begin{figure}[]
  \centering
  \subfloat[\footnotesize{\textit{Roman-empire} under different $K_\mathrm{node}$}]{\includegraphics[width=0.2\columnwidth]{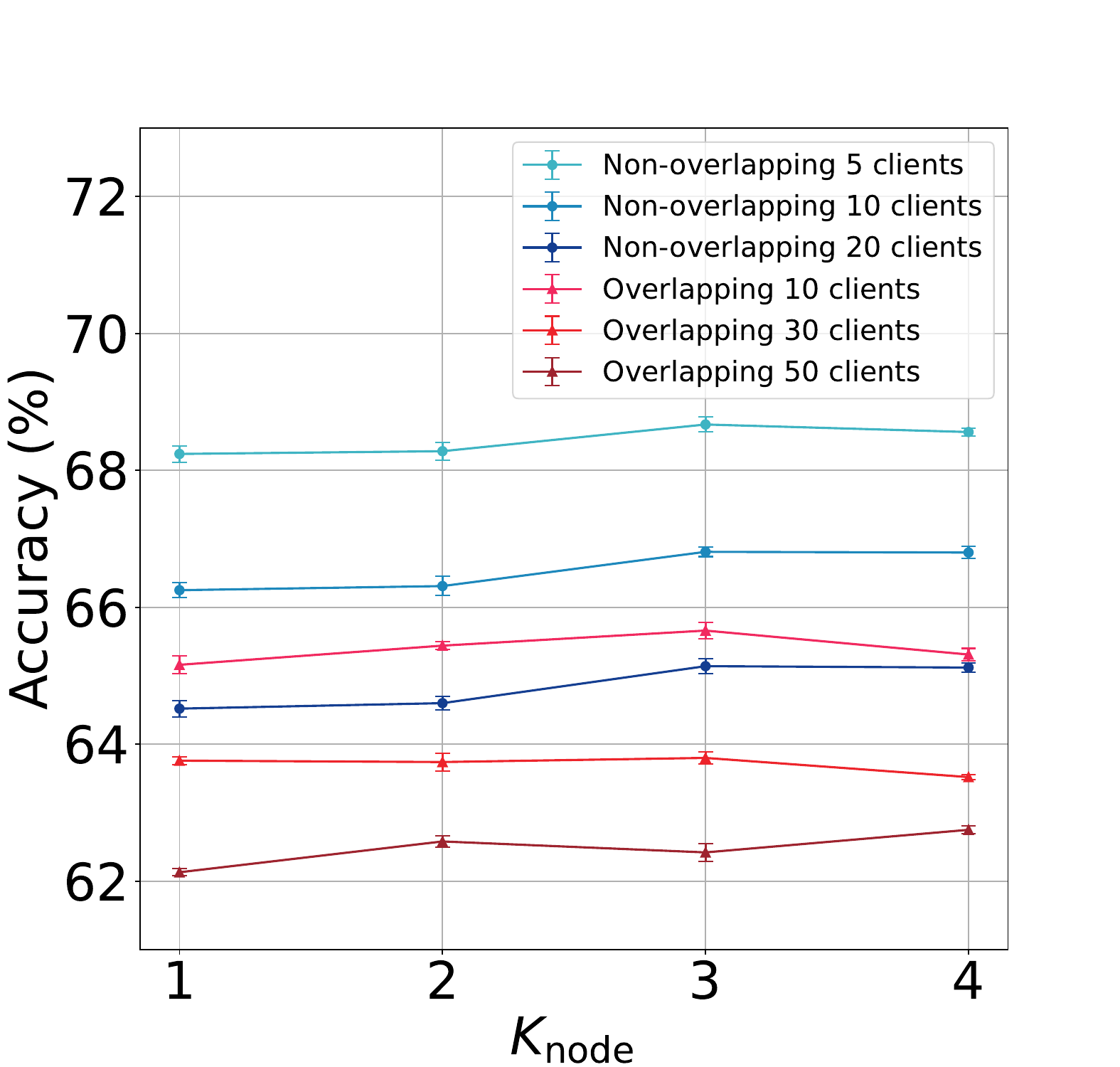}\label{fig_additional_sensitivity_1}}
  \hfill
  \subfloat[\footnotesize{\textit{Roman-empire} under different $K_\mathrm{struct}$}]{\includegraphics[width=0.2\columnwidth]{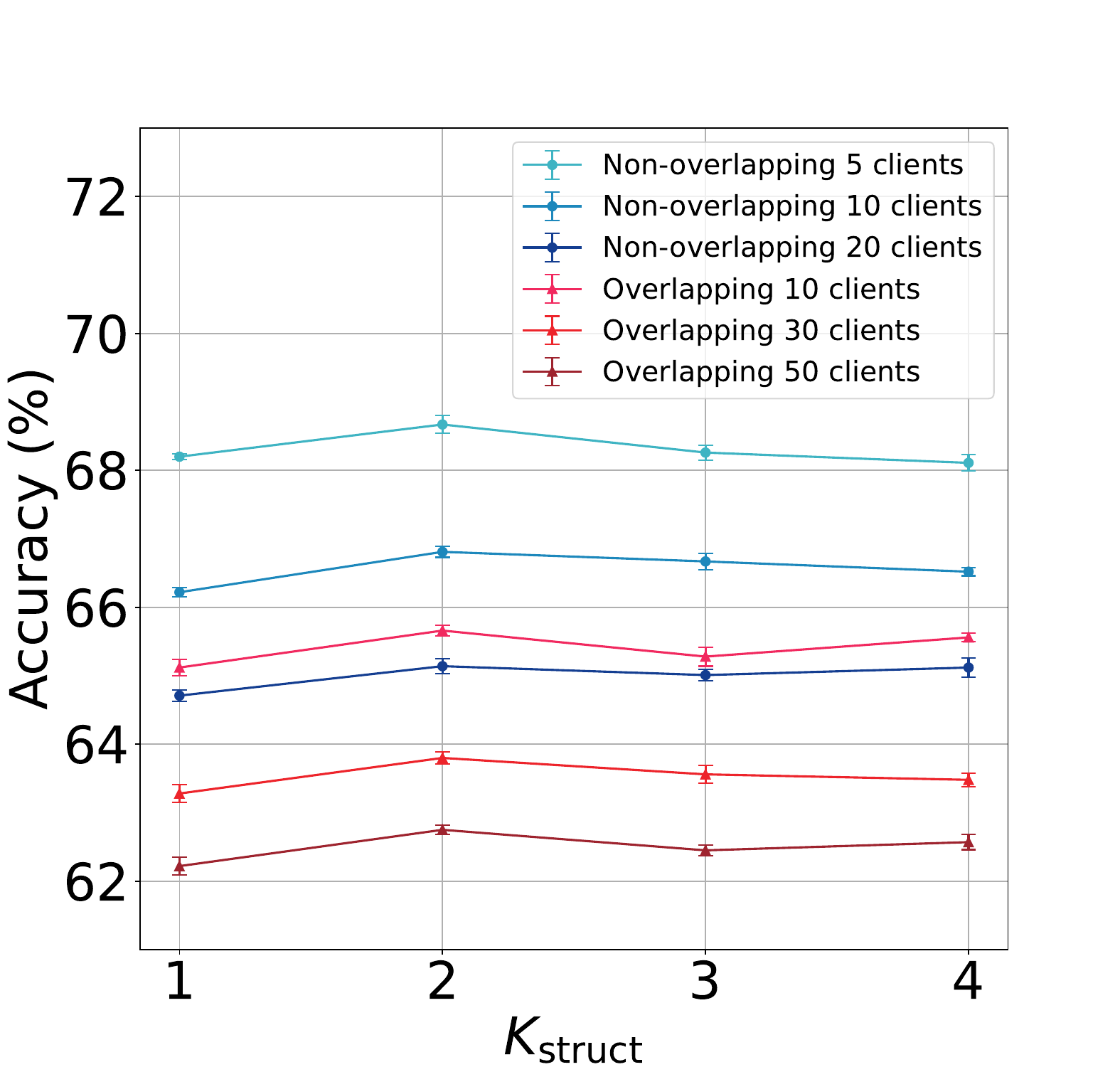}\label{fig_additional_sensitivity_2}}
  \hfill
  \subfloat[\footnotesize{\textit{Roman-empire} under different $\lambda_1$}]{\includegraphics[width=0.2\columnwidth]{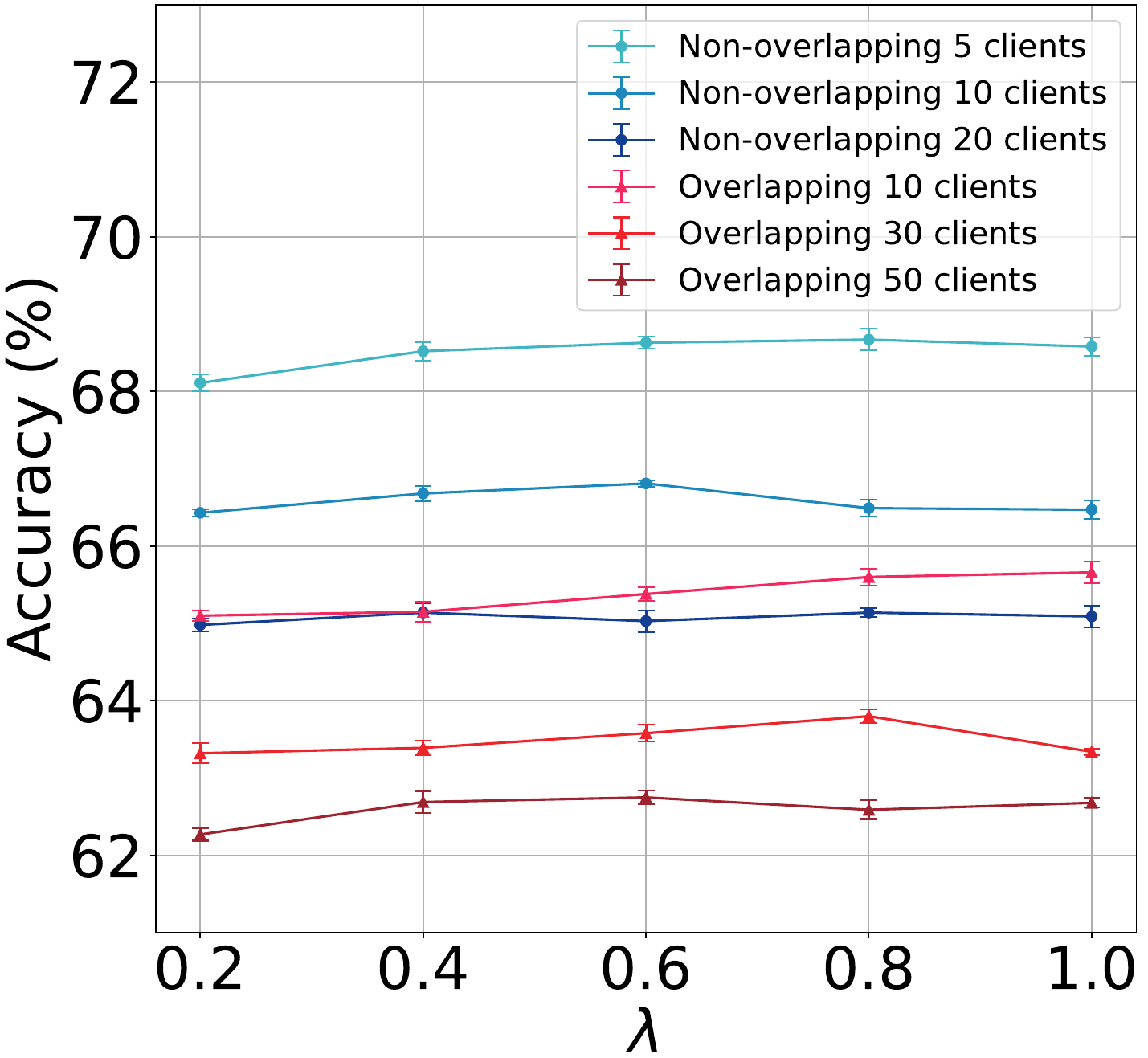}\label{fig_additional_sensitivity_3}}
  \hfill
  \subfloat[\footnotesize{\textit{Roman-empire} under different $\lambda_2$}]{\includegraphics[width=0.212\columnwidth]{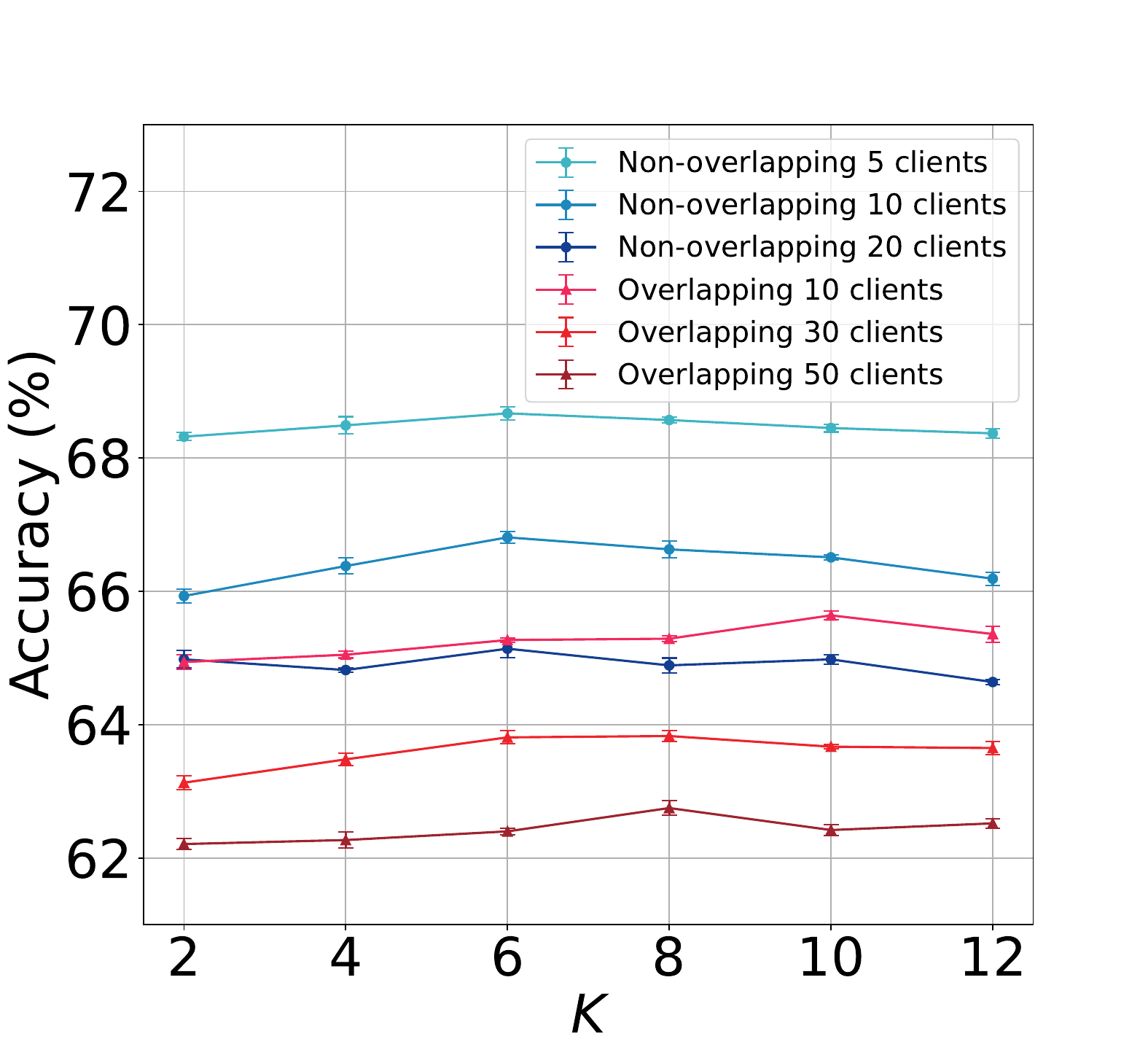}\label{fig_additional_sensitivity_4}}
  \caption{Accuracy curves with variance bars on \textit{Roman-empire} dataset under different values of $K_\mathrm{node}$, $K_\mathrm{struct}$, $\lambda_1$, and $\lambda_2$.}
  \label{fig_additional_sensitivity2}
\end{figure}

\subsection{Additional Case Studies}
\label{ap_case_study}
To further illustrate the effectiveness of our proposed FedSSA in mitigating heterogeneity among clients, we conduct additional case studies on six datasets. To be specific, we first visualize semantic representations of various clients obtained from `w/o semantic' and `with semantic' (\textit{i.e.}, FedSSA without/with semantic knowledge sharing) by using t-SNE~\cite{van2008visualizing} method. Subsequently, we plot spectral properties captured by local models under `w/o structural' and `with structural' (\textit{i.e.}, FedSSA without/with structural knowledge sharing). As shown in~\cref{fig_additional_case_study1}, the 2D projections of different clients exhibit significant divergence under `w/o semantic', which indicates the presence of heterogeneity among clients. In contrast, the 2D projections of representations obtained from `with semantic' show more compact clusters when compared with `w/o semantic', which demonstrates the effectiveness of FedSSA in mitigating node feature heterogeneity. Moreover, as shown in~\cref{fig_additional_case_study2}, spectral properties obtained from `with structural' align more closely to those of cluster-level when compared with `w/o structural', which validates the effectiveness of our FedSSA in addressing structural heterogeneity.

\begin{figure*}[!t]
  \centering
  \subfloat[\footnotesize{\textit{Cora}}]{\includegraphics[width=0.33\columnwidth]{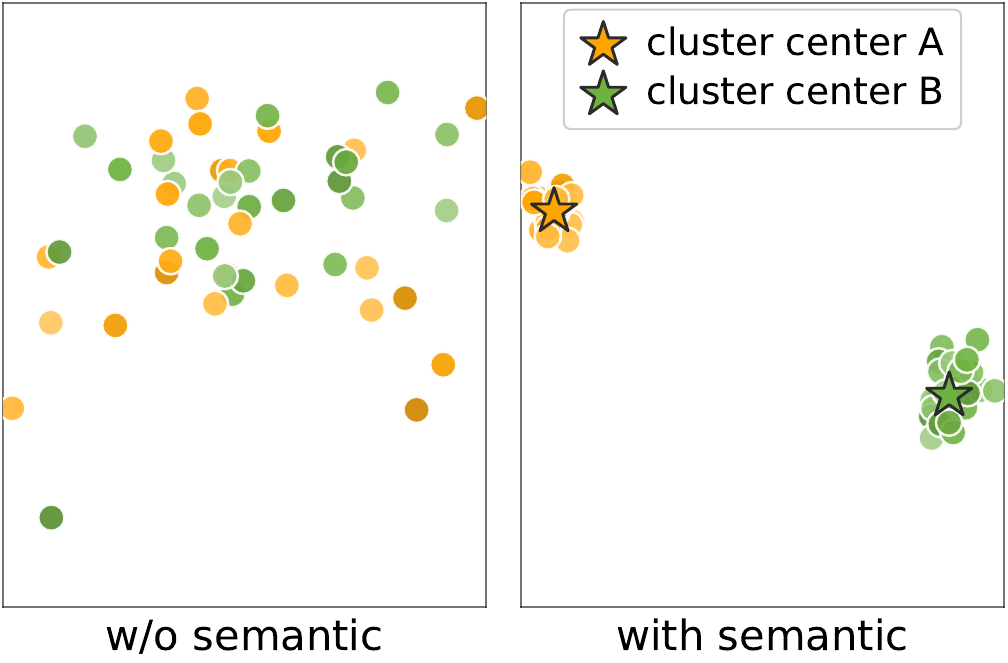}\label{fig_additional_case_study1_1}}
  \hfill
  \subfloat[\footnotesize{\textit{CiteSeer}}]{\includegraphics[width=0.33\columnwidth]{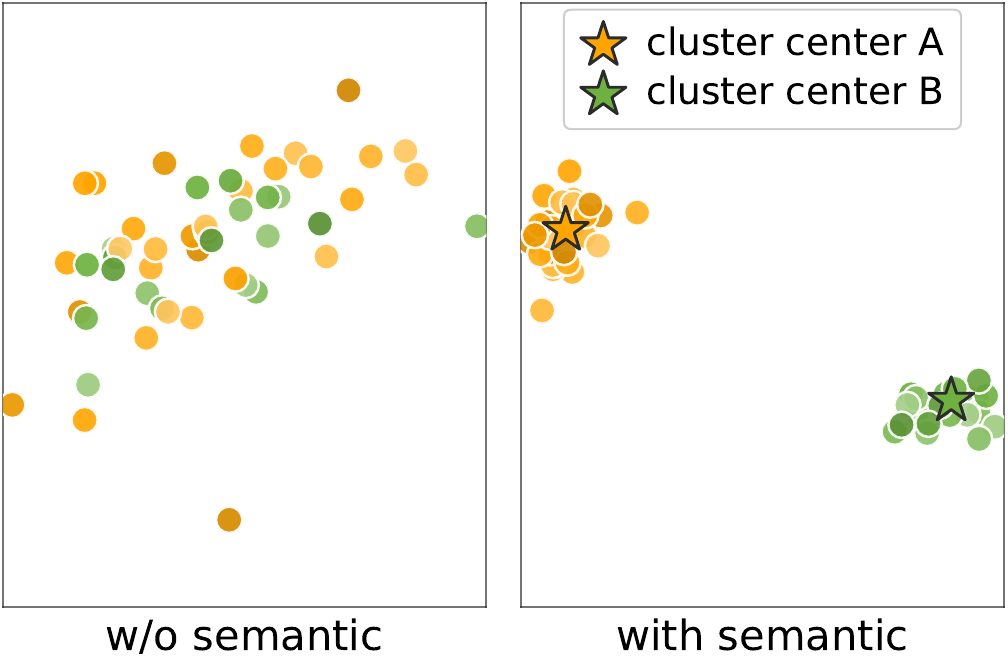}\label{fig_additional_case_study1_2}}
  \hfill
  \subfloat[\footnotesize{\textit{PubMed}}]{\includegraphics[width=0.33\columnwidth]{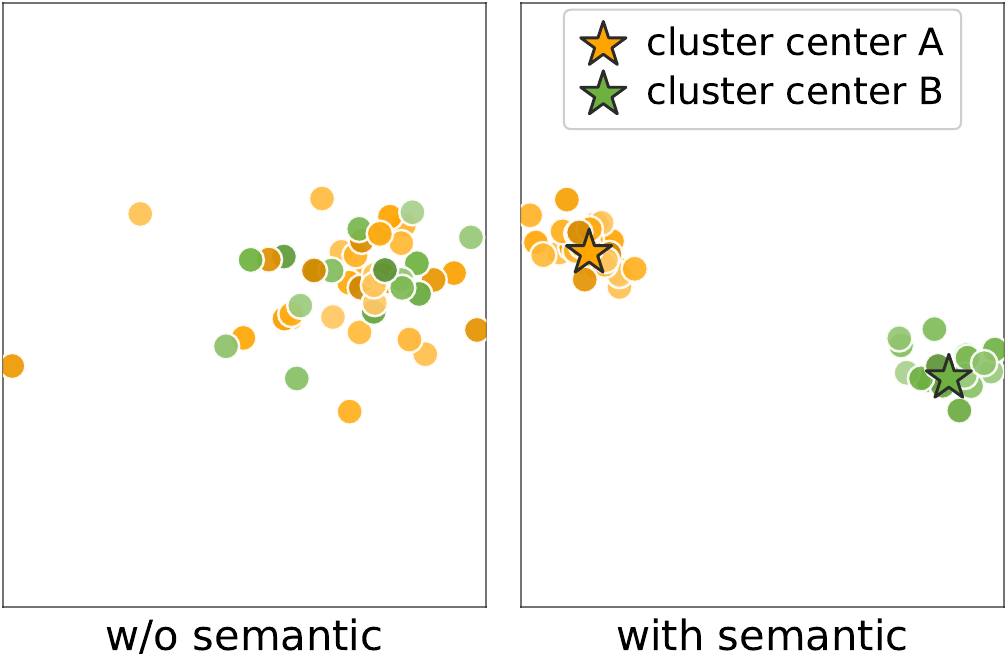}\label{fig_additional_case_study1_3}}
  \hfill
  \subfloat[\footnotesize{\textit{Roman-empire}}]{\includegraphics[width=0.33\columnwidth]{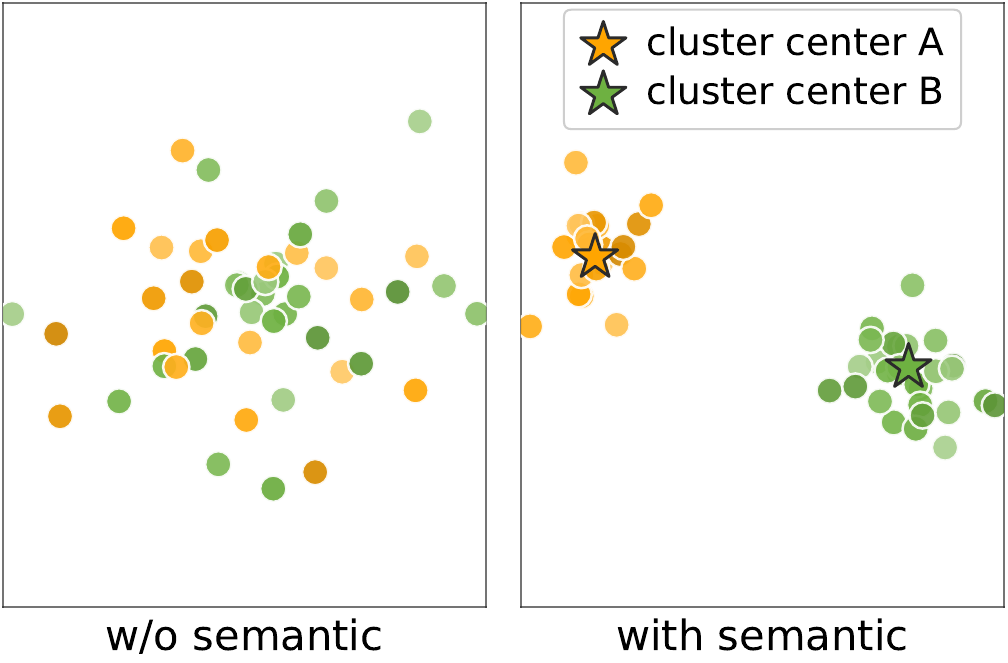}\label{fig_additional_case_study1_4}}
  \hfill
  \subfloat[\footnotesize{\textit{Tolokers}}]{\includegraphics[width=0.33\columnwidth]{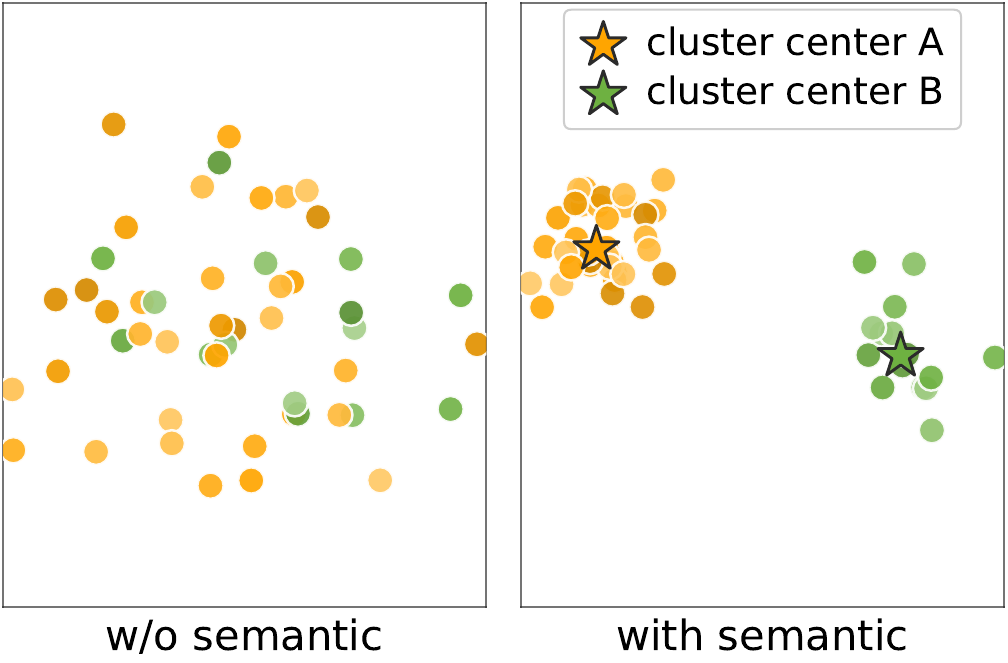}\label{fig_additional_case_study1_5}}
  \hfill
  \subfloat[\footnotesize{\textit{Questions}}]{\includegraphics[width=0.33\columnwidth]{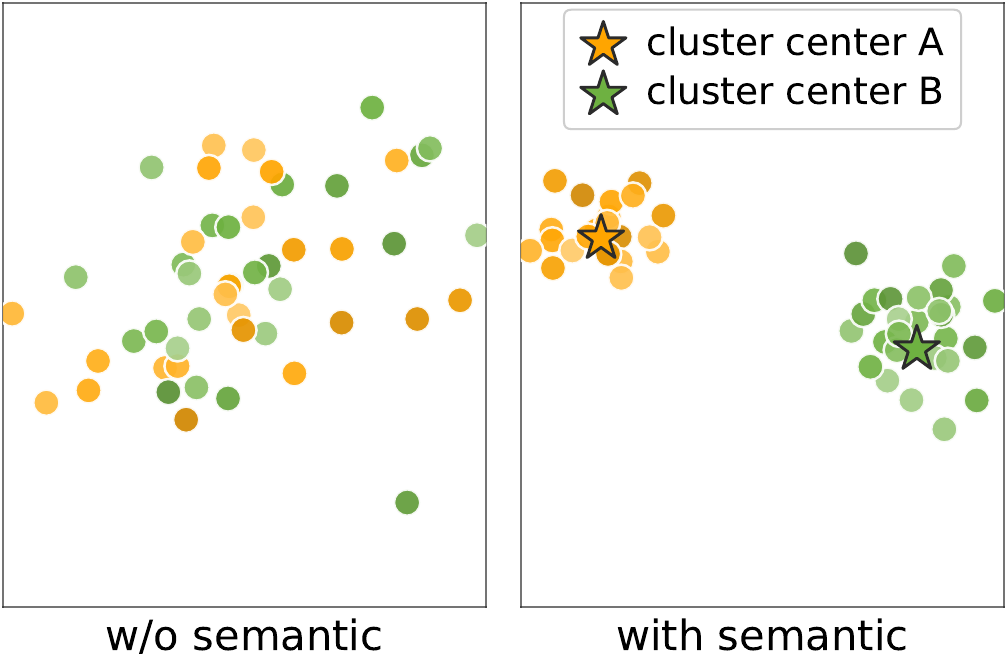}\label{fig_additional_case_study1_6}}
  \caption{Semantic representations on six datasets under overlapping partitioning setting with 50 clients.}
  \label{fig_additional_case_study1}
\end{figure*}

\begin{figure*}[!t]
  \centering
  \subfloat[\footnotesize{\textit{Cora}}]{\includegraphics[width=0.33\columnwidth]{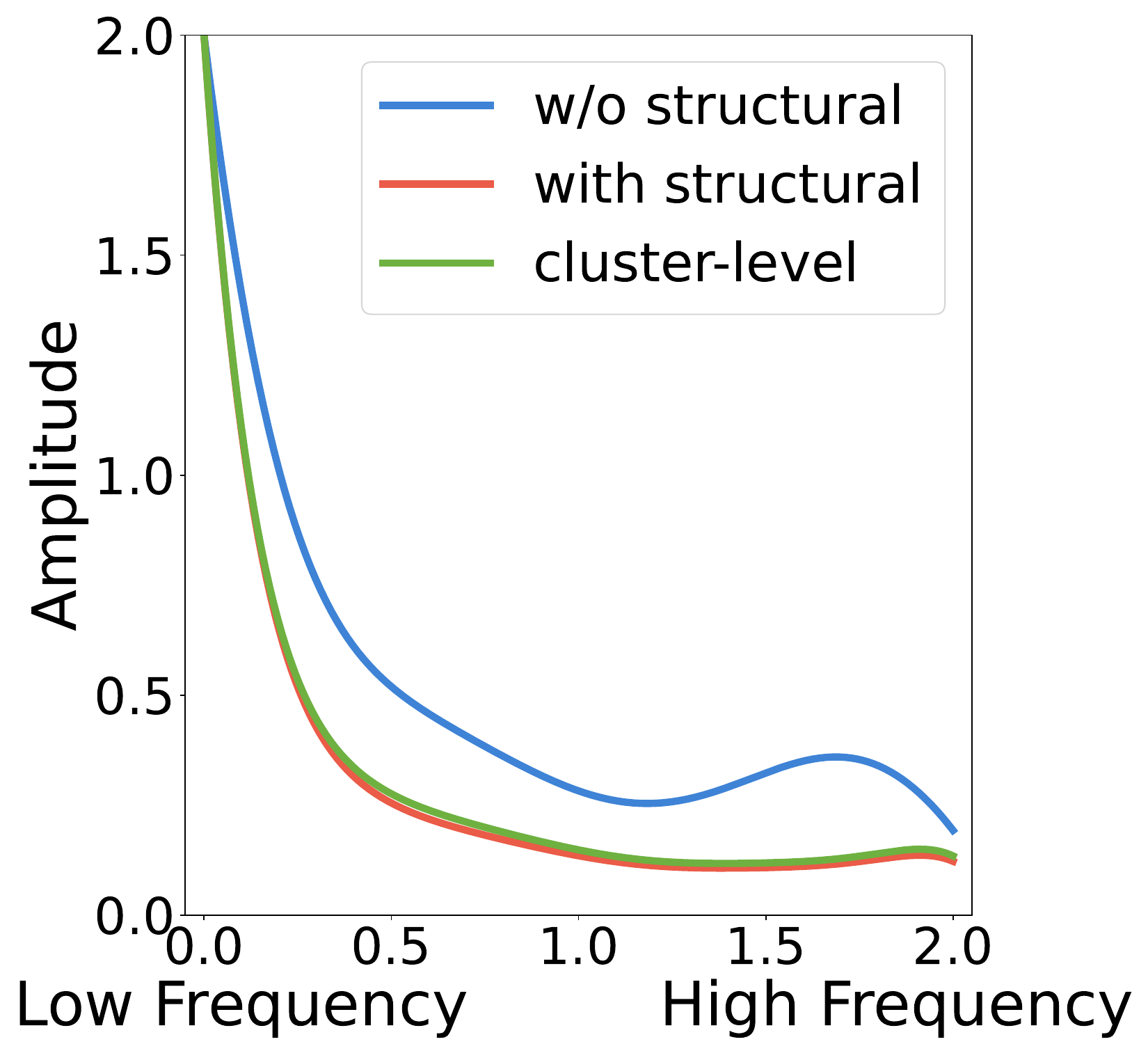}\label{fig_additional_case_study2_1}}
  \hfill
  \subfloat[\footnotesize{\textit{CiteSeer}}]{\includegraphics[width=0.33\columnwidth]{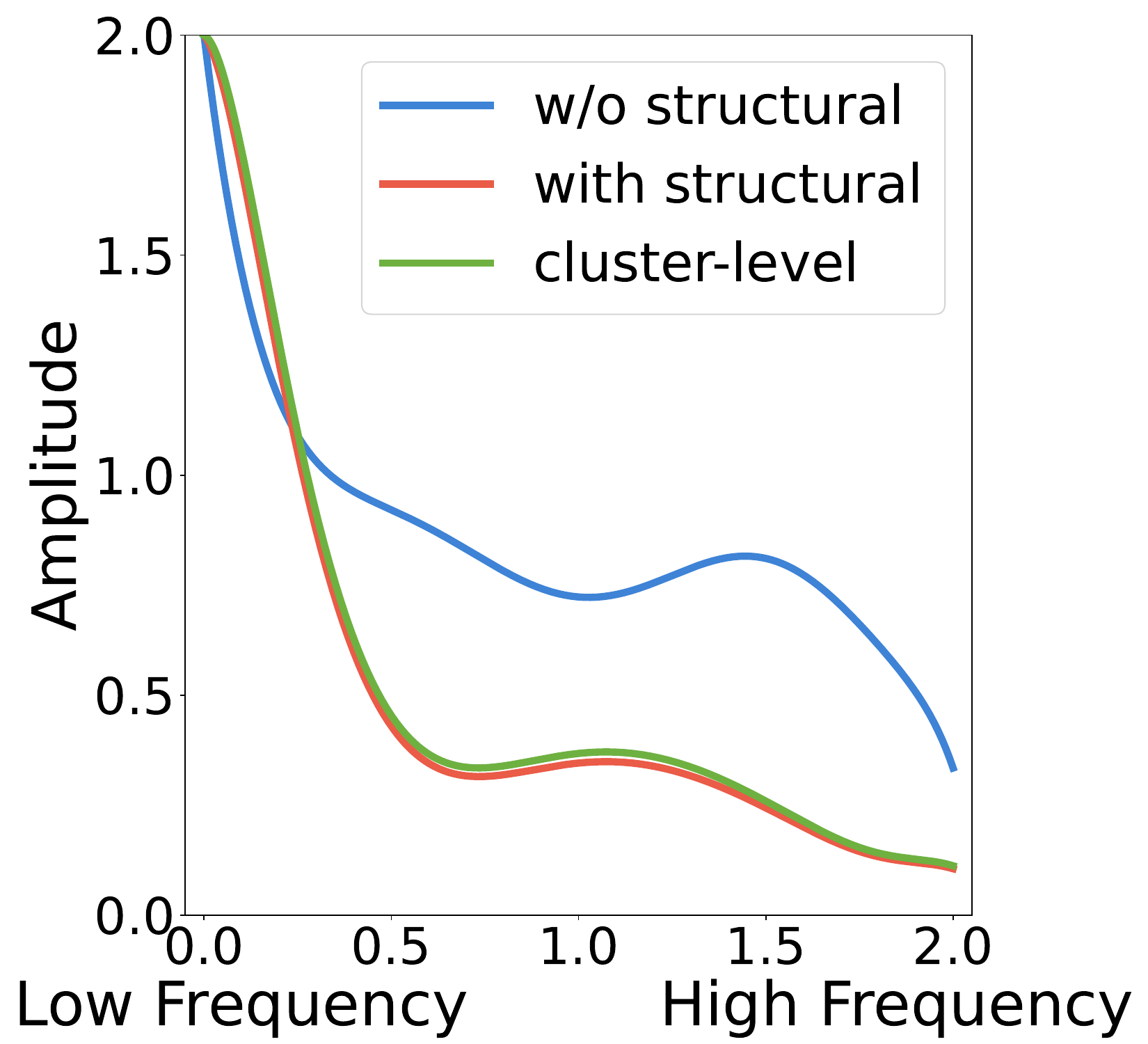}\label{fig_additional_case_study2_2}}
  \hfill
  \subfloat[\footnotesize{\textit{PubMed}}]{\includegraphics[width=0.33\columnwidth]{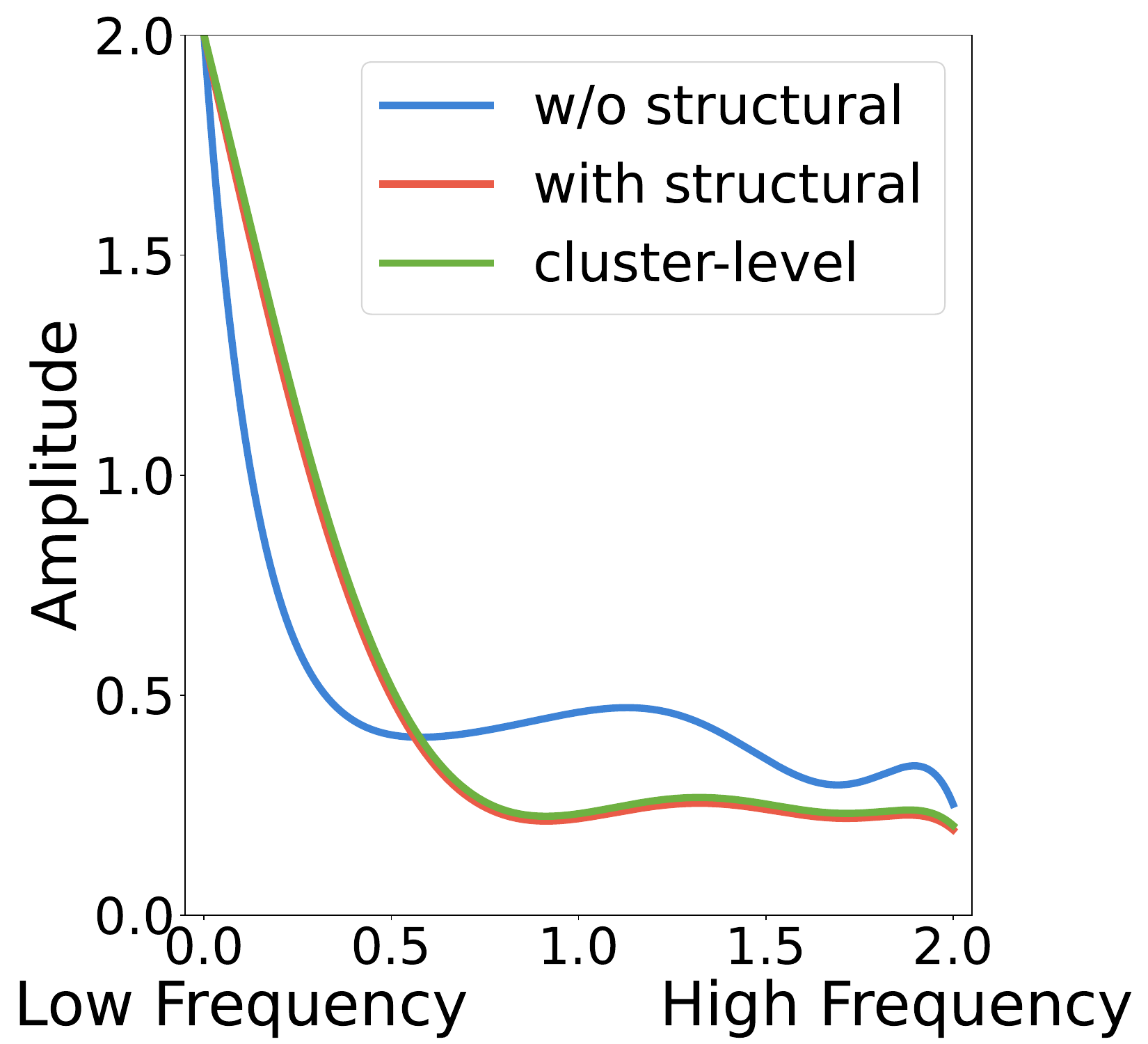}\label{fig_additional_case_study2_3}}
  \hfill
  \subfloat[\footnotesize{\textit{Roman-empire}}]{\includegraphics[width=0.33\columnwidth]{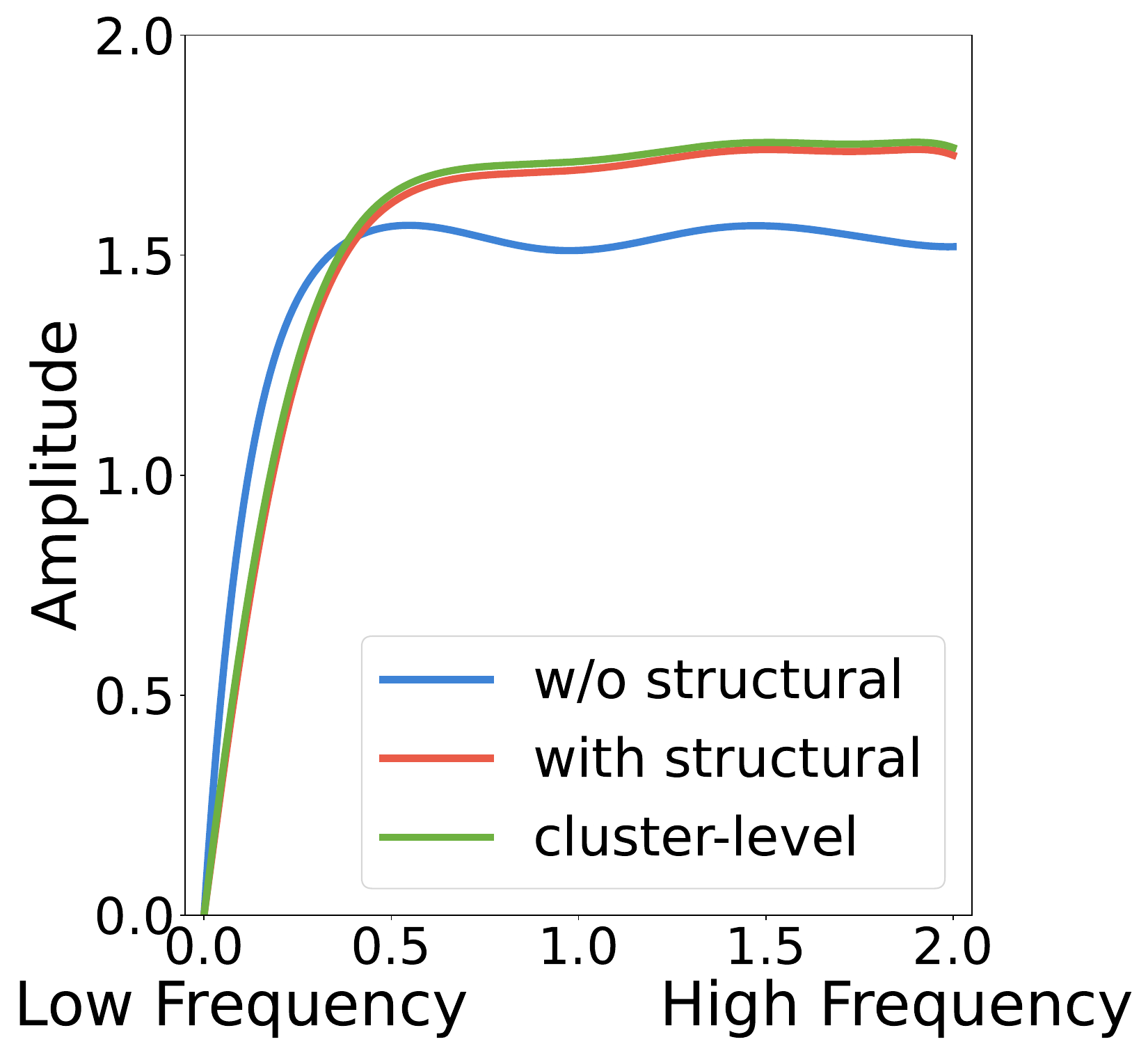}\label{fig_additional_case_study2_4}}
  \hfill
  \subfloat[\footnotesize{\textit{Tolokers}}]{\includegraphics[width=0.33\columnwidth]{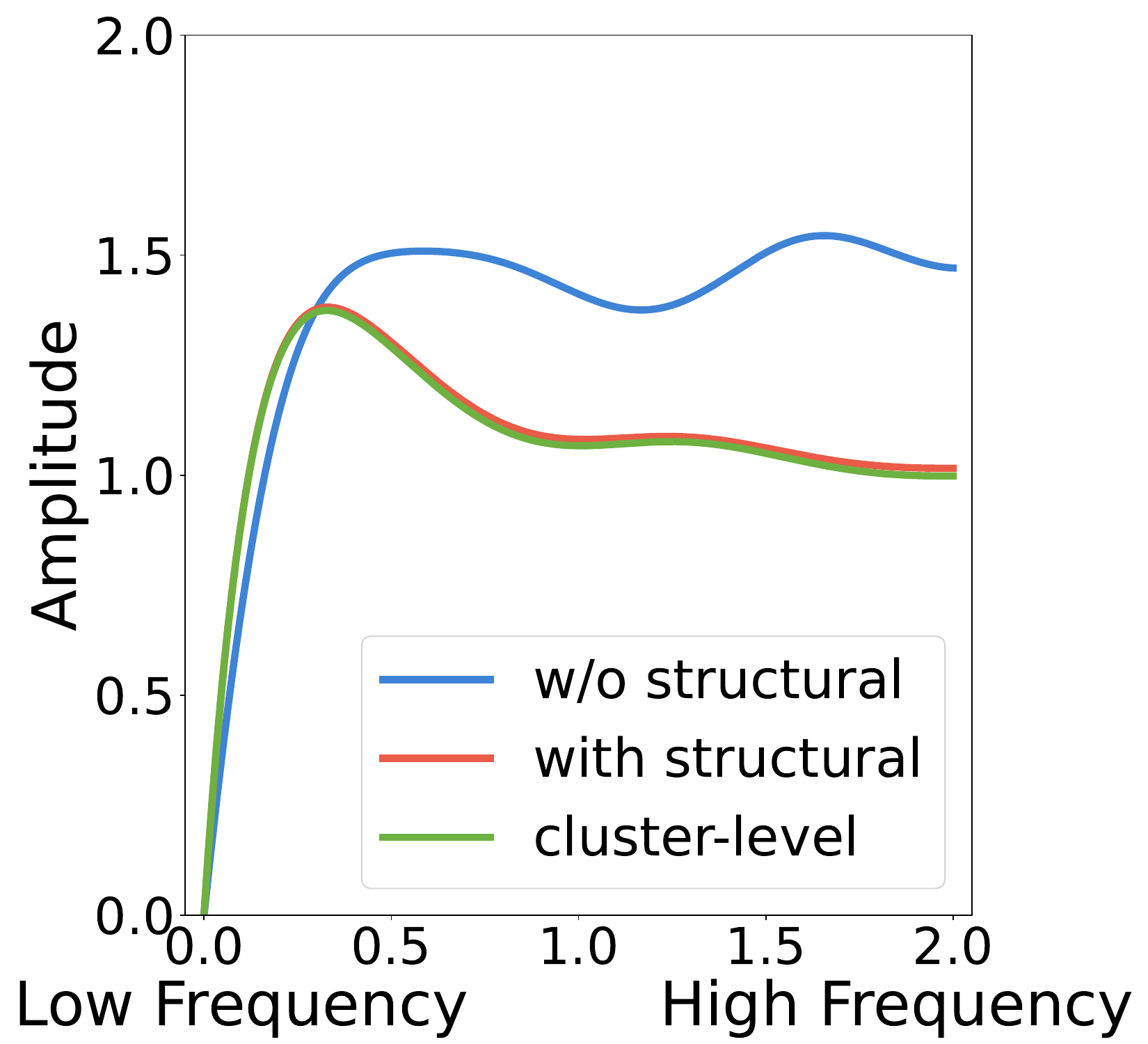}\label{fig_additional_case_study2_5}}
  \hfill
  \subfloat[\footnotesize{\textit{Questions}}]{\includegraphics[width=0.33\columnwidth]{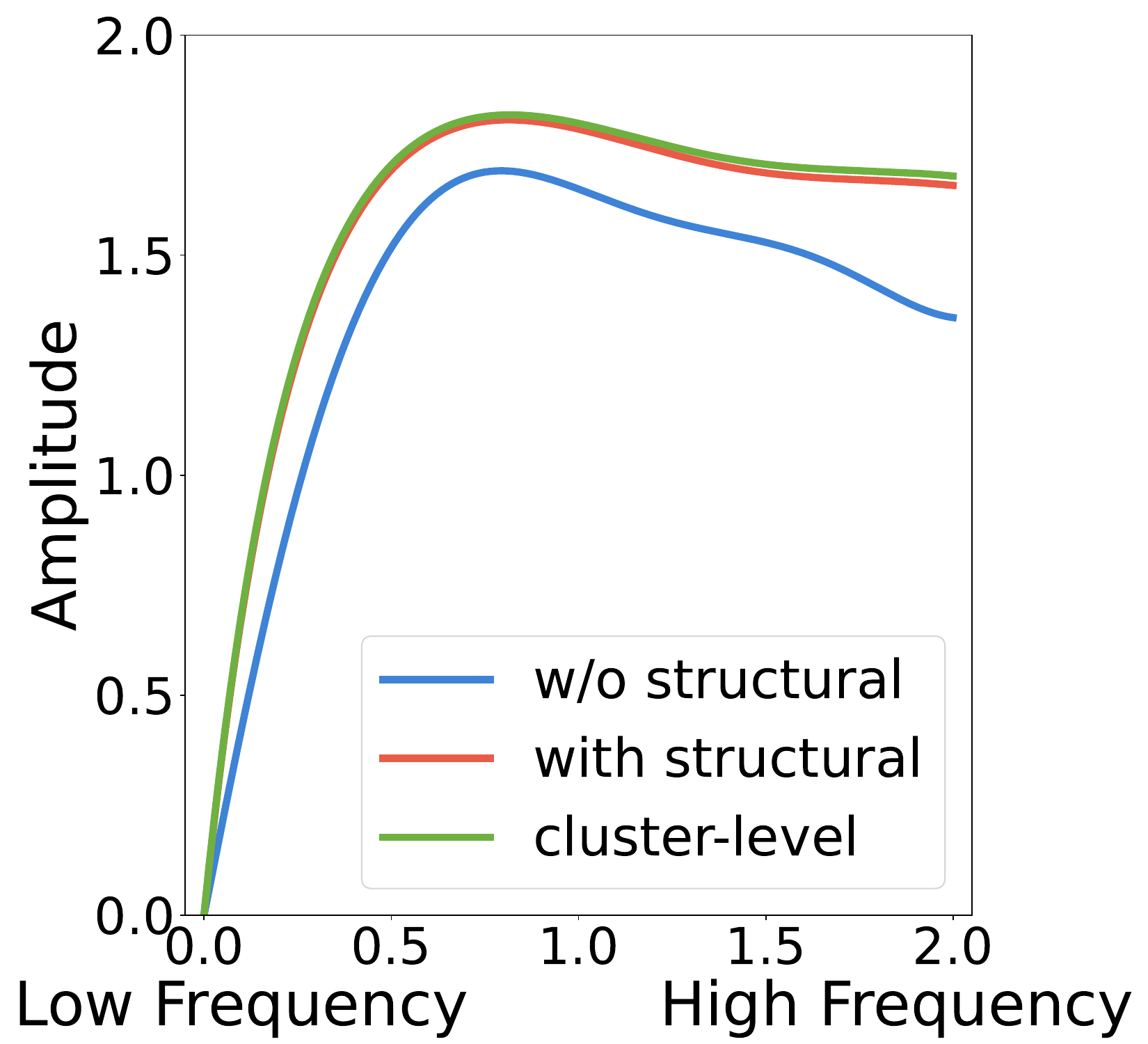}\label{fig_additional_case_study2_6}}
  \caption{Spectral properties on six datasets under overlapping partitioning setting with 50 clients.}
  \label{fig_additional_case_study2}
\end{figure*}

\subsection{Time of Each Communication Round}
\label{time_communication}
We report the time consumed per communication round for our proposed FedSSA and baseline methods in Table~\ref{table6}. We can observe that FedSSA consistently achieves a lower time cost per communication round when compared with the average time of baseline methods. Notably, FedSSA is substantially more efficient than strong baseline methods such as FED-PUB and FedIIH. For example, on \textit{Cora} dataset, FedSSA achieves more than a threefold speed improvement relative to the second-best baseline method (\textit{i.e.}, FedIIH). This efficiency gain is attributed to the lower time complexity of FedSSA on both client side and server side. A detailed analysis of time complexity is provided in Appendix~\ref{Ap_efficiency_analysis}.

\begin{table*}[t]
  \centering
  \scriptsize
  \caption{Time consumption (seconds) of each communication round for our proposed FedSSA and baseline methods on \textit{Cora} and \textit{Roman-empire} datasets.}
  \label{table6}
  \renewcommand{\arraystretch}{1} % 调整行间距
  \scalebox{1}{
  \begin{tabular}{lcccccc}
  \hline
  \rowcolor{gray!50}
                & \multicolumn{6}{c}{Cora}                                                                                                                                                                                                                                                                                                                                                                                                       \\ \hline
  Methods       & \begin{tabular}[c]{@{}c@{}}non-overlapping\\ 5 clients\end{tabular} & \begin{tabular}[c]{@{}c@{}}non-overlapping\\ 10 clients\end{tabular} & \begin{tabular}[c]{@{}c@{}}non-overlapping\\ 20 clients\end{tabular} & \begin{tabular}[c]{@{}c@{}}overlapping\\ 10 clients\end{tabular} & \begin{tabular}[c]{@{}c@{}}overlapping\\ 30 clients\end{tabular} & \begin{tabular}[c]{@{}c@{}}overlapping\\ 50 clients\end{tabular} \\ \hline
  \rowcolor{gray!20}
  FedAvg~\cite{mcmahan2017communication}        & 5.51                                                                & 2.19                                                                 & 5.63                                                                 & 5.40                                                             & 7.27                                                                 & 12.08                                                            \\ \hline
  FedProx~\cite{MLSYS2020_1f5fe839}       & 4.24                                                                & 4.65                                                                 & 8.86                                                                 & 5.49                                                             & 13.71                                                                & 22.43                                                            \\ \hline
  \rowcolor{gray!20}
  FedPer~\cite{Arivazhagan2019}        & 4.06                                                                & 4.13                                                                 & 8.17                                                                 & 4.16                                                             & 12.38                                                                & 20.24                                                            \\ \hline
  GCFL~\cite{NEURIPS2021_9c6947bd}          & 6.30                                                                & 9.38                                                                 & 18.87                                                                & 9.78                                                             & 27.93                                                                & 46.22                                                            \\ \hline
  \rowcolor{gray!20}
  FedGNN~\cite{wu2021fedgnn}        & 2.28                                                                & 4.42                                                                 & 8.76                                                                 & 5.40                                                             & 13.07                                                                & 23.04                                                            \\ \hline
  FedSage+\cite{NEURIPS2021_34adeb8e}      & 6.88                                                                & 8.55                                                                 & 17.88                                                                & 9.37                                                             & 16.97                                                                & 23.35                                                            \\ \hline
  \rowcolor{gray!20}
  FED-PUB~\cite{baek2023personalized}       & 22.04                                                               & 27.34                                                                & 60.31                                                                & 33.46                                                            & 80.54                                                                & 147.84                                                           \\ \hline
  FedGTA~\cite{li2023fedgta}        & 3.36                                                                & 2.24                                                                 & 5.89                                                                 & 4.30                                                             & 5.00                                                                 & 7.18                                                             \\ \hline
  \rowcolor{gray!20}
  AdaFGL~\cite{li2024adafgl}        & 1.99                                                                & 2.49                                                                 & 4.63                                                                 & 4.44                                                             & 6.16                                                                 & 7.83                                                             \\ \hline
  FedTAD~\cite{zhu2024fedtad}        & 4.91                                                                & 5.25                                                                 & 9.13                                                                 & 5.22                                                             & 12.84                                                                & 19.71                                                            \\ \hline
  \rowcolor{gray!20}
  FedIIH~\cite{wentao2025fediih}        & 19.57                                                               & 22.76                                                                & 56.09                                                                & 19.67                                                            & 65.49                                                                & 139.03                                                           \\ \hline
  FedSSA (Ours) & 2.87                                                                & 3.46                                                                 & 4.94                                                                 & 6.48                                                             & 11.17                                                                 & 16.47                                                            \\ \hline
  \rowcolor{yellow!30} \textbf{Average}       & 7.00                                                                & 8.07                                                                 & 17.43                                                                & 9.43                                                             & 22.71                                                            & 40.45                                                            \\ \hline
  \rowcolor{gray!50}
                & \multicolumn{6}{c}
                {Roman-empire}                                                                                                                                                                                                                                                                                                                                                                                               \\ \hline
  Methods       & \begin{tabular}[c]{@{}c@{}}non-overlapping\\ 5 clients\end{tabular} & \begin{tabular}[c]{@{}c@{}}non-overlapping\\ 10 clients\end{tabular} & \begin{tabular}[c]{@{}c@{}}non-overlapping\\ 20 clients\end{tabular} & \begin{tabular}[c]{@{}c@{}}overlapping\\ 10 clients\end{tabular} & \begin{tabular}[c]{@{}c@{}}overlapping\\ 30 clients\end{tabular} & \begin{tabular}[c]{@{}c@{}}overlapping\\ 50 clients\end{tabular} \\ \hline
  \rowcolor{gray!20}
  FedAvg~\cite{mcmahan2017communication}        & 8.20                                                                & 5.76                                                                 & 8.46                                                                 & 10.25                                                            & 18.47                                                                & 19.12                                                            \\ \hline
  FedProx~\cite{MLSYS2020_1f5fe839}       & 4.31                                                                & 6.73                                                                 & 9.04                                                                 & 6.90                                                             & 16.25                                                                & 23.66                                                            \\ \hline
  \rowcolor{gray!20}
  FedPer~\cite{Arivazhagan2019}        & 5.35                                                                & 6.19                                                                 & 9.19                                                                 & 6.47                                                             & 15.58                                                                & 20.22                                                            \\ \hline
  GCFL~\cite{NEURIPS2021_9c6947bd}          & 6.32                                                                & 9.58                                                                 & 18.37                                                                & 10.84                                                            & 29.91                                                                & 50.23                                                            \\ \hline
  \rowcolor{gray!20}
  FedGNN~\cite{wu2021fedgnn}        & 3.33                                                                & 6.60                                                                 & 9.43                                                                 & 6.45                                                             & 15.29                                                                & 22.82                                                            \\ \hline
  FedSage+\cite{NEURIPS2021_34adeb8e}      & 10.06                                                               & 14.82                                                                & 26.27                                                                & 23.09                                                            & 47.42                                                                & 62.72                                                            \\ \hline
  \rowcolor{gray!20}
  FED-PUB~\cite{baek2023personalized}       & 18.81                                                               & 28.03                                                                & 61.75                                                                & 28.45                                                            & 83.07                                                                & 133.05                                                           \\ \hline
  FedGTA~\cite{li2023fedgta}        & 2.17                                                                & 3.58                                                                 & 5.32                                                                 & 3.42                                                             & 6.12                                                                 & 9.92                                                             \\ \hline
  \rowcolor{gray!20}
  AdaFGL~\cite{li2024adafgl}        & 4.34                                                                & 4.14                                                                 & 5.55                                                                 & 6.49                                                             & 7.80                                                                 & 10.69                                                            \\ \hline
  FedTAD~\cite{zhu2024fedtad}        & 4.88                                                                & 8.55                                                                 & 15.03                                                                & 10.98                                                            & 22.42                                                                & 37.01                                                            \\ \hline
  \rowcolor{gray!20}
  FedIIH~\cite{wentao2025fediih}        & 17.45                                                               & 24.90                                                                & 47.01                                                                & 28.19                                                            & 61.17                                                                & 100.10                                                           \\ \hline
  FedSSA (Ours) & 5.78                                                                & 6.01                                                                 & 11.93                                                                 & 9.04                                                            & 16.59                                                                & 24.78                                                            \\ \hline
  \rowcolor{yellow!30} \textbf{Average}       & 7.58                                                                & 10.41                                                                & 18.95                                                                & 12.55                                                            & 28.34                                                            & 42.86                                                            \\ \hline
  \end{tabular}
  }
\end{table*}

%%%%%%%%%%%%%%%%%%%%%%%%%%%%%%%%%%%%%%%%%%%%%%%%%%%%%%%%%%%%%%%%%%%%%%%%%%%%%%%
%%%%%%%%%%%%%%%%%%%%%%%%%%%%%%%%%%%%%%%%%%%%%%%%%%%%%%%%%%%%%%%%%%%%%%%%%%%%%%%

\end{document}